\documentclass{article}

\PassOptionsToPackage{numbers, sort&compress}{natbib}
\usepackage[preprint]{neurips_2026}

\usepackage[utf8]{inputenc} 
\usepackage[T1]{fontenc}    
\usepackage[hidelinks]{hyperref}       
\usepackage{url}            
\usepackage{booktabs}       
\usepackage{amsfonts}       
\usepackage{nicefrac}       
\usepackage{microtype}      
\usepackage{xcolor}         
\usepackage{wrapfig}
\usepackage{enumitem} 
\usepackage{multirow}

\usepackage{amsmath, amssymb,mathtools}
\usepackage{amsthm,thmtools}
\usepackage{bm}
\usepackage{dsfont}  

\usepackage{caption}
\usepackage{subcaption}

\usepackage[noend]{algpseudocode}
\usepackage{algorithm}

\usepackage{xpatch}

\makeatletter
\xpatchcmd{\algorithmic}
  {\ALG@tlm\z@}{\leftmargin\z@\ALG@tlm\z@}
  {}{}
\makeatother

\theoremstyle{definition}
\newtheorem{definition}{Definition}
\newtheorem{assumption}{Assumption}

\theoremstyle{plain}
\newtheorem{theorem}{Theorem}
\newtheorem{proposition}{Proposition}
\newtheorem{lemma}{Lemma}
\newtheorem{corollary}{Corollary}

\newtheorem{restatedthminner}{Theorem}
\newenvironment{restatedthm}[1]{%
\renewcommand\therestatedthminner{#1}%
\restatedthminner
}{\endrestatedthminner}

\newtheorem{restatedpropinner}{Proposition}
\newenvironment{restatedprop}[1]{%
\renewcommand\therestatedpropinner{#1}%
\restatedpropinner
}{\endrestatedpropinner}

\renewcommand{\S}{\mathsf{S}}
\newcommand{\A}{\mathsf{A}}
\newcommand{\R}{\mathbf{R}}
\renewcommand{\P}{\mathbf{P}}
\newcommand{\AP}{\mathsf{AP}}
\newcommand{\X}{\mathbf X}

\renewcommand{\L}{\mathsf{L}}
\newcommand{\Q}{\mathsf{Q}}
\newcommand{\F}{\mathsf{F}}

\renewcommand{\F}{\mathcal{F}}

\newcommand{\ind}[1]{\mathds{1}{\{#1\}}}
\newcommand{\indt}[1]{\mathds{1}_t{\{#1\}}}

\title{\textsc{RecReg}: Recovery-based Shielding for Regular Safety Properties in Reinforcement Learning}

\author{%
  Alexander W. Goodall\\
  Department of Computer Science\\
  Imperial College London\\
  \texttt{a.goodall22@imperial.ac.uk} \\
  \And
  Francesco Belardinelli\\
  Department of Computer Science\\
  Imperial College London\\
  \texttt{francesco.belardinelli@imperial.ac.uk} \\
}

\begin{document}

\maketitle

\begin{abstract}
In many real-world scenarios we need to consider formal requirements, such as safety and constraint compliance, to deploy Reinforcement Learning (RL) agents. In this paper we introduce \textsc{RecReg}, a safe RL framework capable of enforcing non-Markovian regular safety properties via recovery-based shielding. \textsc{RecReg} is a unified approach that supports different levels of {\em a priori} knowledge and safety guarantees. The paper primarily focuses on building the theoretical foundations for \textsc{RecReg} in the tabular setting, providing rigorous safety and convergence guarantees that are not present in related work. We provide an extensive set of empirical results using the MASA-Safe-RL library, and we outline an extension of \textsc{RecReg} to deep RL domains using SOTA model-based RL. Our empirical results demonstrate that \textsc{RecReg} achieves efficient learning and matches the theory in this paper.
\end{abstract}

\section{Introduction}

Safe Reinforcement Learning (RL) is an active area of research focusing on 
the training of agents with RL while maintaining 
some 
standard of safe behaviour at deployment and possibly during training as well \citep{amodei2016concrete}. There exist many competing paradigms in safe RL, the most common being \emph{Constrained Markov Decision Processes} (CMDPs) \citep{altman1999constrained}, which frame safe RL as a constrained optimization problem, enforcing safety through cost functions and expected cumulative cost thresholds; although this approach is
often inappropriate in highly safety-critical domains, as it only enforces safety in expectation and often lacks rigorous guarantees \citep{voloshin2022policy}. 

An alternative view comes from \emph{formal methods}. Methods from this class often leverage some form of \emph{Linear Temporal Logic} (LTL) \citep{pnueli1989synthesis} to express formal safety constraints and task objectives, enabling expressive and interpretable specifications \citep{hasanbeig2018logically,hasanbeig2020cautious, de2020restraining, voloshin2022policy,voloshin2023eventual, shah2024deep, manganaris2025automaton}. Following this line of work, \emph{shielding} has emerged as a powerful approach for enforcing LTL safety properties in RL \citep{alshiekh2018safe}. Predicated on ideas like the \emph{Simplex Architecture} \citep{seto1998simplex}, shielding works by monitoring the agent's actions and overriding them with a safe alternative if they propose an action that could lead to a safety violation. Classical approaches to shielding leverage reactive synthesis to build correct-by-construction monitors from formal safety specification and a formal model of the environment called the \emph{safety abstraction} \citep{bloem2015shield, alshiekh2018safe}. The \emph{safety abstraction} is designed to encode the minimum amount of environment dynamics required to guarantee safety. As such, it is often built from a pessimistic view of the environment derived from prior knowledge, thus classical shielding approaches can only reasonably be applied to known environments or simple safety specifications. This fundamentally departs from the usual RL framework (unknown environment), but does provide strong safety guarantees by construction. 

More recent shielding approaches consider probabilistic safety constraints \citep{jansen2020safe, court2025probabilisticshieldingsafereinforcement}. While these are less pessimistic, they still require knowledge of the transition kernel, using probabilistic model checking to verify actions. There has however been a shift to lift these assumptions, in particular, by leveraging model-based RL to estimate the unknown environment dynamics and verify actions via forward simulation in the latent space of a recurrent state space model \citep{he2022androids, xiao2023model, goodall2023approximate, goodall2024leveraging}, with Gaussian Process dynamics models \citep{goodall2026safe} or ensembles of probabilistic neural networks \citep{gronauer2024reinforcement}. \textsc{RecReg} addresses this same question: \emph{how can we obtain meaningful safety guarantees when the environment dynamics are unknown?}. We summarize our contributions hereafter.
\paragraph{Contributions.} To bridge the gap between the strong empirical results achievable by data driven methods (e.g., CMDP, model-based RL) and formally guaranteed safety (classical shielding) we propose \textsc{RecReg}, a unified framework for enforcing regular safety properties that spans multiple levels of \emph{a priori} knowledge. In a similar fashion to Recovery RL \citep{thananjeyan2021recovery}, \textsc{RecReg} learns a backup policy trained to maximize safety probability and action replacement (shielding) is implemented by identifying \emph{probabilistic safe sets} (or ``recovery zones''). A central contribution of this paper is in the development of explicit safety guarantees, convergence and optimality results 
which are missing in prior work \citep{thananjeyan2021recovery} 
. In particular, we identify and formalize three complementary operation modes. In the \textsc{Exact} mode, \textsc{RecReg} assumes full knowledge of the transition kernel and guarantees safety from the beginning of training.  In the \textsc{ModelBased} mode, \textsc{RecReg} estimates the transition kernel from observed data, then performs on-the-fly verification via finite-horizon reachability analysis, enabling principled shielding under partial model knowledge with statistical guarantees. In the \textsc{ModelFree} mode, \textsc{RecReg} dispenses with explicit model learning altogether and instead directly estimates the safety critic, yielding a purely model-free shielding mechanism with convergence guarantees. We also use \textsc{DreamerV3} \cite{hafner2025mastering} to develop a deep RL variant of \textsc{RecReg} that uses latent rollout models and safety critics to provide safety guarantees under correct model calibration. In the tabular setting, we validate that \textsc{RecReg} matches the outlined theory and that comparable approaches cannot reliably enforce safety constraints while maximising task performance. In deep RL we also observe that \textsc{RecReg} is able to reliably reduce safety violations while achieving good returns compared to similar approaches based on the \textsc{DreamerV3} architecture. Together these results establish \textsc{RecReg} as a practical and theoretically-grounded approach to safe RL
%
%
\paragraph{Related work.}
Existing methods for safe RL can roughly be dividing into two camps: ({\bf A}) methods that incorporate safety into the learning objective, encouraging safe behaviour as an outcome of learning \citep{chow2018risk, liang2018accelerated, tessler2018reward, ray2019benchmarking, as2022constrained, achiam2017constrained, yang2020projection, hasanbeig2018logically, shah2024deep}, and ({\bf B}) methods that enforce safety via external mechanisms that directly modify the agent's actions \citep{alshiekh2018safe, jansen2020safe, konighofer2021online, bastanib2021safe, thananjeyan2021recovery, konighofer2023online,gronauer2024reinforcement}. Methods from camp ({\bf B}) make up a large portion of ``provably safe`` RL algorithms \citep{krasowski2023provably}. They can also be further divided into \emph{hard constraints} \citep{alshiekh2018safe} and \emph{probabilistic constraints} \citep{jansen2020safe, bastanib2021safe, court2025probabilisticshieldingsafereinforcement}. Methods based on CMDPs are from camp ({\bf A}) -- they use a cost function, that often represents safety violations, in the learning objective to reduce the probability of unsafe behaviour. CMDPs themselves can be solved via linear programming in the tabular setting \citep{altman1999constrained}. For more complex environments, methods based on {\em Lagrange relaxations} and  {\em dual-gradient descent} \citep{chow2018risk, liang2018accelerated, tessler2018reward, ray2019benchmarking, as2022constrained}, and {\em local policy search} with gradient projections \citep{achiam2017constrained, yang2020projection} are popular approaches. While CMDPs provide a convenient framework for learning, they have several limitations. In particular, safety constraints are only enforced in expectation, which do not guarantee that unsafe events will never occur. Furthermore, the discounting accumulated costs also poses issues with constraining limiting behaviour \citep{voloshin2022policy}.

In contrast, our approach \textsc{RecReg} fits neatly into camp ({\bf B}), enforcing \emph{regular safety properties}  (more expressive than LTL safety properties) with high probability. LTL constraints have been considered in RL for both task specification \citep{hasanbeig2018logically, hasanbeig2020cautious, de2020restraining} and constraint specification \citep{voloshin2022policy, voloshin2023eventual, shah2024deep}. LTL constraints are PAC learnable under specific conditions \citep{yang2021tractability, voloshin2022policy}, but optimality under LTL constraints can only be achieved subject to optimal discount factors and Lagrange multipliers \citep{voloshin2023eventual,shah2024deep}. Earlier theoretical work assumed access to a perfect generative model of the environment which can be queried from any state, breaking standard RL assumptions \citep{voloshin2022policy}. Recent work on \emph{Automata Constrained Q-learning} uses reward shaping and penalties to ensure the optimal policy is also safe at convergence although does not provide explicit safety guarantees during learning unlike our approach \citep{manganaris2025automaton}. For RL with LTL safety constraints, Shielding can be used to enforce constraint compliance \citep{alshiekh2018safe}. In partially observable settings (e.g., deep RL) safety can be guaranteed if the underlying transition structure is known \citep{carr2023safe}. \textsc{RecReg}, by contrast, offers a unified framework for safety guarantees even when the transition dynamics or structure are truly unknown. 

Another related line of research is \emph{Reward Machines} \citep{icarte2018using,icarte2022reward}; these are closely related to our work as simple reward machines can encode reward functions for learning regular safety properties. While the \emph{Reward Machines} literature is primarily focused on task specification, there are many key insights such as \emph{counterfactual experiences} and \emph{reward shaping} \citep{ng1999policy} that are used to speed up learning over \emph{Reward Machines}, which can applied in our framework too. 

\section{Preliminaries}

For a finite set $\S$, let $P(\S)$ denote the power set and let $D(\S)$ denote the set of all probability distributions on $\S$. Let $\S^*$ and $\S^{\omega}$ denote the set of finite and infinite sequences over $\S$ respectively,
and denote their union as 
$\S^{\infty} = \S^* \cup \S^{\omega}$. Let $|\rho| \in \mathbb{N} \cup \{ \infty \}$ denote the length of a sequence  $\rho \in \S^{\infty}$, with $|\rho| = \infty$ for $\rho \in S^{\omega}$. We also let $\rho[i]$ denote the $i+1$-th element, when $i < |\rho|$, and we denote the last element of a sequence $\rho  \in \mathcal{S}^*$ as $\rho\!\!\downarrow\, = \rho[|\rho|-1]$.
A sequence $\rho_1$ is a prefix of $\rho_2$, denoted $\rho_1 \preceq \rho_2$, if $|\rho_1| \leq |\rho_2|$ and $\rho_1[i] = \rho_2[i]$ for all $0 \leq i \leq |\rho_1|$. A sequence $\rho_1$ is a proper prefix of $\rho_2$, denoted $\rho_1 \prec \rho_2$, if  $\rho_1 \preceq \rho_2$ and $\rho_1 \neq \rho_2$. 

\paragraph{Reinforcement learning.}
We follow the usual convention of modelling the environment as a \emph{Markov Decision Process} (MDP) $\mathcal{M} = (\S, \A, \S_0, \P, \R)$, where $\S$ and $\A$ are finite sets of states and actions respectively; $\S_0 \subseteq \S$ is the set of initial states; $\P : \S \times \A \to D(\S)$ is a (generally unknown) transition kernel; and $\R : \S \times \A \times \S \to \mathbb{R}$ is a reward function.
A solution for the RL problem can take the form of a (memoryless, stochastic) policy $\pi : \S \to D(\A)$, which induces a {\em Markov chain} $\mathcal{M}_{\pi}$ with transition kernel
$\P_\pi(s' | s) = \sum_{a \in \A} \pi(a | s)\P(s' | s,a)$. 
We let $\Pi$ denote the set of memoryless policies on $\S$.
For a given discount factor $\gamma \in [0,1]$, the {\em value function} of policy $\pi$ is the {\em expected discount return}
$V^{\mathcal{M}}_{\pi}(s) = \mathbb{E}^{\mathcal{M}}_{\pi} \left[ \sum^{\infty}_{t=0} \gamma^t \R(s_t, a_t, s_{t+1}) \mid s_0 = s \right]$, where $\mathbb{E}^{\mathcal{M}}_{\pi}$ denotes the expected value of following $\pi$ in $\mathcal{M}$. 
Similarly, the {\em state-action value function} or Q-function, is defined as $Q^\mathcal{M}_{\pi}(s, a) = \mathbb{E}^{\mathcal{M}}_{\pi}[\sum^{\infty}_{t=0}\gamma^t R(s, a, s') \mid s_0 = s, a_0 = a]$. Hereafter 
we drop the superscript $\mathcal{M}$ when it is clear from the context. 
The aim of RL is to find an optimal policy
$\pi^*$ such that for all $s \in \S$, we have $\pi^*(s) = \arg\sup_{\pi \in \Pi} V_{\pi}(s)$. We write, $V^* = \max_{\pi} V_{\pi}$ and $Q^* = \max_{\pi} Q_{\pi}$ to denote the optimal value functions. Note that there always exists an optimal policy in $\Pi$ \citep{puterman2014markov}.

\paragraph{Safety semantics.}
We consider a set $\AP$ of {\em atomic propositions} (or atoms). Then,
$\Sigma = P(\AP)$ is the alphabet over $\AP$ (e.g., if $\AP = \{a, b\}$, then $P(\AP)=\{ \varnothing, \{a\}, \{b\}, \{a, b\}\}$). The \emph{labelling function} $\L : \S \to \Sigma$ maps states to sets of atoms (intuitively, those that hold in a given state). A {\em path} is a sequence  $\rho = s_0s_1\ldots \in \S^{\infty}$ of states which gives rise to a corresponding {\em trace} $\tau(\rho) = \L(s_0)\L(s_1)\ldots \in \Sigma^{\infty}$. By standard probability theory, the transition kernel $\P_\pi$ for a policy $\pi$ in $\mathcal{M}$ induces a probability measure $\Pr^\mathcal{M}_\pi$ over paths and traces in the Markov chain $\mathcal{M}_{\pi}$ \citep{baier2008principles}.

\textbf{Regular safety property.} A linear-time property $P_{\text{safe}} \subseteq \Sigma^\omega$ is a \emph{safety property} if every violating trace has a finite \emph{bad prefix}, i.e., a finite word that cannot be extended to any trace in $P_{\text{safe}}$. A safety property is \emph{regular} if its set of bad prefixes forms a regular language. Equivalently, every regular safety property can be represented by a deterministic finite automaton (DFA) $\mathcal{D} = (\Q, \Sigma, \Delta, \Q_0, \F)$ that accepts exactly its bad prefixes. Here $\Q$ is a finite set of states with initial state $\Q_0 \in \Q$; $\Sigma$ is a finite alphabet; $\Delta : \Q \times \Sigma \to \Q$ is the deterministic transition function;
and $\F \subseteq \Q$ is the set of accepting states. A trace satisfies the safety property iff
no finite prefix of the trace is accepted by $\mathcal{D}$. It is well know that for any LTL safety property there is a regular language that contains its bad prefixes \citep{baier2008principles}. Therefore, \textsc{RecReg} can be straightforwardly applied to LTL safety constraints. 

\textbf{Safety probability.} Given an MDP $\mathcal{M}$ with labelling function $\L$, the probability of satisfying a regular safety property under policy $\pi$ from state $s \in \S$ is defined as the probability that the induced traces satisfies the safety property, i.e.,
\begin{equation}
    \mathbb{P}_\pi(s \models P_{\text{safe}}) := \textstyle{\Pr^{\mathcal{M}}_\pi} \left( \{ \rho \in \S^\omega \mid \tau(\rho) \in P_{\text{safe}},\ \rho[0] = s \} \right) \label{eq:safetyprobability}
\end{equation}
Equivalently, using the characterization via bad prefixes and the DFA $\mathcal{D}$ recognizing them, this is the probability that no finite prefix of the generated trace is accepted by $\mathcal{D}$. That is,
\begin{equation}
    \mathbb{P}_\pi(s \models P_{\text{safe}}) = \textstyle{\Pr^{\mathcal{M}}_\pi} \left(\{ \rho \in \S^\omega \mid \forall\, \rho' \prec \rho,\ \tau(\rho') \notin \mathcal{L}(\mathcal{D}),\ \rho[0] = s \} \right)
\end{equation}
where $\mathcal{L}(\mathcal{D})$ denotes the language of words accepted by $\mathcal{D}$. Intuitively $\mathbb{P}_\pi(s_0 \models P_{\text{safe}})$ measures the probability that from $s_0$ following $\pi$ the system never violates $P_\text{safe}$ in finite time.

\textbf{Problem Statement.} Our objective is to synthesize a policy that maximizes reward while guaranteeing that the safety probability remains above a given threshold. Formally,
\begin{equation}
    \max_\pi V_\pi \;\; \text{subject to} \;\; \mathbb{P}_\pi(s_0 \models P_{\text{safe}})  \geq 1 - \epsilon \;\;  \forall s_0 \in \S_0\label{eq:problem}
\end{equation}
where $\epsilon \in (0, 1)$ is a user-defined safety lower bound. For regular safety properties more complex than simple invariants, memoryless policies may not be able to achieve the desired bound or act suboptimally under the constraint. In the next section we present a solution to this issue. 

\section{\textsc{RecReg}: Recovery-based Shielding}
This section details the \textsc{RecReg} approach for solving \eqref{eq:problem}. We start by introducing a \emph{product MDP} construction; defining memoryless policies on an expanded state space that preserves optimality. We then continue by formalizing \emph{probabilistic safe sets} and how the backup policy is trained, followed by details regarding how the task policy is trained including a safety-aware Bellman operator that guarantees convergence. We then detail each of the three operation modes and corresponding guarantees (safety, convergence, optimality), and finally we conclude this section by providing a high level overview of \textsc{RecReg} applied to deep RL.

\textbf{Product MDP and reachability semantics.} We address \eqref{eq:problem} by synchronizing MDP states with the DFA for monitoring bad prefixes. 
\begin{definition}[Product MDP] \label{def:product-mdp}
For a given MDP $\mathcal{M}$ with labelling function $\L$ and DFA $\mathcal{D}$ whose \emph{language} $\mathcal{L}(\mathcal{D})$ is precisely the set of bad prefixes for $P_\text{safe}$, the \emph{product MDP} is denoted $\mathcal{M}^\otimes$, with state space $\S^\otimes = \S \times \Q$, initial states $\S^\otimes_0 = \S_0 \times \Q_0$, transition kernel $\P^\otimes : \S^\otimes \times \A \to D(\S^\otimes)$, where $\P^\otimes((s, q), a, (s', q')) = \P(s, a, s')$ if $q' = \Delta(q, \L(s'))$. The \emph{product labelling function} can also be defined: $\L^\otimes: \S^\otimes \to P(\{\mathrm{unsafe}\})$, where $\mathrm{unsafe} \in \L((s, q)) $ if $q \in \F$; labelling accepting DFA states as ``unsafe''. The action space $\A$ and reward function $\R$ remain unchanged. 
\end{definition}

For convenience we may also write $s^\otimes = (s, q) \in \S^\otimes$ to denote a generic product state and $\rho^\otimes$ to denote a product path. 
A key observation we can now make is that the safety probability in \eqref{eq:safetyprobability} reduces to a \emph{reachability probability} in the product MDP \citep{baier2008principles} (for proof c.f.~Appendix \ref{sec:reachability-proof}). Formally for a policy $\pi$ from state $s \in \S$,
\begin{equation}
    \mathbb{P}_\pi(s \models P_{\text{safe}}) = 1 - \mathbb{P}_\pi(s^\otimes \models \lozenge \mathrm{unsafe})
\end{equation}
where $s^{\otimes} = (s, \Delta(\Q_0, L(s)))$ and $\lozenge \mathrm{unsafe}$ is an LTL formula that reads ``eventually unsafe''. Let,
\begin{equation}
    \mathbb{P}_\pi(s^\otimes \models \lozenge \mathrm{unsafe}) := \textstyle{\Pr^{\mathcal{M}^\otimes}_\pi}\left(\{ \rho^\otimes \in (\S^\otimes)^\omega 
    \mid \rho^\otimes[0] = s^\otimes ,\ \exists\, t \geq 0,\ 
    \mathrm{unsafe} \in \L^\otimes(\rho^\otimes[t]) \}\right)
\end{equation}
be defined as the \emph{reachability probability}, i.e., the probability of reaching a product state labelled ``unsafe'', where $\Pr^{\mathcal{M}^\otimes}_\pi$ is the corresponding path measure in the product MDP. Intuitively the safety probability corresponds to the probability of avoiding ``unsafe'' (accepting) states in the product MDP. This is useful because we can evaluate and quantify the safety probability of a policy $\pi$ from a given state via probabilistic model checking in the product MDP (e.g., via reachability analysis) which is polynomial in the size of the product state space \citep{baier2008principles}. We now write $\Pi^\otimes$ to denote the set of memoryless policies over $\S^\otimes$, in general these are non-Markovian in the original state space $\S$, and the product states provide the minimum amount of information to learn optimal (safe) policies.

\textbf{Probabilistic safe set.} Even if a state $s^\otimes$ is not labelled ``unsafe'' it does not mean that it is operationally safe, since there could be some unavoidable high-probability of eventually reaching an ``unsafe'' state under the current policy. To truly characterize operationally safe states and actions we use the idea of \emph{probabilistic safe sets} and recoverable actions \citep{abate2008probabilistic,thananjeyan2021recovery}.  We define,
\begin{align}
    S(\pi, \epsilon) &:= \{s^\otimes \in \S^\otimes \mid \mathbb{P}_\pi(s^\otimes \models \lozenge \mathrm{unsafe}) < \epsilon \} \\
    S(\pi, a, \epsilon) &:= \{s^\otimes \in \S^\otimes \mid \mathbb{P}^a_\pi(s^\otimes \models \lozenge \mathrm{unsafe}) < \epsilon \}
\end{align}
where $\mathbb{P}^a_\pi$ denotes the probability of paths under $\pi$ but in the first step resolving the non-determinism of the MDP with action $a$. In words, $S(\pi, \epsilon)$ defines the set of states that have a risk level of no more than $\epsilon$ under policy $\pi$, and $S(\pi, a, \epsilon)$ is the state-action analogue. We also define the \emph{maximal probabilistic safe sets} which exploit the best available policy from each state,
\begin{align}
    S^*(\epsilon) &:= \{s^\otimes \in \S^\otimes \mid \textstyle{\inf_\pi}\{ \mathbb{P}_\pi(s^\otimes \models \lozenge \mathrm{unsafe}) \} < \epsilon \} \\
    S^*(a, \epsilon) &:= \{s^\otimes \in \S^\otimes \mid \textstyle{\inf_\pi}\{  \mathbb{P}^a_\pi(s^\otimes \models \lozenge \mathrm{unsafe}) \} < \epsilon \} \label{eq:maximal-state-action}
\end{align}
It is not too hard to see that $S(\pi, \epsilon) \subseteq S^*(\epsilon)$ (resp.~ $S(\pi, a, \epsilon) \subseteq S^*(a, \epsilon)$) for all $\pi \in \Pi^\otimes$ and $\epsilon \in (0, 1)$ (and $a \in \A$), thus $S^*(\epsilon)$ and $S^*(a,\epsilon)$ are maximal, providing the largest set of operationally safe states and actions \citep{abate2008probabilistic}. These probabilistic safe sets can be viewed as safety value functions constrained by $\epsilon$, thus there exists an optimal or maximally safe policy in $\Pi^\otimes$ \citep{abate2008probabilistic}.

\textbf{Backup (safe) policy.} \textsc{RecReg} trains two policies in $\Pi^\otimes$ simultaneously: (i) a task policy $\widehat{\pi}$, trained to maximize expected discounted reward, and (ii) a backup policy $\pi_{\text{safe}}$ used to characterize the \emph{probabilistic safe sets} and provide backup actions in the event of a shield override. We first discuss here how the backup policy is trained, before discussing how both policies are reconciled via action replacement or shielding. In particular, we want to train $\pi_{\text{safe}}$ to compute the \emph{maximally safe policy} and recover the \emph{maximal probabilistic safe sets}. This can be achieved by solving a \emph{stochastic avoidance problem}. To this end we introduce the following cost function,
\begin{equation}
    \mathbf C(s^\otimes, a, {s^\otimes}') = \begin{cases}
        1 & \text{if $\mathrm{unsafe} \in \L^\otimes({s^\otimes}')$}, \\
        0 & \text{otherwise}.
    \end{cases} \label{eq:costfn}
\end{equation}
Rather than standard exponential discounted with $\gamma<1$ we use \emph{eventual discounting} (ED) which assigns equal cost to all violating traces regardless of when violation occurs \cite{voloshin2023eventual}. Under this criteria minimizing the expected return is equivalent to maximizing the probability of reaching an ``unsafe'' state as established by Proposition~\ref{prop:maximally}.

\begin{proposition}[Maximal safe policy] \label{prop:maximally} The optimal policy $\pi^*$ that minimizes the expected return under eventual discounting \cite{voloshin2023eventual} and with rewards given by \eqref{eq:costfn}, is such that $S(\pi^*, \epsilon) = S^*(\epsilon)$ for any $\epsilon \in (0, 1)$ (is maximally safe). 
\end{proposition}

\textbf{Action replacement (shielding).} We denote the final (shielded) policy that we use during training and deployment as $\pi^\lozenge$. The safety constraint in \eqref{eq:problem} is enforced by action replacement or shielding: filtering actions proposed by the task policy $\widehat{\pi}$, based on our current estimate of the maximally safe set. Formally, fix a sequence of step-wise risk thresholds $\epsilon_0,\epsilon_1, \epsilon_2, \ldots$ satisfying $\epsilon_t \leq \epsilon$ and $\sum_{t=0}^\infty \epsilon_t \leq \epsilon$, the shielded policy is then defined as, 
\begin{equation*}
    \pi^\lozenge(s^\otimes_t) \begin{cases}
        a \sim \widehat{\pi}(\cdot \mid s^\otimes_t) & \text{if $s^\otimes_t \in S(\pi_{\text{safe}}, a ,\epsilon_t)$} 
 \\
        a \sim \pi_{\text{safe}}(\cdot \mid s^\otimes_t) & \text{otherwise.} \label{eq:shieldedpolicy}
    \end{cases}
\end{equation*}
In practice, we can use either a fixed $\epsilon_t=\epsilon/T$ for finite episodes with length $T$ or a convergent series for the infinite horizon guarantees.\footnote{E.g., $\epsilon_t = 6\epsilon/((t+1)^2\pi^2)$ satisfies $\sum_{t=0}^{\infty}\epsilon_t\le \epsilon$.} Intuitively the shielded policy allows the task policy to act whenever the current (product) state lies within probabilistic safe set (under $\pi_{\text{safe}}$). If the condition $s^\otimes_t \in S(\pi_{\text{safe}}, a ,\epsilon_t)$ can be correctly checked then the constraint in \eqref{eq:costfn} is provably satisfied via a simple recursive feasibility argument. Theorem \ref{thm:constraintsatisfaction} establishes the conditions for when this holds.

\begin{theorem}[Constraint satisfaction] \label{thm:constraintsatisfaction} Assume that all initial states satisfy the global safety constraint under the safe policy, i.e., $\S^\otimes_0 \subseteq S(\pi_{\mathrm{safe}}, \epsilon)$. Then if the step-wise risk budgets satisfy $\sum_{t=0}^\infty \epsilon_t \leq \epsilon$, then the shielded policy $\pi^\lozenge$ satisfies $\mathbb{P}_{\pi^\lozenge}(s_0 \models P_{\mathrm{safe}}) \geq 1-\epsilon, \;\; \forall s_0 \in \S_0$.   
\end{theorem}

Importantly the proof for Theorem \ref{thm:constraintsatisfaction} is \emph{policy-agnostic}, i.e., it works for an arbitrary $\widehat \pi$. Furthermore, we do not necessary need $\pi_{\text{safe}}$ to be maximally safe, only ``good enough'' from the initial states. Although as $\pi_{\text{safe}}$ converges to the maximally safe policy it enlarges $S(\pi_{\text{safe}}, a ,\epsilon_t)$. The crux of the problem is how to reliably check $s^\otimes_t \in S(\pi_{\text{safe}}, a ,\epsilon_t)$, this is the key difference between the three complementary modes of \textsc{RecReg}, with the \textsc{Exact} mode directly computing $S(\pi_{\text{safe}}, a ,\epsilon_t)$ via probabilistic model checking to the \textsc{ModelFree} mode estimating it via online Bellman updates. 

This is where our method fundamentally departs from prior work, e.g., Recovery RL \citep{thananjeyan2021recovery}. While Recovery RL also uses a separate
task policy and backup or recovery policy, its recovery set is defined by a
learned safety critic for the composite policy (i.e., $\pi^\lozenge$), and the task
policy is otherwise trained to maximize task reward in the modified MDP
induced by action replacement. Consequently, the safe region used for
filtering actions need not coincide with $S^*(a,\epsilon)$, and may be a strict subset of it, leading to unnecessary interventions and an over conservative shield. In contrast with \textsc{RecReg} as $\pi_{\mathrm{safe}}$
converges, the shield approaches the largest set of actions that can be
certified safe under the given risk budget. This distinction is important
when task reward and safety are conflicting objective and allows us to obtain principled convergence and optimality guarantees for the task policy later on.

\textbf{Task policy.} Another important detail missing in prior work is how corrected actions are used in the Bellman update for the task policy. During early learning if the shield incorrectly allows unsafe actions to pass through, the task policy may get stuck proposing high-reward but unsafe actions, unless we add a corrective mechanism. This phenomenon has been observed in robotics environments \cite{goodall2024leveraging}.  Prior work, e.g., Recovery RL \citep{thananjeyan2021recovery} manages this by treating the proposed action as the executed action, essentially viewing the action interventions as part of the environment. This doesn't resolve the issue however, rather we train the task policy with a \emph{safety-aware Bellman operator} defined as,
\begin{equation}
(\mathbf{B}_{\textbf{safe}}Q)(s^{\otimes}, a) :=\begin{cases}
        \mathbb{E}^{\mathcal{M}^\otimes} \left[ \R(s, a, s') + \gamma \max_{a'} Q({s^{\otimes}}', a')\right] & s^{\otimes}\in S^*(a,\epsilon_t)\\
        0 & \text{otherwise}\\
    \end{cases}\label{eq:modifiedbellman}
\end{equation}
where $S^*(a,\epsilon)$ is the maximal state-action probabilistic safe set \eqref{eq:maximal-state-action}. $\mathbf{B}_{\textbf{safe}}$ essentially mimics the effect of action replacement but on value propagation and under reasonable assumptions Proposition \ref{prop:bellmanoperator} establishes that this Bellman operator converges to a Q-function that is reward-optimal under the action restrictions imposed by thr maximal safe set $\S^*(a, \epsilon_t)$. 

\begin{proposition} \label{prop:bellmanoperator}
If (i) $\gamma\in[0,1)$ and (ii) $0\le \R(s,a,s')\le R_{\max}<\infty$ for all $(s, a, s') \in \S \times \A \times \S$, then $\mathbf{B}_{\textbf{safe}}$ is a $\gamma$-contraction under the sup-norm, $\|\cdot\|_{\infty}$, and converges to the unique fixed point $Q^*_{\text{safe}}$, the optimal Q-function subject to the action restrictions imposed by $S^*(a, \epsilon_t)$.
\end{proposition}

In the tabular setting a Q-learning update for $\widehat \pi$ can be derived from a stochastic approximation of the fixed point of $\mathbf{B}_{\textbf{safe}}$. When $\pi_{\text{safe}}$ is also being updated simultaneously, $S^*(a, \epsilon_t)$ must be replaced by $S(\pi, a,\epsilon_t)$ in \eqref{eq:modifiedbellman}. As $\pi_{\text{safe}}$ converges to the maximally safe policy, then $S(\pi, a,\epsilon_t)$ will converge to $S^*(a, \epsilon_t)$ letting the Q-learning update for $\widehat \pi$ truly become an approximation of $\mathbf{B}_{\textbf{safe}}$. This is essentially the high-level reasoning used in the proof of Theorem \ref{thm:convergence} established in the next section.

\subsection{\textsc{RecReg} for Tabular RL} In tabular RL, both the state space $\S$ and action space $\A$ are assumed to be of a manageable size. This setting admits explicit representations of value functions, exact or empirical transition models, and tractable verification of safety over finite horizons via probabilistic model checking \cite{baier2008principles}. We make the following central assumption for practical and theoretical reasons, which establishes that unsafe actions can be detected within a finite-horizon.

\begin{assumption} \label{ass:finite-horizon} There exists some finite horizon $H^* < \infty$ such that if $\mathbb{P}_\pi^a(s^\otimes \models \lozenge\mathrm{unsafe}) \geq \epsilon_t$ then  $\mathbb{P}_\pi^a(s^\otimes \models \lozenge^{\leq{H^*}}\mathrm{unsafe}) \geq \epsilon_t$ (where $\lozenge^{\leq{H^*}}$ denotes bounded reachability).
\end{assumption}

\begin{wrapfigure}{R}{0.5\textwidth}
\vspace*{-2.0\intextsep}
\begin{minipage}{0.5\textwidth}
\begin{algorithm}[H]
\caption{\textsc{RecReg} (Tabular)}
\label{alg:tabular}
\raggedright
\textbf{Input:} $\mathcal{M}$, $\mathcal{D}$ $\epsilon_0,\epsilon_1, \epsilon_2,\ldots$ \\
\textbf{Initialize:} $Q_{\text{task}}$, $Q_{\text{safe}}$, $Q$, $\widehat{\P}$
\begin{algorithmic}
\For{{\bf each episode}}
\State Observe $s^{\otimes}_0$
\For{$t = 0,1,2,\ldots$}
\State Sample action $a_t \sim \widehat{\pi}(\cdot \mid s^{\otimes}_t)$.
\If{$s^\otimes_t \not\in S(\pi, a_t, \epsilon_t)$}
\State Update $Q_\text{task}(s_t^\otimes, a_t)$ via \eqref{eq:modifiedbellman}
\State Resample action $a_t \sim \pi_{\text{safe}}(\cdot \mid s_t^\otimes)$.
\EndIf
\State Execute $a_t$ and observe $(s_t^\otimes, r_t, c_t)$.
\State Update $Q_\text{task}(s_t^\otimes, a_t)$ via \eqref{eq:modifiedbellman}
\State Update $Q_\text{safe}(s_t^\otimes, a_t)$ with $c_t$ and ED \cite{voloshin2023eventual}.
\If{\textsc{ModelBased}}
    \State Update $\widehat{\P}$ with $(s, a, s')$
\ElsIf{\textsc{ModelFree}}
    \State Update $Q(s_t^\otimes, a_t)$ via \eqref{eq:finitiebellman}
\EndIf
\EndFor
\EndFor
\end{algorithmic}
\end{algorithm}
\end{minipage}
\vspace*{-1.0\intextsep}
\end{wrapfigure}

Similar assumptions have been made in prior work \citep{thomas2021safe} which reflect many real-world invariants, such as, a car that has exceeded its minimum stopping distance will inevitably collide with an obstacle in a bounded time, and a submarine that fails to resurface will exhaust oxygen within a finite limit. In practice $H^*$ can be treated as a hyperparameter, increasing it until the desired safety probability is met, or based on prior knowledge of the environment dynamics. Often the maximum distance from initial state to accepting state in the DFA is a good heuristic. Once the minimum $H^*$ is reached, increasing it further does not affect any guarantees, but may incur additional computational overhead.

Both the task policy and safe policy are represented by Q-tables: $Q_{\text{task}}$ and $Q_{\text{safe}}$. At each step, an action proposed by the task policy is either executed or replaced by a backup action according to $\pi^\lozenge$. Any rejected action is used to update $Q_{\text{task}}$ with the masked Bellman operator $\mathbf{B}_{\textbf{safe}}$, accumulating zero value. The executed action is used to update all Q-tables correspondingly. Algorithm \ref{alg:tabular} provides a complete overview of the procedure and we continue by addressing each operation mode individually.

\textbf{Exact.} In \textsc{Exact} mode, the transition kernel $\P$
of is assumed to be known ($\R$ still unknown). The maximally safe policy and maximal probabilistic safe set $S^*(a, \epsilon)$ are computed offline via dynamic programming and probabilistic model checking in the product MDP \citep{baier2008principles}. During learning the task policy can be directly updated with $\mathbf{B}_{\textbf{safe}}$ with $S^*(a, \epsilon_t)$, ensuring optimality and convergence and satisfying the constraint in \eqref{eq:problem} immediately from the beginning of training.

\textbf{Model-based.} In \textsc{ModelBased} mode, the transition kernel is unknown
and estimated via empirical visit counts. For a given dataset of transitions $D = \langle (s_i, a_i, s_i')\rangle^{n}_{i=1}$, let $\widehat \P(s' \mid s, a) = \vert \{ (s, a, s') \in D \} \vert/\vert \{ (s, a) \in D \}\vert$ denote the empirical transition kernel. Under Assumption \ref{ass:finite-horizon} finite horizon reachability probabilities computed with $\widehat \P$ via exact or statistical probabilistic model checking can be used to check $s_t^\otimes \in S(\pi, a,\epsilon_t)$. Proposition \ref{prop:statisticalmc} establishes the conditions for this fact. 

\begin{proposition} \label{prop:statisticalmc} Let $\epsilon' \leq \epsilon_t$ and suppose the total variation distance between $\P$ and $\widehat{\P}$ is bounded by $\epsilon' / H$, i.e., $\|\widehat \P(\cdot\mid s,a)-\P(\cdot\mid s,a)\|_1\le \epsilon'/H^*$ for all $(s,a)$.\footnote{For distributions $\mu_1$ and $\mu_2$ on the finite set $\X$ we define $\Vert \mu_1(\cdot) - \mu_2(\cdot) \Vert_1 := \frac{1}{2}\sum_{x \in \X} \vert \mu_1(x) - \mu_2(x) \vert$} Then exact or statistical model checking with $\widehat{\P}$ allows $(s_t^\otimes, a_t)$ only if $s^\otimes_t \in S(\pi_{\mathrm{safe}},a_t, \epsilon_t)$ with probability (w.p.)~1 or w.p.~at least $(1-\delta)$ respectively.
\end{proposition}

Furthermore to ensure the conditions of Proposition \ref{prop:statisticalmc} are met we can, we can obtain a high probability safety guarantee for the \textsc{ModelBased} mode, based on visit counts for each state-action pair. 

\begin{theorem}\label{thm:safetyguarantees} Let $\epsilon'<\epsilon_t$ and $\delta_t=2\delta$. Under the conditions in Theorem \ref{thm:constraintsatisfaction} and Assumption \ref{ass:finite-horizon}, and if each $(s,a) \in \S \times \A$ is visited at least $O((H^*/\epsilon')^2 |\S| \log(|\A||\S|/\delta))$ times, then by using exact (or statistical) model checking with $\widehat{\P}$, with shielded policy $\pi^\lozenge$ satisfies the step-wise risk level $\epsilon_t$ w.p.~at least $(1-\delta_t)$, and thus satisfies the global safety constraint in \eqref{eq:problem} w.p.~at least $(1 - \sum_{t=0}^{\infty} \delta_t)$.
\end{theorem}

\textbf{Model-free.} In \textsc{ModelFree} mode, the \emph{finite-horizon safety value function} is estimated via,
\begin{equation}
    (\mathbf{B}_\text{fin} Q^{(i)})(s^\otimes, a) :=\mathbb{E}_\pi^{\mathcal{M}^\otimes}\left[ \max\{\mathbf C(s^\otimes, a, {s^\otimes}'), Q^{(i+1)}({s^{\otimes}}', \arg\min Q_\text{safe}({s^\otimes}')) \}\right] \label{eq:finitiebellman}
\end{equation}
for $i < H^*$. Intuitively by applying a stochastic approximation for the Bellman operator $\mathbf{B}_\text{fin}$ we can directly estimate the finite-horizon reachability probability (under $\pi_{\text{safe}}$) by construction of the cost function $\mathbf C$ and the $\max$ operator, avoiding model checking entirely. Then checking $s_t^\otimes \in S(\pi, a,\epsilon_t)$ is implemented as a simple threshold check, i.e., $Q^{(0)}(s^\otimes, a) \leq \epsilon_t$. In Theorem \ref{thm:convergence} we establish unified convergence for all operation modes of \textsc{RecReg}, the statement is generic but complete and broken down for all operation modes in Appendix \ref{sec:convergence-proof}.

\begin{theorem} \label{thm:convergence} Assume: (i) $\gamma\in[0,1)$ and rewards are bounded and non-negative;
(ii) the step-sizes for each Q-approximation satisfy standard stochastic approximation conditions
(e.g., $\sum_t \alpha_t=\infty$ and $\sum_t \alpha_t^2<\infty$); (iii) the step-wise risk levels and initial states satisfy the conditions for Theorem \ref{thm:constraintsatisfaction}; and (iv) Assumption \ref{ass:finite-horizon} holds. Then, in all modes the task Q-iterates ($Q_\textrm{task}$) produced by stochastic approximation of $\mathbf{B}_{\mathbf{safe}}$ converge w.p.~1 to the unique fixed point $Q^*_{\mathrm{safe}}$, i.e., the optimal Q-function subject to the action restrictions imposed by $S^*(a, \epsilon_t)$. And the deployed policy $\pi^\lozenge$ satisfies the global safety constraint in \eqref{eq:problem} w.p.~1.
\end{theorem}

\subsection{\textsc{RecReg} with Deep RL}
\label{sec:deep-recreg}

While the tabular case provides a convenient framework for obtaining concrete guarantees, many interesting applications of RL consider continuous or high-dimensional environments. Results from the tabular setting still hold under the assumption of a suitably accurate dynamics (analogous to the conditions in Proposition \ref{prop:statisticalmc}) \citep{goodall2023approximate, goodall2024leveraging}. To utilize \textsc{RecReg} in these scenarios we adopt SOTA model-based RL, leveraging \textsc{DreamerV3}'s \emph{Recurrent State Space Model} (RSSM) for joint dynamics learning and policy optimization \citep{hafner2019learning}. The RSSM is a sequential variational auto-encoder type model that learns latent states and transition dynamics from observations. The RSSM includes a \emph{recurrent model} $h_t = f_{\theta}(h_{t-1}, z_{t-1}, a_{t-1})$, an \emph{encoder} $z_t \!\sim\! q_{\theta}(z_t \mid o_t, h_t)$, a \emph{prior/transition predictor} $\hat z_t \!\sim\! p_{\theta}(\hat z_t \mid h_t)$, an \emph{observation decoder} $\hat o_t \!\sim\! p_{\theta}(\hat o_t \mid h_t, z_t)$, a \emph{reward head} $\hat r_t \!\sim\! p_{\theta}(\hat r_t \mid h_t, z_t)$, and a \emph{termination head} $\hat \gamma_t \!\sim\! p_{\theta}(\hat \gamma_t \mid h_t, z_t)$. Similar to \emph{approximate model-based shielding} (AMBS) \citep{goodall2023approximate} we augment the RSSM with a \emph{cost head} $\hat c_t \!\sim\! p_{\theta}(\hat c_t \mid h_t, z_t)$ to model the cost function introduced in \eqref{eq:costfn}. Since the RSSM summarizes history in $(h_t, z_t)$ we can accurately learn the cost function in \eqref{eq:costfn} by back-propagating its gradients through the RSSM to capture the temporal dependencies of the DFA without needing to build the product MDP or synchronize states. Both the task policy $\widehat \pi$ and safe policy $\pi_{\text{safe}}$ are implemented as neural networks, with inputs derived from the RSSM latent state. To check if $s^\otimes \in S(\pi_{\text{safe}}, a, \epsilon_t)$ we rollout the latent dynamics $p_{\theta}$ with $a \sim \pi_{\text{task}}$ and $\pi_\text{safe}$ thereafter for a fixed horizon $H^*$, evaluating the cost head as a proxy for DFA satisfaction, and averaging over rollouts in parallel to obtain probability estimates. See Appendix \ref{sec:deep-rl} for further details.

\section{Experimental Evaluation}
\label{sec:experiments}

\begin{figure*}[t]
    \centering
    \includegraphics[width=\linewidth]{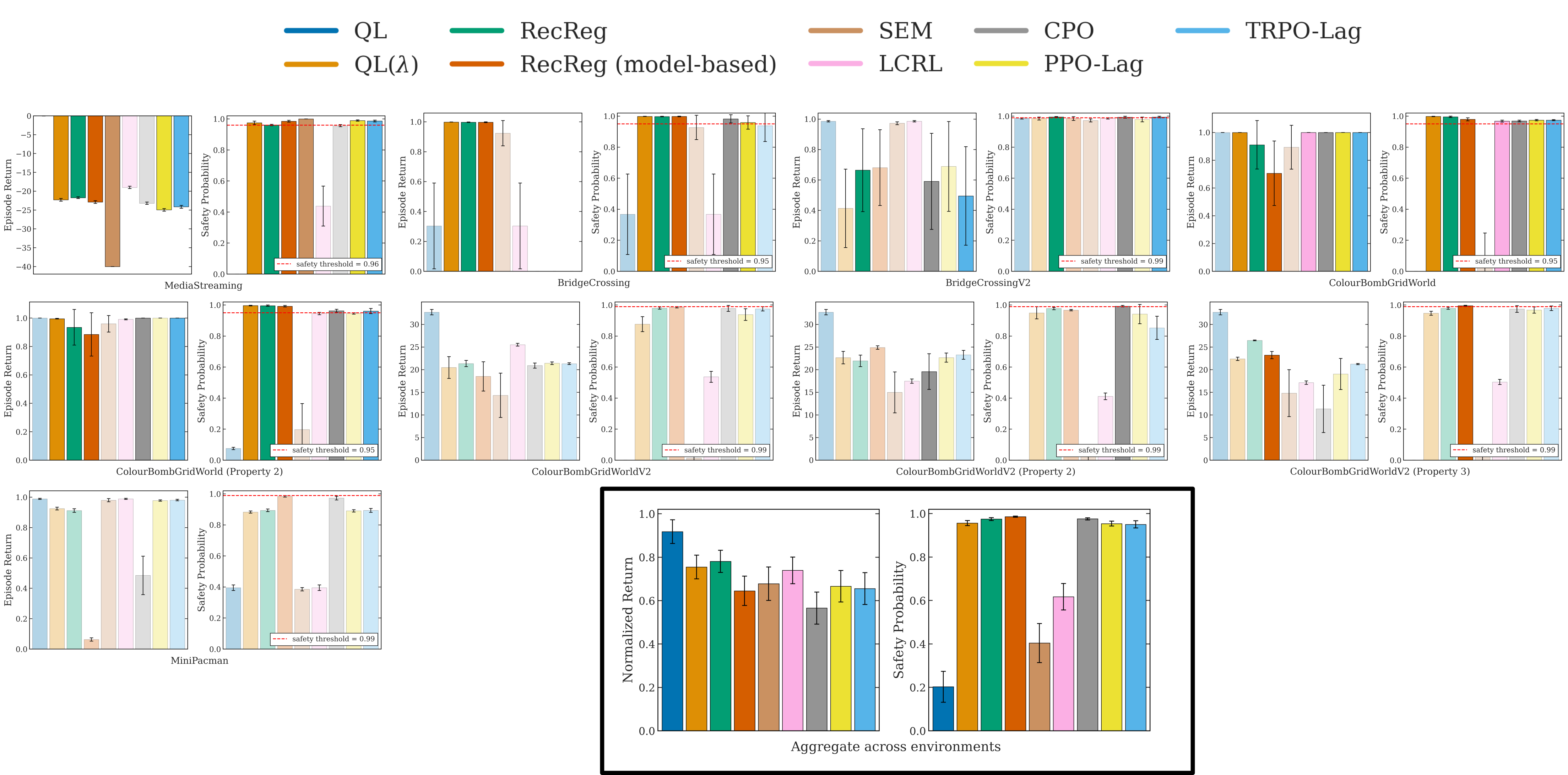}
    \caption{Experimental results for the MASA-Safe-RL \citep{Goodall2025MASASafeRL} benchmark. \emph{Normalized returns} and \emph{safety probabilities} are averaged over 10 random seeds for the final 100 episodes in training. The dashed horizontal line in the safety probability plots denotes the theoretical safety probability given by Theorem \ref{thm:constraintsatisfaction}. Scores aggregated across all environments are provided in the bottom right corner.}
    \label{fig:results}
\end{figure*}


\textbf{Tabular RL.} We consider a suite of tabular environments provided by MASA-Safe-RL \citep{Goodall2025MASASafeRL}, spanning navigation and control tasks, e.g., grid worlds and Pacman. For complete environment descriptions we refer the reader to Appendix \ref{sec:environment-details}. The environments chosen cover a range of safety properties, from simple invariants to more complex regular safety constraints specified via automata. The difficulty varies both in terms of state-space size and the complexity of the underlying safety specification. We compare \textsc{RecReg} against several representative safe RL approaches. Logically-Constrained RL (\textsc{LCRL}) \citep{hasanbeig2018logically} incorporates LTL constraints via reward shaping and maximal end components. \textsc{QL-$\lambda$} follows Lagrangian-based approaches to LTL-constrained RL, treating the penalty coefficient as a hyperparameter \citep{voloshin2022policy, voloshin2023eventual, shah2024deep}. Safe Exploration Mask (\textsc{SEM}) \citep{xuan2022sem} learns auxiliary value functions to estimate proximity to unsafe states and controllability, suppressing risky actions by renormalizing the Q function. The conceptually closest baseline is \textsc{Recovery RL} \citep{thananjeyan2021recovery}, which learns a recovery policy to guide the agent back to safe regions although these are specified via safety critics with respect to the task policy, not the safe policy, leading to sub-optimal and overly conservative action replacement strategies. We also include three representative CMDP-based methods: \textsc{PPO-Lag}, \textsc{TRPO-Lag} \citep{ray2019benchmarking} and \textsc{CPO} \citep{achiam2017constrained}, adapted to the tabular setting, which enforce constraints in expectation via Lagrangian and trust-region updates respectively. The main results are presented in Fig.~\ref{fig:results}, with corresponding tabular summaries and full learning curves provided in Appendix \ref{sec:additional-results}. Across environments, both modes of \textsc{RecReg} consistently achieve the desired safety probability, matching the theoretical constraint in almost all cases. Deviations primarily occur in larger state spaces, where estimation or exploration challenges arise. In contrast, baseline methods exhibit more variable safety performance, reflecting their heuristic nature or reliance on asymptotic guarantees. Even when satisfying the constraint, this often comes at the cost of reduced returns. Overall, \textsc{RecReg} provides a more reliable trade-off between safety and performance. Additional analyses, including runtime (across modes) and ablations (e.g., reward shaping \citep{ng1999policy} and counter factual experiences \citep{icarte2022reward}), are in Appendix \ref{sec:ablation-studies}.

\textbf{Distributing risk in practice.} For \textsc{RecReg}, distributing risk across timesteps is non-trivial. A simple baseline is to allocate risk uniformly, enforcing $\epsilon_t = \epsilon/T$ (where $T$ is the episode length), although this is typically overly conservative. A less naïve approach maintains a residual budget $\epsilon_0 = \epsilon$ and updates it over time; however, directly subtracting the full tail risk (i.e., $\epsilon_t = \epsilon_{t-1} - \mathbb{P}_\pi^a(s_{t-1}^\otimes \models \lozenge \mathrm{unsafe})$) leads to double-counting of future risk. Instead, we subtract only the immediate risk of violation at each step and propagate the remaining budget by conditioning on survival, yielding a more accurate and less conservative update. The immediate and tail risk can be computed exactly (\textsc{Exact}) or estimated (e.g., via \textsc{ModelBased} or \textsc{ModelFree}). More sophisticated approaches exist, such as \textsc{ProbShield}, which allocates risk across successor states and preserves optimality under the global constraint $\epsilon$, but induces a more complex action space that is harder to optimize \citep{court2025probabilisticshieldingsafereinforcement}. Treating $\epsilon_t$ as a free hyperparameter is also possible, although it sacrifices the formal guarantees of Theorem~\ref{thm:constraintsatisfaction}. 

\begin{table}[H]
    \caption{Experimental results for \textsc{ProbShield} and \textsc{RecReg-Exact}. Bold text denotes the best return; grey text denotes algorithms that fail the safety threshold.}
    \label{tab:probshield}
    \centering
    \tiny
    \begin{tabular}{llcccc}
    \toprule
\multirow{2}{*}{Algorithm} & \multirow{2}{*}{Metric} & MediaStreaming & BridgeCrossing & ColourBomb...V2 (Prop.~3) & MiniPacman \\
& & $1-\epsilon=0.96$ & $1-\epsilon=0.95$ & $1-\epsilon=0.99$ & $1-\epsilon=0.99$ \\
\midrule
\multirow{2}{*}{\textsc{RecReg-Exact}} & Ep. Return & \textcolor{gray}{$-21.355\pm 0.240$} & $\mathbf{1.000\pm 0.000}$ & $\mathbf{26.294\pm 0.359}$ & $\mathbf{0.995\pm 0.00}1$ \\
 & Safety Prob. & \textcolor{gray}{$0.951\pm 0.009$} & $1.000\pm 0.000$ & $0.997\pm 0.002$ & $0.997\pm 0.002$ \\
\midrule
\multirow{2}{*}{\textsc{ProbShield}} & Ep. Return & $\mathbf{-25.964\pm 1.040}$ & $0.933\pm 0.067$ & \textcolor{gray}{$19.046\pm 1.490$} & $0.966\pm 0.006$ \\
 & Safety Prob. & $0.990\pm 0.005$ & $0.978\pm 0.011$ & \textcolor{gray}{$0.985\pm 0.006$} & $0.993\pm 0.003$ \\
\bottomrule
\end{tabular}
\end{table}

We directly compare \textsc{Exact-RecReg} with \textsc{ProbShield}, both of which assume full knowledge of the transition dynamics and provide formal safety guarantees. Results are summarized in Table~\ref{tab:probshield}. Both approaches satisfy the safety constraint across all environments (up to empirical error), \textsc{ProbShield} struggles as its augmented action space scales poorly in the number of actions and successor states. In contrast, \textsc{RecReg}'s simpler decision structure achieves stronger returns across the board. 

\begin{wrapfigure}{R}{0.5\textwidth}
\vspace*{-1.0\intextsep}
\begin{minipage}{0.5\textwidth}
\begin{center}
    \includegraphics[width=\linewidth]{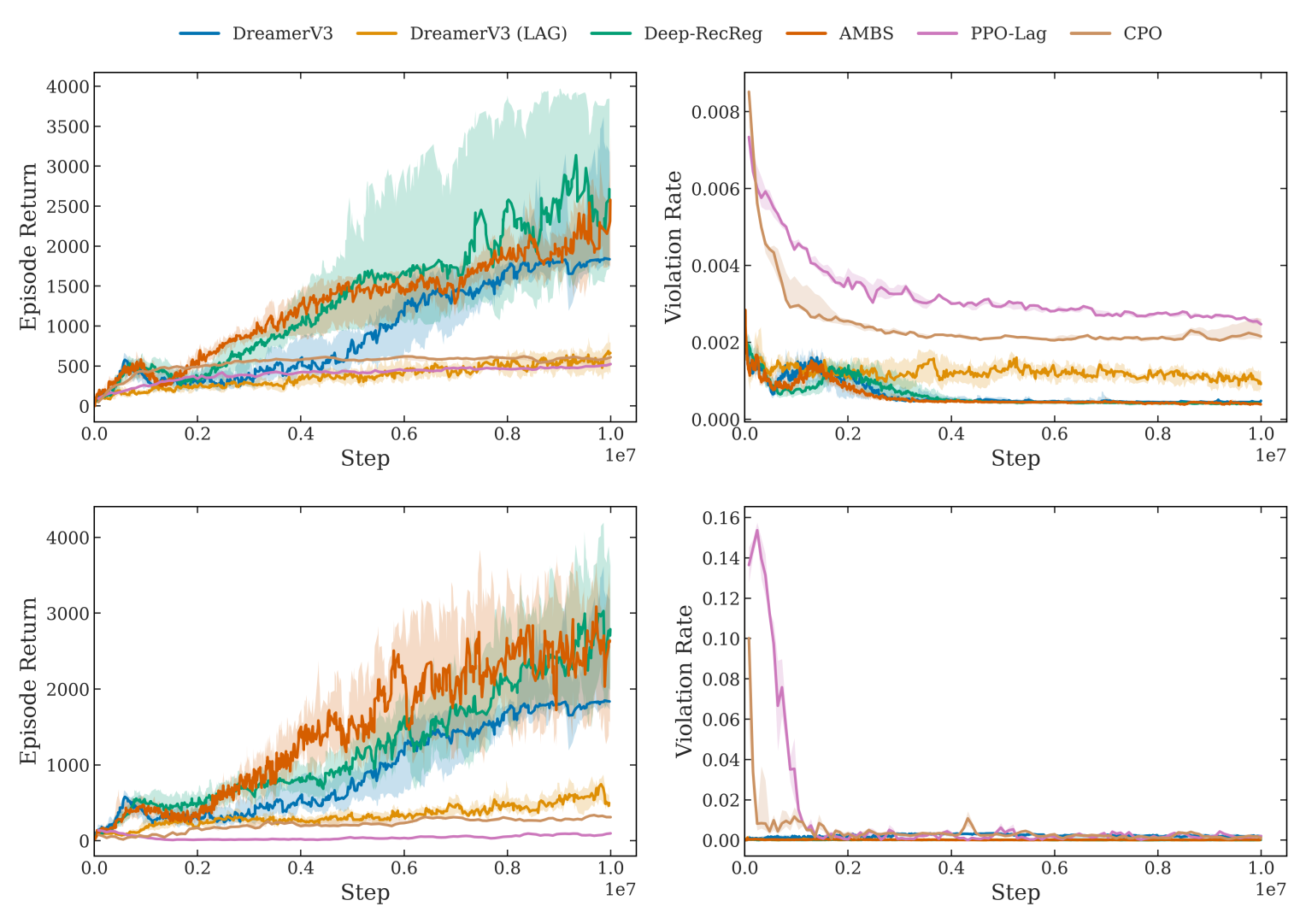}
    \caption{Learning curves for Atari Seaquest. \emph{Episodic reward} and \emph{violation rate} are averaged over 5 random seeds.}
    \label{fig:seaquest-results}
\end{center}
\end{minipage}
\vspace*{-1.0\intextsep}
\end{wrapfigure}

\textbf{Deep RL.} We investigate \textsc{RecReg} with DreamerV3's RSSM \citep{hafner2025mastering} on two high-dimensional benchmarks, Atari Seaquest \citep{machado18arcade} and \textit{SafetyGym} \citep{ray2019benchmarking} (with pixel observations). The goal in Seaquest is to rescue divers while managing a limited supply of oxygen that can only be replenished at the surface. We run two separate experiments with two different safety properties $\varphi_1$ and $\varphi_2$. The first (in LTL) $\varphi_1 := \square((\texttt{surface} \land X \neg \texttt{surface}) \to \square(\texttt{surface} \to \texttt{diver})) \land \square (\neg \texttt{out-of-oxygen}) \land \square(\neg \texttt{hit})$, requires the agent to surface only when carrying a diver, never run out of oxygen, and avoid being hit by enemies (5 DFA-states). The second, $\varphi_2 := \square(\texttt{diver} \to \lozenge^{\leq 30} \texttt{surface}) $, requires that once a diver has been picked up, the agent must return to the surface within 30 steps (31 DFA-states). We compare with two representative baseline algorithms: DreamerV3 with the Augmented Lagrangian \cite{wright2006numerical} and \emph{approximate model-based shielding} (AMBS) \cite{goodall2023approximate}. The learning curves are presented in Figure \ref{fig:seaquest-results}, along with tabular summaries in Appendix \ref{sec:deep-rl}. For both properties $\varphi_1$ and $\varphi_2$, \textsc{RecReg} with DreamerV3 outperforms all the baselines in terms of returns while also achieving a low violation rate. Note that property $\varphi_1$ is closely aligned with ``ideal'' gameplay, thus the unconstrained DreamerV3 also obtains a low violation rate. Qualitative analysis can be found in the supplementary material. Due to space constraints the learning curves and results for \textit{SafetyGym} are in Appendix \ref{sec:deep-rl}.

\section{Conclusion}

We introduced \textsc{RecReg}, a unified approach for enforcing high-probability satisfaction of regular safety properties in RL which, compared to the most closely related literature (e.g., Recovery RL \citep{thananjeyan2021recovery}), provides formal safety and convergence guarantees, with optimality under step-wise constraints. The experimental results presented in this paper validate our theoretical claims, and provide further evidence that the separation of reward and safety objectives is an effective strategy towards safety-aware decision making. Future works includes extending to our method to more RL algorithms and other scenarios (e.g., safe control, policy-gradient).

\bibliographystyle{ACM-Reference-Format} 
\bibliography{refs}

@book{baier2008principles,
  title={Principles of model checking},
  author={Baier, Christel and Katoen, Joost-Pieter},
  year={2008},
  publisher={MIT press}
}

@book{puterman2014markov,
  title={Markov decision processes: discrete stochastic dynamic programming},
  author={Puterman, Martin L},
  year={2014},
  publisher={John Wiley \& Sons}
}

@article{abate2008probabilistic,
  title={Probabilistic reachability and safety for controlled discrete time stochastic hybrid systems},
  author={Abate, Alessandro and Prandini, Maria and Lygeros, John and Sastry, Shankar},
  journal={Automatica},
  volume={44},
  number={11},
  pages={2724--2734},
  year={2008},
  publisher={Elsevier}
}

@inproceedings{alshiekh2018safe,
  title={Safe reinforcement learning via shielding},
  author={Alshiekh, M. and Bloem, R. and Ehlers, R. and K{\"o}nighofer, B. and Niekum, S. and Topcu, U.},
  booktitle={Proceedings of the AAAI Conference on Artificial Intelligence},
  volume={32},
  number={1},
  year={2018}
}

@inproceedings{voloshin2023eventual,
  title={Eventual discounting temporal logic counterfactual experience replay},
  author={Voloshin, Cameron and Verma, Abhinav and Yue, Yisong},
  booktitle={International Conference on Machine Learning},
  pages={35137--35150},
  year={2023},
  organization={PMLR}
}

@inproceedings{carr2023safe,
  title={Safe reinforcement learning via shielding under partial observability},
  author={Carr, Steven and Jansen, Nils and Junges, Sebastian and Topcu, Ufuk},
  booktitle={Proceedings of the AAAI Conference on Artificial Intelligence},
  volume={37},
  number={12},
  pages={14748--14756},
  year={2023}
}

@article{shah2024deep,
  title={Deep Policy Optimization with Temporal Logic Constraints},
  author={Shah, Ameesh and Voloshin, Cameron and Yang, Chenxi and Verma, Abhinav and Chaudhuri, Swarat and Seshia, Sanjit A},
  journal={arXiv preprint arXiv:2404.11578},
  year={2024}
}

@inproceedings{court2025probabilisticshieldingsafereinforcement,
  title={Probabilistic Shielding for Safe Reinforcement Learning},
  author={Hamel-De le Court, Edwin and Belardinelli, Francesco and Goodall, Alexander W},
  booktitle={Proceedings of the AAAI Conference on Artificial Intelligence},
  volume={39},
  number={15},
  pages={16091--16099},
  year={2025}
}

@article{jaakkola1993convergence,
  title={Convergence of stochastic iterative dynamic programming algorithms},
  author={Jaakkola, Tommi and Jordan, Michael and Singh, Satinder},
  journal={Advances in neural information processing systems},
  volume={6},
  year={1993}
}

@article{weissman2003inequalities,
  title={Inequalities for the L1 deviation of the empirical distribution},
  author={Weissman, Tsachy and Ordentlich, Erik and Seroussi, Gadiel and Verdu, Sergio and Weinberger, Marcelo J},
  journal={Hewlett-Packard Labs, Tech. Rep},
  pages={125},
  year={2003}
}

@article{thomas2021safe,
  title={Safe reinforcement learning by imagining the near future},
  author={Thomas, Garrett and Luo, Yuping and Ma, Tengyu},
  journal={Advances in Neural Information Processing Systems},
  volume={34},
  pages={13859--13869},
  year={2021}
}

@article{hasanbeig2020cautious,
  title={Cautious reinforcement learning with logical constraints},
  author={Hasanbeig, Mohammadhosein and Abate, Alessandro and Kroening, Daniel},
  journal={arXiv preprint arXiv:2002.12156},
  year={2020}
}

@article{voloshin2022policy,
  title={Policy optimization with linear temporal logic constraints},
  author={Voloshin, Cameron and Le, Hoang and Chaudhuri, Swarat and Yue, Yisong},
  journal={Advances in Neural Information Processing Systems},
  volume={35},
  pages={17690--17702},
  year={2022}
}

@article{ray2019benchmarking,
  title={Benchmarking safe exploration in deep reinforcement learning},
  author={Ray, Alex and Achiam, Joshua and Amodei, Dario},
  journal={arXiv preprint arXiv:1910.01708},
  volume={7},
  number={1},
  pages={2},
  year={2019}
}

@inproceedings{achiam2017constrained,
  title={Constrained policy optimization},
  author={Achiam, Joshua and Held, David and Tamar, Aviv and Abbeel, Pieter},
  booktitle={International conference on machine learning},
  pages={22--31},
  year={2017},
  organization={PMLR}
}

@article{amodei2016concrete,
  title={Concrete problems in AI safety},
  author={Amodei, Dario and Olah, Chris and Steinhardt, Jacob and Christiano, Paul and Schulman, John and Man{\'e}, Dan},
  journal={arXiv preprint arXiv:1606.06565},
  year={2016}
}

@article{chow2018risk,
  title={Risk-constrained reinforcement learning with percentile risk criteria},
  author={Chow, Yinlam and Ghavamzadeh, Mohammad and Janson, Lucas and Pavone, Marco},
  journal={Journal of Machine Learning Research},
  volume={18},
  number={167},
  pages={1--51},
  year={2018}
}

@article{liang2018accelerated,
  title={Accelerated primal-dual policy optimization for safe reinforcement learning},
  author={Liang, Qingkai and Que, Fanyu and Modiano, Eytan},
  journal={arXiv preprint arXiv:1802.06480},
  year={2018}
}

@article{tessler2018reward,
  title={Reward constrained policy optimization},
  author={Tessler, Chen and Mankowitz, Daniel J and Mannor, Shie},
  journal={arXiv preprint arXiv:1805.11074},
  year={2018}
}

@article{yang2020projection,
  title={Projection-based constrained policy optimization},
  author={Yang, Tsung-Yen and Rosca, Justinian and Narasimhan, Karthik and Ramadge, Peter J},
  journal={arXiv preprint arXiv:2010.03152},
  year={2020}
}

@article{hasanbeig2018logically,
  title={Logically-constrained reinforcement learning},
  author={Hasanbeig, Mohammadhosein and Abate, Alessandro and Kroening, Daniel},
  journal={arXiv preprint arXiv:1801.08099},
  year={2018}
}

@inproceedings{de2020restraining,
  title={Restraining bolts for reinforcement learning agents},
  author={De Giacomo, Giuseppe and Iocchi, Luca and Favorito, Marco and Patrizi, Fabio},
  booktitle={Proceedings of the AAAI Conference on Artificial Intelligence},
  volume={34},
  number={09},
  pages={13659--13662},
  year={2020}
}

@book{altman1999constrained,
  title={Constrained Markov decision processes: stochastic modeling},
  author={Altman, Eitan},
  year={1999},
  publisher={Routledge}
}

@article{icarte2022reward,
  title={Reward machines: Exploiting reward function structure in reinforcement learning},
  author={Icarte, Rodrigo Toro and Klassen, Toryn Q and Valenzano, Richard and McIlraith, Sheila A},
  journal={Journal of Artificial Intelligence Research},
  volume={73},
  pages={173--208},
  year={2022}
}

@inproceedings{icarte2018using,
  title={Using reward machines for high-level task specification and decomposition in reinforcement learning},
  author={Icarte, Rodrigo Toro and Klassen, Toryn and Valenzano, Richard and McIlraith, Sheila},
  booktitle={International Conference on Machine Learning},
  pages={2107--2116},
  year={2018},
  organization={PMLR}
}

@inproceedings{jansen2020safe,
  title={Safe reinforcement learning using probabilistic shields},
  author={Jansen, Nils and K{\"o}nighofer, Bettina and Junges, Sebastian and Serban, Alex and Bloem, Roderick},
  booktitle={31st International Conference on Concurrency Theory (CONCUR 2020)},
  year={2020},
  organization={Schloss-Dagstuhl-Leibniz Zentrum f{\"u}r Informatik}
}

@inproceedings{konighofer2021online,
  title={Online shielding for stochastic systems},
  author={K{\"o}nighofer, Bettina and Rudolf, Julian and Palmisano, Alexander and Tappler, Martin and Bloem, Roderick},
  booktitle={NASA Formal Methods Symposium},
  pages={231--248},
  year={2021},
  organization={Springer}
}

@article{konighofer2023online,
  title={Online shielding for reinforcement learning},
  author={K{\"o}nighofer, Bettina and Rudolf, Julian and Palmisano, Alexander and Tappler, Martin and Bloem, Roderick},
  journal={Innovations in Systems and Software Engineering},
  volume={19},
  number={4},
  pages={379--394},
  year={2023},
  publisher={Springer}
}

@inproceedings{bastanib2021safe,
  title={Safe Reinforcement Learning via Statistical Model Predictive Shielding.},
  author={Bastani, Osbert and Li, Shuo and Xu, Anton},
  booktitle={Robotics: Science and Systems},
  pages={1--13},
  year={2021}
}

@article{as2022constrained,
  title={Constrained policy optimization via bayesian world models},
  author={As, Yarden and Usmanova, Ilnura and Curi, Sebastian and Krause, Andreas},
  journal={arXiv preprint arXiv:2201.09802},
  year={2022}
}

@incollection{goodall2023approximate,
  title={Approximate Model-Based Shielding for Safe Reinforcement Learning},
  author={Goodall, Alexander W and Belardinelli, Francesco},
  booktitle={ECAI 2023},
  pages={883--890},
  year={2023},
  publisher={IOS Press}
}

@article{he2022androids,
  title={Do Androids Dream of Electric Fences? Safety-Aware Reinforcement Learning with Latent Shielding},
  author={He, Chloe and Le{\'o}n, Borja G and Belardinelli, Francesco},
  url = {https://ceur-ws.org/Vol-3087/paper_50.pdf},
  organization={CEUR Workshop Proceedings},
  year={2022}
}

@inproceedings{xiao2023model,
  title={Model-based Dynamic Shielding for Safe and Efficient Multi-agent Reinforcement Learning},
  author={Xiao, Wenli and Lyu, Yiwei and Dolan, John},
  booktitle={Proceedings of the 2023 International Conference on Autonomous Agents and Multiagent Systems},
  pages={1587--1596},
  year={2023}
}

@article{gronauer2024reinforcement,
  title={Reinforcement Learning with Ensemble Model Predictive Safety Certification},
  author={Gronauer, Sven and Haider, Tom and da Roza, Felippe Schmoeller and Diepold, Klaus},
  journal={arXiv preprint arXiv:2402.04182},
  year={2024}
}

@article{thananjeyan2021recovery,
  title={Recovery rl: Safe reinforcement learning with learned recovery zones},
  author={Thananjeyan, Brijen and Balakrishna, Ashwin and Nair, Suraj and Luo, Michael and Srinivasan, Krishnan and Hwang, Minho and Gonzalez, Joseph E and Ibarz, Julian and Finn, Chelsea and Goldberg, Ken},
  journal={IEEE Robotics and Automation Letters},
  volume={6},
  number={3},
  pages={4915--4922},
  year={2021},
  publisher={IEEE}
}

@Article{machado18arcade,
    author = {Marlos C. Machado and Marc G. Bellemare and Erik Talvitie and Joel Veness and Matthew J. Hausknecht and Michael Bowling},
    title = {Revisiting the Arcade Learning Environment: Evaluation Protocols and Open Problems for General Agents},
    journal = {Journal of Artificial Intelligence Research},
    volume = {61},
    pages = {523--562},
    year = {2018}
}

@inproceedings{ng1999policy,
  title={Policy invariance under reward transformations: Theory and application to reward shaping},
  author={Ng, Andrew Y and Harada, Daishi and Russell, Stuart},
  booktitle={Icml},
  volume={99},
  pages={278--287},
  year={1999}
}

@software{jax2018github,
  author = {James Bradbury and Roy Frostig and Peter Hawkins and Matthew James Johnson and Chris Leary and Dougal Maclaurin and George Necula and Adam Paszke and Jake Vander{P}las and Skye Wanderman-{M}ilne and Qiao Zhang},
  title = {{JAX}: composable transformations of {P}ython+{N}um{P}y programs},
  url = {http://github.com/google/jax},
  version = {0.3.13},
  year = {2018},
}

@article{omnisafe,
  title   = {OmniSafe: An Infrastructure for Accelerating Safe Reinforcement Learning Research},
  author  = {Ji, Jiaming and Zhou, Jiayi and Zhang, Borong and Dai, Juntao and Pan, Xuehai and Sun, Ruiyang and Huang, Weidong and Geng, Yiran and Liu, Mickel and Yang, Yaodong},
  journal = {arXiv preprint arXiv:2305.09304},
  year    = {2023}
}

@article{hafner2023mastering,
  title={Mastering diverse domains through world models},
  author={Hafner, Danijar and Pasukonis, Jurgis and Ba, Jimmy and Lillicrap, Timothy},
  journal={arXiv preprint arXiv:2301.04104},
  year={2023}
}

@misc{wright2006numerical,
  title={Numerical optimization},
  author={Wright, Jorge Nocedal Stephen J},
  year={2006},
  publisher={springer publication}
}

@inproceedings{hafner2019learning,
  title={Learning latent dynamics for planning from pixels},
  author={Hafner, Danijar and Lillicrap, Timothy and Fischer, Ian and Villegas, Ruben and Ha, David and Lee, Honglak and Davidson, James},
  booktitle={International conference on machine learning},
  pages={2555--2565},
  year={2019},
  organization={PMLR}
}

@article{hoeffding1963probability,
  title={Probability Inequalities for Sums of Bounded Random Variables},
  author={Hoeffding, Wassily},
  journal={Journal of the American Statistical Association},
  volume={58},
  number={301},
  pages={13--30},
  year={1963},
  publisher={Taylor \& Francis}
}

@inproceedings{rajeswaran2020game,
  title={A game theoretic framework for model based reinforcement learning},
  author={Rajeswaran, Aravind and Mordatch, Igor and Kumar, Vikash},
  booktitle={International conference on machine learning},
  pages={7953--7963},
  year={2020},
  organization={PMLR}
}

@article{kearns2002near,
  title={Near-optimal reinforcement learning in polynomial time},
  author={Kearns, Michael and Singh, Satinder},
  journal={Machine learning},
  volume={49},
  pages={209--232},
  year={2002},
  publisher={Springer}
}

@book{durrett1999essentials,
  title={Essentials of stochastic processes},
  author={Durrett, Richard and Durrett, R},
  volume={1},
  year={1999},
  publisher={Springer}
}

@article{xuan2022sem,
  title={SEM: Safe exploration mask for q-learning},
  author={Xuan, Chengbin and Zhang, Feng and Lam, Hak-Keung},
  journal={Engineering Applications of Artificial Intelligence},
  volume={111},
  pages={104765},
  year={2022},
  publisher={Elsevier}
}

@inproceedings{pnueli1989synthesis,
  title={On the synthesis of a reactive module},
  author={Pnueli, Amir and Rosner, Roni},
  booktitle={Proceedings of the 16th ACM SIGPLAN-SIGACT symposium on Principles of programming languages},
  pages={179--190},
  year={1989}
}

@inproceedings{bloem2015shield,
  title={Shield synthesis: Runtime enforcement for reactive systems},
  author={Bloem, Roderick and K{\"o}nighofer, Bettina and K{\"o}nighofer, Robert and Wang, Chao},
  booktitle={International conference on tools and algorithms for the construction and analysis of systems},
  pages={533--548},
  year={2015},
  organization={Springer}
}

@misc{Goodall2025MASASafeRL,
  title        = {{MASA-Safe-RL}: Multi and Single Agent Safe Reinforcement Learning},
  author       = {Goodall, Alexander W. and Adalat, Omar and Hamel De-le Court, Edwin and Belardinelli, Francesco},
  year         = {2025},
  howpublished = {\url{https://github.com/sacktock/MASA-Safe-RL/}},
  note         = {GitHub repository}
}

@inproceedings{todorov2012mujoco,
  title={Mujoco: A physics engine for model-based control},
  author={Todorov, Emanuel and Erez, Tom and Tassa, Yuval},
  booktitle={2012 IEEE/RSJ international conference on intelligent robots and systems},
  pages={5026--5033},
  year={2012},
  organization={IEEE}
}

@inproceedings{seto1998simplex,
  title={The simplex architecture for safe online control system upgrades},
  author={Seto, Danbing and Krogh, Bruce and Sha, Lui and Chutinan, Alongkrit},
  booktitle={Proceedings of the 1998 American Control Conference. ACC (IEEE Cat. No. 98CH36207)},
  volume={6},
  pages={3504--3508},
  year={1998},
  organization={IEEE}
}

@article{goodall2026safe,
  title={Safe Reinforcement Learning via Recovery-based Shielding with Gaussian Process Dynamics Models},
  author={Goodall, Alexander W and Belardinelli, Francesco},
  journal={arXiv preprint arXiv:2602.12444},
  year={2026}
}

@article{hafner2025mastering,
  title={Mastering diverse control tasks through world models},
  author={Hafner, Danijar and Pasukonis, Jurgis and Ba, Jimmy and Lillicrap, Timothy},
  journal={Nature},
  volume={640},
  number={8059},
  pages={647--653},
  year={2025},
  publisher={Nature Publishing Group UK London}
}

@article{manganaris2025automaton,
  title={Automaton Constrained Q-Learning},
  author={Manganaris, Anastasios and Giammarino, Vittorio and Qureshi, Ahmed H},
  journal={arXiv preprint arXiv:2510.05061},
  year={2025}
}

@article{yang2021tractability,
  title={On the (in) tractability of reinforcement learning for LTL objectives},
  author={Yang, Cambridge and Littman, Michael and Carbin, Michael},
  journal={arXiv preprint arXiv:2111.12679},
  year={2021}
}

@article{krasowski2023provably,
  title={Provably safe reinforcement learning: Conceptual analysis, survey, and benchmarking},
  author={Krasowski, Hanna and Thumm, Jakob and M{\"u}ller, Marlon and Sch{\"a}fer, Lukas and Wang, Xiao and Althoff, Matthias},
  journal={Transactions on Machine Learning Research},
  year={2023}
}

@inproceedings{goodall2024leveraging,
  title={Leveraging Approximate Model-based Shielding for Probabilistic Safety Guarantees in Continuous Environments},
  author={Goodall, Alexander W and Belardinelli, Francesco},
  booktitle={Proceedings of the 23rd International Conference on Autonomous Agents and Multiagent Systems},
  pages={2291--2293},
  year={2024}
}

@article{banach1922operations,
  title={Sur les op{\'e}rations dans les ensembles abstraits et leur application aux {\'e}quations int{\'e}grales},
  author={Banach, Stefan},
  journal={Fundamenta mathematicae},
  volume={3},
  number={1},
  pages={133--181},
  year={1922},
  publisher={Polska Akademia Nauk. Instytut Matematyczny PAN}
}

@article{towers2024gymnasium,
  title={Gymnasium: A standard interface for reinforcement learning environments},
  author={Towers, Mark and Kwiatkowski, Ariel and Terry, Jordan and Balis, John U and De Cola, Gianluca and Deleu, Tristan and Goul{\~a}o, Manuel and Kallinteris, Andreas and Krimmel, Markus and KG, Arjun and others},
  journal={arXiv preprint arXiv:2407.17032},
  year={2024}
}


\clearpage
\newpage
\appendix

\section{Proofs}
\label{sec:proofs}

\subsection{Theorem 10.52 in \citep{baier2008principles}}
\label{sec:reachability-proof}

\begin{theorem}[[Variant of Theorem 10.52 in \cite{baier2008principles}] For a fixed policy $\pi$ from state $s \in \S$, the safety probability under the path measure $\Pr^\mathcal{M}_\pi$,
\begin{equation*}
    \mathbb{P}_\pi(s \models P_{\mathrm{safe}}) = 1 - \mathbb{P}_\pi(s^\otimes \models \lozenge \mathrm{unsafe})
\end{equation*}
where $s^\otimes = (s, \Delta(\Q_0, L(s)))$ and,
\begin{equation*}
    \mathbb{P}_\pi(s^\otimes \models \lozenge \mathrm{unsafe}) := \textstyle{\Pr^{\mathcal{M}^\otimes}_\pi}\left(\{ \rho^\otimes \in (\S^\otimes)^\omega 
    \mid \rho^\otimes[0] = s^\otimes ,\ \exists\, t \geq 0,\ 
    \mathrm{unsafe} \in \L^\otimes(\rho^\otimes[t]) \}\right)
\end{equation*}
where $\Pr^{\mathcal{M}^\otimes}_\pi$ is the path measure in the product MDP (c.f., Definition \ref{def:product-mdp}) under policy $\pi$. 
\end{theorem}

\begin{proof}
For a given DFA $\mathcal{D} = (\Q \Sigma, \Delta, \Q_0, \F)$ we start by recursively defining the extended transition function $\Delta^{*} :\Q \times \Sigma^* \to \Q$ as $\Delta^*(q, w) = \Delta( \Delta^*(q, w \setminus w\!\!\downarrow), w\!\!\downarrow  )$. The language accepted by the DFA is then defined as,
\begin{equation*}
    \mathcal{L}(\mathcal{D}) :=  \{ w \in \Sigma^{*} \mid \Delta^*(\Q_0, w) \in \F \}
\end{equation*}

Now let $\mathcal{D}$ be the safety property whose language is precisely the set of bad prefixes for $P_{\mathrm{safe}}$, then formally,
\begin{equation*}
    P_{\text{safe}} := \{ w \in \Sigma^{\omega} \mid \forall w_{\text{pref}} \in \Sigma^{\omega}, w_{\text{pref}} \preceq w, w_{\text{pref}} \not \in \mathcal{L}(\mathcal{D}) \}
\end{equation*}

We continue with the following auxiliary results.
\begin{corollary} \label{cor:unqiuerun}
    For any path $\rho = s_0s_1\ldots \in \S^{\omega}$, there exists a unique run $q_0q_1\ldots \in \Q^{\omega}$ for the trace $\tau(\rho) = L(s_0)L(s_1)\ldots \in \Sigma^{\omega}$, denote,
\begin{equation*}
    \rho^{\otimes} = (s_0,q_0)(s_1,q_1)\ldots \in (\S \times \Q)^{\omega} \;(\text{or equiv.} \in (S^{\otimes})^{\omega})
\end{equation*}
where the $(s_0, q_0) = (s_0, \Delta(\Q_0, \L(s_0)))$ and $(s_t, q_t) = (s_t, \Delta(q_{t-1}, \L(s_t)))$ for all $t = 1, 2, \ldots$. This is easy to see because the transition function $\Delta$ is deterministic.
\end{corollary}

\begin{lemma}  Consider a fixed path $\rho \in \S^{\omega}$ and corresponding run $q_0q_1\ldots \in \Q^{\omega}$, trace $\tau(\rho) \in \Sigma^{\omega}$ and path $\rho^{\otimes} \in (S^{\otimes})^{\omega}$ (defined in Corollary \ref{cor:unqiuerun}), then,
\begin{equation*}
    \rho \not \models P_{\text{safe}} \quad \text{if and only if} \quad \rho^{\otimes} \models \lozenge \text{unsafe} 
\end{equation*}
\end{lemma}

\begin{proof}
    ($\Rightarrow$). By definition $\rho \not \models P_{\text{safe}}$ if and only if $\tau(\rho) \in \mathcal{L}(\mathcal{D})$. Recall the definition $\mathcal{L}(\mathcal{D}) = \{ w \in \Sigma^{*} \mid \Delta^*(\Q_0, w) \in \F \}$, thus $\tau(\rho) \in \mathcal{L}(\mathcal{D})$ implies that there exists some $t = 0,1,\ldots$ such that $q_t \in \F$, then by construction (c.f.~ Definition ~\ref{def:product-mdp}) this implies that $\rho^{\otimes} \models \lozenge \text{unsafe}$.

    ($\Leftarrow$). By LTL semantics and Definition \ref{def:product-mdp} (product MDP), $\rho^{\otimes} \models \lozenge \text{unsafe}$ implies that for the unique run $q_0q_1\ldots \subset \rho^{\otimes}$ there exists $t = 0,1,\ldots$ such that $q_t \in \F$. We notice that since $\mathcal{L}(\mathcal{D}) = \textit{BadPref}(P_{\textit{safe}})$ then once the DFA reaches an accepting state it will remain in an accepting state for the rest of the run, if not we can derive a contradiction (e.g., if there is an unsafe prefix for a trace not in $\mathcal{L}(\mathcal{D})$ then this violates the definition of a safety property and thus $\mathcal{L}(\mathcal{D}) \neq \textit{BadPref}(P_{\textit{safe}})$). Therefore, if there exists $t=0,1,\ldots$ such that $q_t \in \F$ then $\tau(\rho) \in \mathcal{L}(\mathcal{D})$, which implies that $\rho \not \models P_{\text{safe}}$.
\end{proof}

Noticing that the DFA $\mathcal{D}$ is deterministic and does not affect the path probabilities in the product Markov chain $\mathcal{M}^{\otimes}_{\pi}$ we supply the following Corollary. 

\begin{corollary} For every measurable set of paths $P$ in $\mathcal{M}$, 
    \begin{equation*}
    \textstyle{\Pr_{\pi}^{\mathcal{M}}}(P) = \textstyle{\Pr_{\pi}^{\mathcal{M}^{\otimes}}}(\{\rho^{\otimes} \mid \rho \subset \rho^{\otimes}, \rho \in P\})
\end{equation*}
where $\rho^\otimes$ is defined as before (see Corollary \ref{cor:unqiuerun}). The proof can be found in \cite{baier2008principles}.
\end{corollary}

It remains to build the appropriate $P$. Let $P$ be the set of paths starting in some $s \in \S$ that refute $P_{\text{safe}}$, i.e.,
\begin{equation*}
    P := \{\rho \in \S^{\omega} \mid \rho[0] = s, \{w_{\text{pref}} \in \Sigma^{\omega} \mid w_{\text{pref}} \preceq \tau(\rho) \} \cap \mathcal{L}(\mathcal{D}) \neq \varnothing\}
\end{equation*}
and $P^{\otimes}$ is the set of paths starting in $s^{\otimes} = (s, \Delta(\Q_0, L(s)))$ that eventually reach an accepting (`unsafe' labelled) state in the DFA $\mathcal{D}$, i.e.,
\begin{equation*}
    P^{\otimes} := \{ \rho^{\otimes} \in \S^{\otimes} \mid \rho^{\otimes}[0] =  s^{\otimes}, \rho^{\otimes} \models \lozenge \text{unsafe} \}
\end{equation*}
then by construction (of $P$ and $P^{\times}$) we have,
\begin{equation*}
    \textstyle{\Pr^{\mathcal{M}}_{\pi}}(P) = \textstyle{\Pr^{\mathcal{M}^{\otimes}}_{\pi}}(\{\rho^{\otimes} \mid \rho[0] = s, \rho \in P\}) = \textstyle{\Pr^{\mathcal{M}^{\otimes}}_{\pi}}(P^{\otimes})
\end{equation*}
Finally, the probability $\Pr^{\mathcal{M}}_{\pi}(P)$ and $\Pr^{\mathcal{M}}_{\pi}(s \models P_{\text{safe}})$ are related as follows,
\begin{align*}
    \textstyle{\Pr^{\mathcal{M}}_{\pi}}(s \models P_{\text{safe}}) &= 1 - \textstyle{\Pr_{\pi}^{\mathcal{M}}}(P)\\
    &= 1 - \textstyle{\Pr^{\mathcal{M}^{\otimes}}_{\pi}}(P^{\times})\\
    & = 1 -  \textstyle{\Pr^{\mathcal{M}^{\otimes}}_{\pi}}(s^{\otimes} \models \lozenge \text{unsafe})\\
\end{align*}
This completes the proof.
\end{proof}

\subsection{Proof of Proposition \ref{prop:maximally}}

\begin{restatedprop}{\ref{prop:maximally} (restated)} The optimal policy $\pi^*$ that minimizes the expected return under eventual discounting \cite{voloshin2023eventual} and with rewards given by \eqref{eq:costfn}, is such that $S(\pi^*, \epsilon) = S^*(\epsilon)$ for any $\epsilon \in (0, 1)$ (is maximally safe).     
\end{restatedprop}

\begin{proof} 
We start by proving the existence of a non-empty maximally safe set $S^*(\epsilon)$ for all $\epsilon\in [0, 1]$, under the assumption that there exists at least one policy with non-zero safety probability.

Consider the transition kernel $\P^{\otimes}_{\pi}({s^{\otimes}}' \mid s^{\otimes}) := \sum_{a \in \A} \P^{\otimes}({s^{\otimes}}'\mid s^{\otimes}, a)\pi(a \mid s^{\otimes})$ for a fixed policy $\pi \in \Pi^\otimes$ in the product MDP $\mathcal{M}^\otimes$ (c.f., Definition \ref{def:product-mdp}). Standard Markov chain results imply that the Markov chain induced by $\P^{\otimes}_{\pi}$ partitions the state space $S^\otimes$ into a disjoint union of transient states $\mathcal{T}_{\pi}$ (visited finitely many times) and $k$ irreducible recurrent classes $\mathcal{R}^1_{\pi},\mathcal{R}^2_{\pi}, \ldots \mathcal{R}^k_{\pi}$ \cite{durrett1999essentials}, i.e. $\S^{\otimes} = \mathcal{T}_{\pi} \cup \mathcal{R}^1_{\pi} \cup \mathcal{R}^2_{\pi} \ldots \cup \mathcal{R}^k_{\pi}$. We also denote $\F^{\otimes} \subseteq \S^{\otimes}$ as the set of states labelled as ``unsafe'' in the product MDP $\mathcal{M}^{\otimes}$, i.e., $\F^{\otimes} = \{ s^{\otimes} \in \S^{\otimes}\mid \text{unsafe} \in L^{\otimes}(s^{\otimes})\}$ (see Definition \ref{def:product-mdp}). Assuming there exists at least one policy with non-zero safety probability, i.e.,$\exists \epsilon> 0,\: \pi'\in \Pi^{\otimes} \;\text{s.t.}\; S(\pi', \epsilon) \neq \varnothing$, then we can conclude that, 
\begin{equation*}
    \exists \mathcal{R}^i_{\pi'}\; \text{s.t.}\; \F^{\otimes} \cap \mathcal{R}^i_{\pi'} = \varnothing
\end{equation*}
In words, there exists at least one policy $\pi'$ such that there is a recurrent set $\mathcal{R}^i_{\pi'}$ with no ``unsafe'' states. Thus $S(\pi_{\text{safe}}, 0) \neq \varnothing$ which implies $S(\pi_{\text{safe}}, \epsilon) \neq \varnothing$ for all $\epsilon \in [0, 1]$. Since $S(\pi, \epsilon) \subseteq S^*(\epsilon)$ for all $\pi \in \Pi^{\otimes}$ and $\epsilon\in [0, 1]$ then $S^*(\epsilon) \neq \varnothing$ for all $\epsilon \in [0, 1]$ by our earlier observation. This completes the first part of the proof.

We continue by showing the optimal policy under eventual discounting is also maximally safe. First we define the optimal policy $\pi^*$ under the eventual-discounted return \cite{voloshin2023eventual}; discounting only on accepting product states (i.e., product states in $\F^\otimes$). 
\begin{equation*}
    \pi^*(s^{\otimes}) := \arg\inf_{\pi \in \Pi^{\otimes}}\mathbb{E}^{\mathcal{M}^\otimes}_\pi\left[ \sum^{\infty}_{t=0} \Gamma_t \mathbf C(s^{\otimes}, a, {s^{\otimes}}') \mid s_0^{\otimes} = s^{\otimes} \right] \quad \text{(for all $s^{\otimes} \in \S^{\otimes}$)}
\end{equation*}
where,
\begin{equation*}
 \gamma_t =
\begin{cases}
    \gamma & \text{if $\text{unsafe} \in L^{\otimes}(s^{\otimes})$,}\\
    1 &\text{otherwise.}
\end{cases} \quad \text{and} \quad \Gamma_t = \prod^{t-1}_{i=0} \gamma^i, \, \Gamma_0 = 1
\end{equation*}
We can rewrite this in terms of the expected eventual-discounted return along ``safe'' and ``unsafe'' paths,
\begin{align*}
    \pi^*(s^{\otimes}) & = \arg\inf_{\pi\in \Pi^{\otimes}} \Bigg\{
     \mathbb{E}^{\mathcal{M}^\otimes}_{\pi} \left[ \sum^{\infty}_{t=0} \Gamma_t \mathbf C(s^{\otimes}_t, a_t, s^{\otimes}_{t+1}) \mid \rho^{\otimes}[0] = s^{\otimes}, \rho^{\otimes} \not \models \lozenge\text{unsafe}\right] \times \mathbb{P}_\pi(s^{\otimes} \not \models \lozenge \text{unsafe})\\
    &\quad + \mathbb{E}^{\mathcal{M}^\otimes}_{\pi} \left[ \sum^{\infty}_{t=0} \Gamma_t \mathbf C(s^{\otimes}_t, a_t, s^{\otimes}_{t+1}) \mid \rho^{\otimes}[0] = s^{\otimes}, \rho^{\otimes}\models \lozenge\text{unsafe}\right] \times \mathbb{P}_\pi(s^{\otimes} \models \lozenge \text{unsafe}) \Bigg\}
\end{align*}
Now consider the following two facts:
\begin{enumerate}[label=(\arabic*)]
    \item All ``safe'' paths have a discounted return of zero.
    \item All ``unsafe'' paths have a discounted return of precisely $1 / (1 -\gamma)$ (as the accepting product state $\F^\otimes$ constitutes a strongly connected end component, i.e., once you enter you cannot leave).
\end{enumerate}
Thus,
\begin{align*}
    \pi^*(s^{\otimes}) &= \arg\inf_{\pi\in \Pi^{\otimes}} \left \{ 0/(1 - \gamma) \times \mathbb{P}_\pi(s^{\otimes} \not \models \lozenge \text{unsafe}) + 1/(1-\gamma) \times \mathbb{P}_\pi(s^{\otimes} \models \lozenge \text{unsafe}) \right\}\\
    & = \arg\inf_{\pi\in \Pi^{\otimes}} \left \{ 1/(1-\gamma) \times \mathbb{P}_\pi(s^{\otimes} \models \lozenge \text{unsafe})\right  \}
\end{align*}
Given $1 / (1 - \gamma)$ is a constant, then, 
\begin{equation*}
    \pi^*(s^{\otimes}) = \arg\inf_{\pi\in \Pi^{\otimes}} \mathbb{P}_\pi(s^{\otimes} \models \lozenge \text{unsafe}) \quad \text{(for all $s^{\otimes} \in \S^{\otimes}$)}
\end{equation*}

To complete the proof we derive a contradiction.

Suppose that the optimal policy $\pi^*$ (minimizing the eventual-discounted return) is not maximally safe. Then that means there exists some $\epsilon \in [0, 1]$ and $s^{\otimes} \in \S^{\otimes}$ such that, $s^{\otimes} \in S^*(\epsilon)$ but $s^{\otimes} \not \in S(\pi^*, \epsilon)$. However, by definition we have,
\begin{equation*}
    S(\pi^*, \epsilon) = \{s^{\otimes} \in \S^{\otimes} \mid \mathbb{P}_{\pi^*}(s^{\otimes} \models \lozenge \text{unsafe}) < \epsilon \}
\end{equation*}
Substituting in our previous result, we have,
\begin{equation*}
    S(\pi^*, \epsilon) = \{s^{\otimes} \in \S^{\otimes} \mid \inf_{\pi\in \Pi^{\otimes}} \{\mathbb{P}_\pi(s^{\otimes} \models \lozenge \text{unsafe})\} < \epsilon \}
\end{equation*}
Thus $s^{\otimes} \not \in S(\pi^*, \epsilon)$ implies that $\inf_{\pi\in \Pi^{\otimes}} \mathbb{P}_\pi(s^{\otimes} \models \lozenge \text{unsafe}) \geq \epsilon$, which directly contradicts our assumption that $s^{\otimes} \in S^*(\epsilon)$. 
Therefore, the optimal policy $\pi^*$ that that minimizes the eventual-discounted return must also be a maximally safe policy.
\end{proof}

\subsection{Proof of Theorem \ref{thm:constraintsatisfaction}}

\begin{restatedthm}{\ref{thm:constraintsatisfaction} (restated)} Assume that all initial states satisfy the global safety constraint under the safe policy, i.e., $\S^\otimes_0 \subseteq S(\pi_{\mathrm{safe}}, \epsilon)$. Then if the step-wise risk budgets satisfy $\sum_{t=0}^\infty \epsilon_t \leq \epsilon$, then the shielded policy $\pi^\lozenge$ satisfies $\mathbb{P}_{\pi^\lozenge}(s_0 \models P_{\mathrm{safe}}) \geq 1-\epsilon, \;\; \forall s_0 \in \S_0$. 
\end{restatedthm}

\begin{proof}
Without loss of generality we fix an initial state $s_0 \in \S_0$ and $s_0^\otimes = (s_0, \Delta(\Q_0, L(s_0)))$ and prove by induction.

Let $t$, be the last time that $\pi^\lozenge$ used an action sampled from the task policy $\widehat{\pi}$. Consider the following two cases:

\begin{enumerate}[label=(\arabic*)]
    \item No such $t$ exists. Then $\pi^\lozenge$ always picks actions $a_t \sim \pi_{\text{safe}}(s_t^{\otimes})$, in that case we have $\mathbb{P}_{\pi^\lozenge} (s_0 \models P_{\mathrm{safe}}) \geq 1 - \epsilon$, given that, $s_0^{\otimes} \in S(\pi_{\text{safe}}, \epsilon)$ from the theorem statement.
    \item If $t = 0, 1, 2, \ldots$ then the last time that $\pi^\lozenge$ allowed $a_t \sim \widehat{\pi}(s_t^{\otimes})$ (at timestep $t$), we verified that $s^\otimes_t \in S(\pi_{\text{safe}}, a_t, \epsilon_t)$, thus for the remainder of the episode we have $\mathbb{P}_{\pi^\lozenge} (s^\otimes_t \models \lozenge \mathrm{unsafe}) < \epsilon_t$. By a simple union bound the probability of reaching an unsafe state up until timestep $t$ (i.e., at timestep $t-1$) is less than $\sum_{i=0}^{t-1} \epsilon_i$. Thus, by union bounding both events we establish the following lower bound on the safety probability,
    \begin{align*}
        \mathbb{P}_{\pi^\lozenge} (s_0 \models P_{\mathrm{safe}}) & > 1 - \epsilon_t - \sum_{i=0}^{t-1} \epsilon_i \\
        & =  1 - \sum_{i=0}^t\epsilon_i \\
        & \geq  1 - \sum_{i=0}^{\infty}\epsilon_i \\
        & \geq 1 - \epsilon
    \end{align*}
\end{enumerate}
\end{proof}

\subsection{Proof of Proposition \ref{prop:bellmanoperator}}

\begin{restatedprop}{\ref{prop:bellmanoperator} (restated)}
If (i) $\gamma\in[0,1)$ and (ii) $0\le \R(s,a,s')\le R_{\max}<\infty$ for all $(s, a, s') \in \S \times \A \times \S$, then $\mathbf{B}_{\textbf{safe}}$ is a $\gamma$-contraction under the sup-norm, $\|\cdot\|_{\infty}$, and converges to the unique fixed point $Q^*_{\text{safe}}$, the optimal Q-function subject to the action restrictions imposed by $S^*(a, \epsilon_t)$.
\end{restatedprop}

\begin{proof}
In this proof we aim to show the following:
\begin{enumerate}[label=(\arabic*)]
        \item $\mathbf{B}_{\textbf{safe}}$ is a $\gamma$-contraction under the sup-norm (i.e., $\|\cdot\|_{\infty}$).
        \item Iterating on $\mathbf{B}_{\textbf{safe}}$ converges to the unique fixed point $Q^*_{\text{safe}}$.
        \item $Q^*_{\text{safe}}$ is the optimal state-action value function subject to the action restrictions imposed by $\S^*(a, \epsilon_t)$.
    \end{enumerate}

\textit{Contraction property.} We define the  sup-norm (or max norm) for Q-functions in the metric space on $\S^\otimes \times \A$ as follows,
\begin{equation*}
    \Vert Q_1 - Q_2 \Vert_{\infty} = \max_{(s^{\otimes}, a)} \vert Q_1(s^{\otimes}, a) - Q_2(s^{\otimes}, a) \vert
\end{equation*}
The goal is to show that,
\begin{equation}
    \Vert \mathbf{B}_{\textbf{safe}} Q_1 - \mathbf{B}_{\textbf{safe}} Q_2  \Vert_{\infty} \leq \gamma \Vert Q_1 - Q_2 \Vert_{\infty} \label{eq:contraction}
\end{equation}
We identify the following two cases,
\begin{enumerate}[label=(\arabic*)]
    \item $s^{\otimes} \not \in \S^*(a, \epsilon_t) \Rightarrow$ unsafe action $a$ under step-wise constraint $\epsilon_t$. Then,

\begin{equation*}
        \mathbf{B}_{\textbf{safe}} Q_1(s^{\otimes}, a) =0,\;
        \mathbf{B}_{\textbf{safe}} Q_2(s^{\otimes}, a) =0,\;
         \Rightarrow \vert \mathbf{B}_{\textbf{safe}} Q_1(s^{\otimes}, a) - \mathbf{B}_{\textbf{safe}} Q_2(s^{\otimes}, a) \vert = 0
    \end{equation*}
\item $s^{\otimes} \in \S^*(a, \epsilon_t) \Rightarrow$ safe action $a$ under step-wise constraint $\epsilon_t$. Then,
\begin{align*}
        \mathbf{B}_{\textbf{safe}}Q_1(s^{\otimes}, a) &= \mathbb{E}^{\mathcal{M}^\otimes}\left[ \R(s, a, s') + \gamma \max_{a'} Q_1({s^{\otimes}}', a')\right]\\
        \mathbf{B}_{\textbf{safe}}Q_2(s^{\otimes}, a) &= \mathbb{E}^{\mathcal{M}^\otimes} \left[ \R(s, a, s') + \gamma \max_{a'} Q_2({s^{\otimes}}', a')\right]
    \end{align*}
    subtracting,
    \begin{align*}
        &\vert  \mathbf{B}_{\textbf{safe}}Q_1(s^{\otimes}, a) - \mathbf{B}_{\textbf{safe}}Q_2(s^{\otimes}, a)\vert \\
        & \qquad = \gamma \left\vert \mathbb{E}^{\mathcal{M}^\otimes}\left[ \max_{a'} Q_1({s^{\otimes}}', a') - \max_{a'} Q_2({s^{\otimes}}', a') \right] \right\vert \\
        &\qquad\leq \gamma  \mathbb{E}^{\mathcal{M}^\otimes}\left\vert \max_{a'} Q_1({s^{\otimes}}', a') - \max_{a'} Q_2({s^{\otimes}}', a') \right\vert\\
        &\qquad = \gamma \Vert Q_1({s^{\otimes}}', a') - \max_{a'} Q_2({s^{\otimes}}', a') \Vert_{\infty} 
    \end{align*}
\end{enumerate}
combining both cases and taking the maximum over $(s^{\otimes}, a) \in \S^{\otimes} \times \A$ establishes \eqref{eq:contraction}.

\textit{Unique fixed point.} (i) $\gamma\in[0,1)$ and (ii) $0\le \R(s,a,s')\le R_{\max}<\infty$ in the proposition statement are sufficient conditions to apply the \emph{Banach Fixed Point Theorem} \citep{banach1922operations}, which states (informally), any $\gamma$-contraction in a complete metric space (bounded real-valued functions on $\S^{\otimes} \times \A$) has: a unique fixed point, which is obtained by iterative application of the operator.

Therefore, starting from any $Q_0$, we let $Q^*_{\text{safe}} := \lim_{n \to \infty}(\mathbf{B}_{\textbf{safe}})^n Q_0$ denote the fixed point of the operator $\mathbf{B}_{\textbf{safe}}$. Furthermore, $Q^*_{\text{safe}}$ satisfies,
\begin{equation*}
    Q^*_{\text{safe}} = \mathbf{B}_{\textbf{safe}}(Q^*_{\text{safe}})
\end{equation*}

\textit{Optimal state-action value function.} We argue that the fixed point of the operator $\mathbf{B}_{\textbf{safe}}$, namely $Q^*_{\text{safe}}$ is the optimal state-action value function subject to the action restrictions imposed by the maximally safe set $\S^*(a, \epsilon_t)$. Furthermore, the corresponding greedy policy $\pi^*(s^{\otimes}) = \arg\max_a Q^*_{\text{safe}}(s^{\otimes}, a)$ upper bounds every safe policy in terms of discounted reward. First we establish the following Lemmas.
\begin{lemma} \label{lem:valuezero}
    For every $(s^{\otimes},a) \in \S^{\otimes}\times \A$ such that $s^{\otimes} \not \in \S^*(a, \epsilon_t)$, we have,
\begin{equation*}
    Q^*_{\text{safe}}(s^{\otimes}, a) = 0
\end{equation*}
\end{lemma}
\begin{proof}
    This is a direct consequence of the Bellman operator $\mathbf{B}_{\textbf{safe}}$; since $Q^*_{\text{safe}}(s^{\otimes}, a)$ is a fixed point of $\mathbf{B}_{\textbf{safe}}$,
\begin{equation*}
    Q^*_{\text{safe}}(s^{\otimes}, a) = (\mathbf{B}_{\textbf{safe}}Q^*_{\text{safe}})(s^{\otimes}, a)
\end{equation*}
and by construction $\mathbf{B}_{\textbf{safe}}(s^{\otimes}, a)=0$ whenever $s^{\otimes} \not \in \S^*(a, \epsilon_t)$.
\end{proof}

Let the set of safe policies be,
\begin{equation*}
    \Pi_{\text{safe}}^{\otimes} = \{ \pi \in \Pi^{\otimes} \mid \pi(a \mid s^{\otimes})=0 \;\text{if $s^{\otimes} \not \in \S^*(a, \epsilon_t)$}\}
\end{equation*}
And for any safe policy $\pi \in \Pi^{\otimes}_{\text{safe}}$ define the \emph{policy evaluation operator},
\begin{equation*}
    (\mathbf{B}_{\pi}Q)(s^{\otimes}, a) = \mathbb{E}^{\mathcal{M}^\otimes} \left[ \R(s, a, s') + \gamma \mathbb{E}_{\pi}Q({s^{\otimes}}', a)\right]
\end{equation*}
\begin{lemma}{Monotonicity of the operator} \label{lem:monotonicity}
    For every $\pi \in \Pi_{\text{safe}}^{\otimes}$ and for every $Q$ in the metric space on $\S^\otimes \times \A$,
\begin{equation*}
    \mathbf{B}^*_{\text{safe}}Q \geq \mathbf{B}_{\pi}Q \qquad \text{(element wise)}
\end{equation*}
\end{lemma}
\begin{proof}
The proof is a straightforward argument by cases:
\begin{enumerate}[label=(\arabic*)]
    \item If $s^{\otimes} \not \in \S^*(a, \epsilon_t)$ then both sides are zero, so the inequality holds.
    \item If $s^{\otimes} \in \S^*(a, \epsilon_t)$ then $\mathbf{B}^*_{\text{safe}}$ is identical to the usual optimality update but maximizing only over safe actions (all unsafe actions have value $0$); the expectation under $\pi \in \Pi^{\otimes}_{\text{safe}}$ is a convex combination of those same safe actions, thus its value cannot exceed the maximum.
\end{enumerate}
\end{proof}

\begin{corollary} \label{cor:upperbound} For any starting function $Q_0$ and any safe policy $\pi \in \Pi^{\otimes}_{\text{safe}}$
    \begin{equation*}
        (\mathbf{B}^*_{\text{safe}})^nQ_0 \geq (\mathbf{B}_{\pi})^nQ_0 \quad \forall n \geq 0
    \end{equation*}
Taking $n \to \infty$ gives,
\begin{equation*}
    Q^{*}_{\text{safe}} \geq Q_{\pi} \quad \forall \pi \in \Pi^{\otimes}_{\text{safe}}
\end{equation*}
which follows from Lemma \ref{lem:monotonicity} and because $(\mathbf{B}^*_{\text{safe}})^nQ_0 \to Q^{*}_{\text{safe}}$ (by the contraction property) and $(\mathbf{B}_{\pi})^nQ_0 \to Q_{\pi}$. Thus no safe policy can achieve a higher value than $Q^{*}_{\text{safe}}$. 
\end{corollary}

Now define the greedy safe policy as,
\begin{equation*}
    \pi^*(s^{\otimes}) = \arg\max_a Q^*_{\text{safe}}(s^{\otimes}, a) \quad \text{(for all $s^{\otimes} \in \S^{\otimes}$)}
\end{equation*}
Since rewards are non-negative, then $Q^{*}(s^{\otimes}, a) \geq 0$ for all $(s^{\otimes}, a)  \in \S^{\otimes} \times \A$, thus we always implicitly prioritize safe actions first when taking the $\arg\max$ (recall all unsafe actions have value $0$, c.f., Lemma \ref{lem:valuezero}). Ties can be broken in the $\arg\max$ by picking any available safe action, when for example, all safe and unsafe actions have value 0. We require one final lemma before putting it all together.
\begin{lemma}[Bellman consistency for $\pi^*$] \label{lem:consistency}
\begin{equation*}
    \mathbf{B}_{\pi^*} Q^*_{\text{safe}} = Q^*_{\text{safe}}
\end{equation*}
\end{lemma}
\begin{proof}
\begin{enumerate}[label=(\arabic*)]
\item If $s^{\otimes} \not \in \S^*(a, \epsilon_t)$ then both sides are zero, so the equality holds.
\item If $s^{\otimes} \in \S^*(a, \epsilon_t)$ then,
\begin{equation*}
    (\mathbf{B}^*_{\text{safe}} Q^*_{\text{safe}})(s^{\otimes}, a) =\mathbb{E}^{\mathcal{M}^\otimes}\left[ \R(s, a, s') + \gamma \max_{a'} Q^*_\text{safe}({s^{\otimes}}', a')\right]
\end{equation*}
by construction $\pi^*$ puts all its probability mass on the maximizing action $a'$, thus the RHS equals $(\mathbf{B}_{\pi^*} Q^*_{\text{safe}})(s^{\otimes}, a)$.
\end{enumerate}
\end{proof}

Since the value for all unsafe actions is precisely 0 by Lemma  \ref{lem:valuezero} and any safe policy (those restricted to only permissible actions) has value upper bounded by $Q^{*}_{\text{safe}}$ (c.f., Corollary \ref{cor:upperbound}), with the greedy policy achieving equality (c.f., Lemma \ref{lem:consistency}). Therefore $Q^{*}_{\text{safe}}(s^{\otimes}, a)$ equals:
\begin{enumerate}[label=(\arabic*)]
    \item $0$ if $s^{\otimes} \not\in \S^*(a, \epsilon_t)$,
    \item $\max_{\pi \in \Pi^{\otimes}_{\text{safe}}} Q_{\pi} (s^{\otimes}, a)$ if $s^{\otimes} \in \S^*(a, \epsilon_t)$.
\end{enumerate}
Hence $Q^{*}_{\text{safe}}$ is precisely ``Q-function subject to the action restrictions imposed by $\S^*(a, \epsilon_t)$''.
\end{proof}

\subsection{Proof of Proposition \ref{prop:statisticalmc}}

Before we begin the proof we start with some auxiliary lemmas. First we establish a lower bound on the number of samples required for statistical model checking $m$. With the transition kernel $\P$.

\begin{lemma}\label{lem:bounded} Let $\epsilon' >0$, $\delta'>0$, $H\geq1$, $s^{\otimes} \in \S^{\otimes}$, $a \in \A$ and $\pi$ be given. By sampling $m \geq 1/(2{\epsilon'}^2)\log(2/\delta')$ paths, $M= \langle \rho^{\otimes}_i \in (\S^{\otimes})^H \rangle^{m}_{i=1}$, with the transition kernel $\P^{\otimes}_{\pi} ({s^{\otimes}}' \mid s^{\otimes})=\sum_{a \in A}\P^{\otimes}({s^{\otimes}}' \mid s^{\otimes}, a)\pi(a \mid s^{\otimes})$ (c.f., Definition \ref{def:product-mdp}), we can estimate $\mathbb{P}_\pi^{a}(s^{\otimes} \models \lozenge^{\leq H} \mathrm{unsafe})$ within an $\epsilon'$ error bound and w.p.~at least $1-\delta$. 
\end{lemma}

\begin{proof} Formally, let $X_1, X_2, \ldots X_m$ be indicator random variables such that,
\begin{equation*}
        X_i = \begin{cases}
            1 & \text{if $\rho^{\otimes} \models \lozenge^{\leq H} \text{unsafe}$}\\
            0 & \text{otherwise}
        \end{cases}
    \end{equation*}
Note that $\rho^{\otimes} \models \lozenge^{\leq H} \text{unsafe}$ can be checking in time $O(H)$ and even parallelized across the batch. Now let,
    \begin{equation*}
        \bar X = \frac{1}{m}\sum^m_{i=1} X_i
    \end{equation*}
and notice that,
\begin{equation*}
        \mathbb{E}\left[ \bar X\right] = \mathbb{P}_\pi^a(s^{\otimes} \models \lozenge^{\leq H} \text{unsafe})
    \end{equation*}
then by Hoeffding's inequality \citep{hoeffding1963probability},
    \begin{equation*}
        \mathbb{P}\left[ \left\vert \bar{X} -  \mathbb{E}\left[ \bar X\right] \geq \epsilon' \right\vert\right] \leq 2 \exp(-2m{\epsilon'}^2)
    \end{equation*}
    Bounding the RHS by $\delta'$ and rearranging gives the desired result.
\end{proof}

We continue with a variant of the \emph{Error Amplification Lemma}, which establishes a linear (in $H$) worst-case error amplification when sampling paths from the empirical transition kernel $\widehat \P$. 

\begin{lemma}[Error amplification for path distributions] \label{lem:paths}  Let $\alpha > 0$, the start state $s^{\otimes}_0 \in \S^{\otimes}$ and $\pi$ be given. Suppose true transition kernel $\P$ and the empirical transition kernel $\widehat \P$ are such that, 
\begin{equation*}
    \Vert \widehat\P(\cdot \mid s, a) - \P(\cdot \mid s, a)\Vert_1 \leq \alpha \quad \forall (s, a) \in \S \times \A 
\end{equation*}
Let $\P^t_{\pi}(\cdot)$ and $\widehat \P^t_{\pi}(\cdot)$ denote the path distributions (at timestep $t$) for the two Markov chains with transition kernels $\P^{\otimes}_{\pi}$ and $\widehat \P^{\otimes}_{\pi}$ respectively. Then the L1-distance between the two path distributions is bounded as follows,
\begin{equation*}
    \Vert \P_{\pi}^t(\cdot) - \widehat\P_{\pi}^{t}(\cdot) \Vert_1 \leq \alpha t \quad \forall t=0,1,\ldots
\end{equation*}
\end{lemma}
\begin{proof}
First we provide formal definitions for the path distributions $\P^t_{\pi}(\cdot)$ and $\widehat \P^t_{\pi}(\cdot)$. Let,
\begin{align*}
    \P^t_{\pi}(\rho^\otimes) &:= \Pr(s_0^{\otimes}s_1^{\otimes}\ldots s_t^{\otimes} \preceq \rho^\otimes \mid \rho^\otimes[0] =s_{0}^{\otimes}, \P^{\otimes}_{\pi}) \\
    \widehat\P^t_{\pi}(\rho^\otimes) &:= \Pr(s_0^{\otimes}s_1^{\otimes}\ldots s_t^{\otimes} \preceq \rho^\otimes \mid \rho^\otimes[0] =s_{0}^{\otimes}, \widehat\P^{\otimes}_{\pi}) 
\end{align*}
where,
\begin{align*}
    \P^{\otimes}_{\pi}({s^\otimes}'\mid s^{\otimes}) &:=\sum_{a \in A}\P^{\otimes}({s^{\otimes}}' \mid s^{\otimes}, a)\pi(a \mid s^{\otimes})\\
    \widehat\P^{\otimes}_{\pi}({s^\otimes}'\mid s^{\otimes}) &:=\sum_{a \in A}\widehat\P^{\otimes}({s^{\otimes}}' \mid s^{\otimes}, a)\pi(a \mid s^{\otimes})
\end{align*}
are the corresponding transition kernels. The probability $\P^t_{\pi}(\rho^\otimes)$ reads: ``the probability of the prefix $s_0^{\otimes}s_1^{\otimes}\ldots s_t^{\otimes} \preceq \rho^\otimes$ at time $t$''. The proof now follows by induction on $t$, for the base case we have,
\begin{equation*}
    \P^0_{\pi}(\cdot) = \widehat\P^0_{\pi}(\cdot)
\end{equation*}
since the start state $s^{\otimes}_0 \in \S^{\otimes}$ is fixed for both. Before we continue with the induction, we make the following observation for any path $\rho^\otimes \in (\S^{\otimes})^{\omega}$,
\begin{multline}
    \vert \P_{\pi}^{t}(\rho^\otimes) - \widehat\P^t_{\pi}(\rho^\otimes) \vert = \vert \P^{\otimes}_{\pi}(s^{\otimes}_t \mid s^{\otimes}_{t-1})\P_{\pi}^{t-1}(\rho^\otimes) - \widehat\P^{\otimes}_{\pi}(s^{\otimes}_t \mid s^{\otimes}_{t-1})\widehat\P_{\pi}^{t-1}(\rho^\otimes) \vert \\
    \leq \P_{\pi}^{t-1}(\rho^\otimes) \vert \P_{\pi}^{\otimes}(s^{\otimes}_t \mid s^{\otimes}_{t-1}) - \widehat\P_{\pi}^{\otimes}(s^{\otimes}_t \mid s^{\otimes}_{t-1}) \vert + \widehat \P_{\pi}^{\otimes}(s^{\otimes}_t \mid s^{\otimes}_{t-1}) \vert  \P_{\pi}^{t-1}(\rho^\otimes) - \widehat\P^{t-1}_{\pi}(\rho^\otimes) \vert \label{eq:triangle}
\end{multline}
holds by the triangle inequality. Now we continue with the induction on $t$,
\begin{align}
    2\Vert \P_{\pi}^t(\cdot) - \widehat\P_{\pi}^{t}(\cdot) \Vert_1 & = \sum_{\rho^\otimes \in (\S^{\otimes})^{\omega}} \vert  \P_{\pi}^t(\rho^\otimes) - \widehat\P_{\pi}^{t}(\rho^\otimes) \vert \label{eq:l1definition}\\
    \begin{split}
        & \leq \sum_{\rho^\otimes \in (\S^{\otimes})^{\omega}} \P_{\pi}^{t-1}(\rho^\otimes) \vert \P_{\pi}^{\otimes}(s^{\otimes}_t \mid s^{\otimes}_{t-1}) - \widehat\P_{\pi}^{\otimes}(s^{\otimes}_t \mid s^{\otimes}_{t-1}) \vert \\
        & \qquad + \sum_{\rho^\otimes \in (\S^{\otimes})^{\omega}} \widehat \P_{\pi}^{\otimes}(s^{\otimes}_t \mid s^{\otimes}_{t-1}) \vert  \P_{\pi}^{t-1}(\rho^\otimes) - \widehat\P^{t-1}_{\pi}(\rho^\otimes) \vert
    \end{split} \label{eq:trianglebefore} \\    
    &\leq \sum_{\rho^\otimes \in (\S^{\otimes})^{\omega}} \P_{\pi}^{t-1}(\rho^\otimes) \cdot (2 \alpha) + \sum_{\rho^\otimes \in (\S^{\otimes})^{\omega}} \vert  \P_{\pi}^{t-1}(\rho^\otimes) - \widehat\P^{t-1}_{\pi}(\rho^\otimes) \vert \label{eq:fromassumption}\\
    & = 2 \alpha + 2 \Vert \P_{\pi}^{t-1}(\cdot) - \widehat\P_{\pi}^{t-1}(\cdot) \Vert_1\\
    & \leq 2 \alpha t 
\end{align}
Here \eqref{eq:l1definition} comes from the definition of the L1-distance: $\Vert \mu_1(\cdot) - \mu_2(\cdot) \Vert_1 = \frac{1}{2}\sum_{x \in \X} \vert \mu_1(x) - \mu_2(x) \vert$. The rest follows from our earlier observations (triangle inequality) and assumptions in the lemma statement.
\end{proof}

We continue now with the main proof for Proposition \ref{prop:statisticalmc}. Restating the proposition here.

\begin{restatedprop}{\ref{prop:statisticalmc} (restated)} Let $\epsilon' \leq \epsilon_t$ and suppose the total variation distance between $\P$ and $\widehat{\P}$ is bounded by $\epsilon' / H$, i.e., $\|\widehat \P(\cdot\mid s,a)-\P(\cdot\mid s,a)\|_1\le \epsilon'/H^*$ for all $(s,a)$. Then exact or statistical model checking with $\widehat{\P}$ allows $(s_t^\otimes, a_t)$ only if $s^\otimes_t \in S(\pi_{\mathrm{safe}},a_t, \epsilon_t)$ with probability (w.p.)~1 or w.p.~at least $(1-\delta)$ respectively.
\end{restatedprop}

\begin{proof} For convenience let $\mathcal{E} :=\mathbb{P}_\pi^a(s^{\otimes} \models \lozenge^{\leq H^*} \text{unsafe})$ with probability taken under $\P^{\otimes}$ and let $\widehat{\mathcal{E}}$ be defined analogously but with probability taken under $\widehat \P^\otimes$. Furthermore, let  $g(\cdot)$ and $\widehat g(\cdot)$ denote the average path distribution (over the next $H^*$ timesteps) for the two Markov chains under $\P^{\otimes}_{\pi}$ and $\widehat \P^{\otimes}_{\pi}$ respectively. Formally,
\begin{align*}
    g(\rho^\otimes) &:= \frac{1}{H^*}\sum^{H^*}_{t=0}\P^t_{\pi}(\rho^\otimes)\\
    g(\rho^\otimes) &:= \frac{1}{H^*}\sum^{H^*}_{t=0}\widehat\P^t_{\pi}(\rho^\otimes)
\end{align*}
abusing notation here (by dropping the action $a$ which is now fixed and assumed to be implicit). Before continuing we make the following observation,
\begin{enumerate}[label=(\arabic*)]
\item Let $f(x): x \in \X \to [0, 1]$ be a real valued function. Let $\mu_1(\cdot)$ and $\mu_2(\cdot)$ be probability distributions over the space $\X$, then,
    \begin{equation*}
        \left\vert \mathbb{E}_{x \sim \mu_1(\cdot)}\left[f(x)\right] - \mathbb{E}_{x \sim \mu_2(\cdot)}\left[f(x)\right] \right\vert \leq \Vert \mu_1(\cdot) - \mu_2(\cdot)\Vert_1
    \end{equation*}
\end{enumerate}
Thus,
\begin{align}
    \vert \widehat{\mathcal{E}} - \mathcal{E} \vert &= \left\vert \mathbb{E}_{\rho^\otimes \sim g}\left[ \ind{ \rho^\otimes \models \lozenge^{\leq H^*} \text{unsafe} } \right] - \mathbb{E}_{\rho^\otimes \sim g}\left[\ind{ \rho^\otimes \models \lozenge^{\leq H^*} \text{unsafe} }\right] \right\vert \\
    & \leq \Vert g(\cdot) - \widehat g (\cdot) \Vert_1 \label{eq:firstineq}
    \\
    &  = \frac{1}{2}\sum_{\rho^\otimes \in (\S^{\otimes})^{H^*}} \vert g(\rho^\otimes) - g(\rho^\otimes) \vert \label{eq:bydef} \\
    &  = \frac{1}{2}\sum_{\rho^\otimes \in (\S^{\otimes})^{H^*}} \left\vert \frac{1}{H^*} \sum^{H^*}_{t=0} \P^t_{\pi}(\rho^\otimes) - \widehat \P^t_{\pi}(\rho^\otimes)  \right\vert \label{eq:bydef2}\\
    & \leq \frac{1}{2H^*} \sum^{H^*}_{t=0} \left\vert \sum_{\rho^\otimes \in (\S^{\otimes})^{H^*}}  \P^t_{\pi}(\rho^\otimes) - \widehat \P^t_{\pi}(\rho^\otimes)  \right\vert \\
    &  \leq \frac{1}{2H^*}\sum^{H^*}_{t=0} H^* (\epsilon'/H^*) \label{eq:fromlem}\\
    & = \epsilon'/2
\end{align}
Here \eqref{eq:firstineq} comes from our preceding observation, \eqref{eq:bydef} and \eqref{eq:bydef2} are by construction and \eqref{eq:fromlem} comes from the result of Lemma \ref{lem:paths} (by substituting in $\alpha = \epsilon'/H^*$). This intermediate result is closely related to the \emph{simulation lemma} \cite{kearns2002near}, which has been proved for the model-based setting (with bounded L1-distance) \cite{rajeswaran2020game}.

Now to put it all together. It has been established that $\vert \widehat{\mathcal{E}} - \mathcal{E} \vert \leq \epsilon'/2$. Thus, if we calculate the quantity $\widehat{\mathcal{E}}$ via model checking (or similar), then we can obtain an estimate for $\mathcal{E}$ (i.e., the quantity we are interested in, namely, $\mathbb{P}_\pi^a (s^{\otimes} \models \lozenge^{\leq H^*})$. Consider the two cases:
\begin{enumerate}[label=(\arabic*)]
    \item \emph{Exact model checking}: with the empirical (product) transition kernel $\widehat \P^{\otimes}$, we can precisely calculate $\widehat{\mathcal{E}}$ via exact probabilistic model checking with $O(H^*)$ matrix multiplications \cite{baier2008principles}.
    \item \emph{Statistical model checking}: by sampling $m \geq 2/({\epsilon'}^2)\log(2/\delta)$ paths $M= \langle \rho^\otimes_i \in (\S^{\otimes})^H \rangle^{m}_{i=1}$ with the empirical (product) transition kernel $\widehat \P^{\otimes}$, we can estimate $\widehat{\mathcal{E}}$ with an error of up to $\epsilon'/2$ and with confidence at least $(1 - \delta)$. This follows from Lemma \ref{lem:bounded}.
\end{enumerate}
Thus via exact (or statistical model) checking we have established that we can estimate $\mathcal{E}$ with an error up to $\epsilon'$ w.p.~1 (or at least $1-\delta$ resp.).

Let $\tilde{\mathcal{E}}$ denote our estimate for $\mathcal{E}$, the final part of the proof concerns how we use our estimate checking if $s^\otimes_t \in S(\pi_{\mathrm{safe}},a_t, \epsilon_t)$. 

To establish the \emph{only if} direction we take a conservative view of our uncertainty that does not permit ``false positives'', i.e., Assuming $\epsilon' < \epsilon_t$ (by the proposition statement), $s^\otimes_t \in S(\pi_{\mathrm{safe}},a_t, \epsilon_t)$ only if $\tilde{\mathcal{E}} \leq \epsilon_t - \epsilon'$, this holds as follows.

Since we have already established that $\vert \tilde{\mathcal{E}} - \mathcal{E} \vert \leq \epsilon'$ w.p.~1 or at least $1-\delta$,  then $\tilde{\mathcal{E}} \leq \epsilon_t -\epsilon'$ implies that $\mathcal{E} \leq \epsilon_t$ (w.p.~1 or $\geq1-\delta$). Furthermore, under Assumption \ref{ass:finite-horizon}, $\mathbb{P}_{\pi_{\mathrm{safe}}}^a(s^{\otimes} \models \lozenge^{\leq H^*} \text{unsafe}) \leq \epsilon_t$ implies that $\mathbb{P}_{\pi_{\mathrm{safe}}}^a(s^{\otimes} \models \lozenge \text{unsafe}) \leq \epsilon_t$ and thus $(s_t^\otimes, a)$ is allowed only if $s^\otimes_t \in S(\pi_{\mathrm{safe}},a_t, \epsilon_t)$ w.p.~1 or at least $1-\delta$.
\end{proof}

\subsection{Proof of Theorem \ref{thm:safetyguarantees}}

We start with an auxiliary result for transition kernel estimation. The proof of Theorem \ref{thm:safetyguarantees} follows closely thereafter. 

\begin{lemma} \label{lem:tvdistance} Let $p> 0$. Then if each $(s, a) \in \S \times \A$ is visited at least $O(\vert \S \vert/p^2\log(\vert\S\vert \vert A \vert/\delta))$ times, then w.p.~at least ($1-\delta$) the empirical transition kernel $\widehat \P$ satisfies $\Vert\widehat \P(\cdot \mid s, a) - \P(\cdot \mid s, a) \Vert_1 \leq p$ for all $(s, a) \in \S \times \A$.
\end{lemma}

\begin{proof} First for a given dataset $D = \langle (s_i, a_i, s_i')\rangle^{n}_{i=1}$  let,
\begin{equation*}
    n (s, a) := \vert \{ (s, a, \cdot) \in D \} \vert
\end{equation*}
and
\begin{equation*}
    n(s, a, s') := \vert \{ (s, a, s') \in D \} \vert
\end{equation*}
The by standard results for L1-distances \citep{weissman2003inequalities}, the TV-distance between $\P$ and $\widehat \P$ is bounded as follows,
\begin{equation*}
        \mathbb{P}_{s' \sim \P(\cdot \mid s,a)} \left[ \Vert \widehat\P(\cdot\mid s, a) -  \P(\cdot\mid s, a) \Vert_1 \geq p\right] \leq (2^{\vert \S \vert}-2)\exp\left(\frac{-n(s, a)}{2p^2}\right)
    \end{equation*}
Bounding the RHS from above by $\delta/(\vert\S\vert\vert\A\vert)$ and rearranging gives the following lower bound for $n(s, a)$,
    \begin{equation*}
        n(s, a) \geq \frac{2}{p^2}\log\left(\frac{\vert\S\vert\vert\A\vert(2^{\vert\S\vert}-2)}{\delta} \right) = O\left(\frac{\vert\S\vert}{p^2}\log\left(\frac{\vert\S\vert\vert\A\vert}{\delta} \right)\right)
    \end{equation*}
    Finally, taking a union bound over all state action pairs $(s, a) \in \S \times \A$ (of which there are $\vert\S\vert\vert\A\vert$ many) completes the proof.
\end{proof}

\begin{restatedthm}{\ref{thm:safetyguarantees} (restated)} Let $\epsilon'<\epsilon_t$ and $\delta_t=2\delta$. Under the conditions in Theorem \ref{thm:constraintsatisfaction} and Assumption \ref{ass:finite-horizon}, and if each $(s,a) \in \S \times \A$ is visited at least $O((H^*/\epsilon')^2 |\S| \log(|\A||\S|/\delta))$ times, then by using exact (or statistical) model checking with $\widehat{\P}$, with shielded policy $\pi^\lozenge$ satisfies the step-wise risk level $\epsilon_t$ w.p.~at least $(1-\delta_t)$, and thus satisfies the global safety constraint in \eqref{eq:problem} w.p.~at least $(1 - \sum_{t=0}^{\infty} \delta_t)$.
\end{restatedthm}

\begin{proof}
    The proof is a relatively straightforward application of Lemma \ref{lem:tvdistance} and Proposition \ref{prop:statisticalmc}. Firstly, by Lemma \ref{lem:tvdistance}, if each $(s,a) \in \S \times \A$ is visited at least $O((H^*/\epsilon')^2 |\S| \log(|\A||\S|/\delta))$ times, then we have,
\begin{equation}
    \Vert \widehat \P(\cdot \mid s, a) - \P(\cdot \mid s, a) \Vert_1 \leq \epsilon'/H^* \quad \forall (s, a) \in \S \times \A \label{eq:conditions}
\end{equation}
w.p.~at least $1 - \delta$. Now that \eqref{eq:conditions} satisfies the conditions of Proposition \ref{prop:statisticalmc}, then exact (or statistical) model checking with $\widehat{\P}$ allows $(s_t^\otimes, a_t)$ only if $s^\otimes_t \in S(\pi_{\mathrm{safe}},a, \epsilon_t)$ w.p.~1 or at least ($1-\delta$).

Finally, taking a union bound now over both the L1-distance and model checking procedure gives us a worst case step-wise failure probability $\delta_t = 2 \delta$. Thus at timestep $t$, $\pi^\lozenge$ allows $a_t \sim \widehat \pi(\cdot \mid s_t^{\otimes})$ only if $s^{\otimes}_t \in S(\pi_{\text{safe}}, a_t, \epsilon_t)$ w.p.~at least $1-\delta_t$, which by definition provides a step-wise safety guarantee of $\epsilon_t$. On the other hand, if $a_t \sim \pi_{\text{safe}}(s^{\otimes}_t)$ cannot be verified, then by induction it can be shown that the step-wise safety guarantee of $\epsilon_t$ holds, using similar reasoning as in the proof of Theorem \ref{thm:constraintsatisfaction}. Taking a union bound over the step-wise failure probability $\delta_t$ completes the proof.
\end{proof}

\subsection{Proof of Theorem \ref{thm:convergence}}
\label{sec:convergence-proof}

\begin{restatedthm}{\ref{thm:convergence} (restated)} Assume: (i) $\gamma\in[0,1)$ and rewards are bounded and non-negative;
(ii) the step-sizes for each Q-approximation satisfy standard stochastic approximation conditions
(e.g., $\sum_t \alpha_t=\infty$ and $\sum_t \alpha_t^2<\infty$); (iii) the step-wise risk levels and initial states satisfy the conditions for Theorem \ref{thm:constraintsatisfaction}; and (iv) Assumption \ref{ass:finite-horizon} holds. Then, in all modes the task Q-iterates ($Q_\textrm{task}$) produced by stochastic approximation of $\mathbf{B}_{\mathbf{safe}}$ converge w.p.~1 to the unique fixed point $Q^*_{\mathrm{safe}}$, i.e., the optimal Q-function subject to the action restrictions imposed by $S^*(a, \epsilon_t)$. And the deployed policy $\pi^\lozenge$ satisfies the global safety constraint in \eqref{eq:problem} w.p.~1.
\end{restatedthm}

\begin{proof}
     The proof is lengthy and split into several parts, each addressing the three different operation modes of \textsc{RecReg} to a different extent. We summarize the organization below and detail parts are required for convergence of the different operation mode.

     \begin{enumerate}
        \item (\textsc{ModelBased} and \textsc{ModeFree}). We establish convergence of the safe policy $\pi_{\text{safe}}$ and corresponding probabilistic safe set $S(\pi_{\text{safe}}, a, \epsilon_t)$ to the maximally safe policy $\pi^*_{\text{safe}}$ and maximal safe set $S^*(a, \epsilon_t)$ respectively.
        \item (\textsc{ModelBased} and \textsc{ModeFree}). We establish convergence of the indicator $\ind {s^{\otimes} \in S(\pi_{\text{safe}}, a, \epsilon_t)}$ used in the Bellman operator $\mathbf{B}_{\textbf{safe}}$.
        \item (\textsc{Exact}, \textsc{ModelBased} and \textsc{ModeFree}) Analysing convergence of the shielded policy $\pi_{\text{shield}}$ defined in Eq.~\eqref{eq:shieldedpolicy} under a perfect indicator $\ind {s^{\otimes} \in S(\pi_{\text{safe}}, a, \epsilon_t)}$.  
        \item (\textsc{ModelBased}). Analysing the optimality of $\pi_{\text{shield}}$ under $\delta$ corruptions that might persist as a consequence of using statistical model checking
    \end{enumerate}

\paragraph{Convergence of the safe policy.} We first formalize the convergence claim for the backup policy with the following Lemma.

\begin{lemma}[Convergence of the safe policy] \label{lem:safe-policy-convergence} Consider the finite product MDP $\mathcal{M}^\otimes = (\S^\otimes, \A, \S^\otimes_0, \P^\otimes)$ and without loss of generality assume the accepting set $\F^\otimes$ is made absorbing. The cost function in \eqref{eq:costfn} can thus be equivalently defined,
\begin{equation*}
    \mathbf C(s^\otimes, a, {s^\otimes}') = \begin{cases}
        1 & \text{if $s^\otimes \in \F^\otimes$} \\
        0 & \text{otherwise}
    \end{cases}
\end{equation*}
Define the eventual-discounting scheme with state-dependent discount factor,
\begin{equation*}
    \gamma({s^\otimes})
=
\begin{cases}
\gamma, & {s^\otimes}\in\F^\otimes,\\
1, & \text{otherwise},
\end{cases}
\end{equation*}
and consider the update step,
\begin{equation*}
Q_{t+1}(s^\otimes, a) = (1 - \alpha_t(s^\otimes, a))Q_t(s^\otimes, a) + \alpha_t(s^\otimes, a) \left(\mathbf C(s^\otimes, a, {s^\otimes}') + \gamma({s^\otimes}' )\min_{a'} Q_t({s^\otimes}', a') \right)
\end{equation*}
with step-sizes satisfying the Robbins-Monro conditions,
\begin{equation*}
\sum_t \alpha_t(s^\otimes,a)=\infty,\qquad
\sum_t \alpha_t(s^\otimes,a)^2<\infty
\end{equation*}
and such that every $(s^\otimes, a)$ is visited infinitely often. Further, assuming the initialization,
\begin{equation*}
Q_0(s^\otimes,a)=0 \quad \forall (s^\otimes,a) \in \S^\otimes \times \A
\end{equation*}
For each state-action pair $(s^\otimes,a)$ define,
\begin{equation*}
    p^*(s^\otimes,a)
:=
\inf_{\pi\in\Pi^\otimes}
\mathbb{P}^a_\pi(s^\otimes\models\lozenge \mathrm{unsafe})
\end{equation*}
Then
\begin{equation*}
    Q_t(s^\otimes,a)\to Q^*(s^\otimes,a)
=
\frac{1}{1-\gamma}\,p^*(s^\otimes,a)
\quad
\forall (s^\otimes,a) \in \S^\otimes \times \A  \quad \text{(almost surely)}
\end{equation*}
as $t \to \infty$. Consequently, any greedy policy with respect to $Q^*$ is maximally safe,
and for every product state $s^\otimes$ visited infinitely often,
\begin{equation*}
    \lim_{t\to\infty}
\mathbb{P}\!\left(
\pi_t(s^\otimes)\in
\arg\min_{a\in\A}p^*(s^\otimes,a)
\right)
=1
\end{equation*}
where $\pi_t$ is the greedy policy induced by $Q_t$, with a fixed
deterministic tie-breaking rule.
\end{lemma}

\begin{proof} Since $\F^\otimes$ is absorbing, every trajectory has a well-defined hitting time,
\begin{equation*}
    \tau_{\F}:=\inf\{t\ge 1:s^\otimes_t\in\F^\otimes\}.
\end{equation*}
Under the eventual-discounting convention, no discounting is applied before $\tau_{\F}$, and after $\tau_{\F}$ the process remains in $\F^\otimes$ and receives a cost of 1 geometrically discounted by $\gamma < 1$. Hence, for
any policy $\pi$ and first action $a$,
\begin{equation*}
    Q_\pi(s^\otimes,a)
=
\mathbb{E}_\pi
\left[
\sum_{t=0}^{\infty}
\Gamma_t
\mathbf C(s^\otimes_t,a_t,s^\otimes_{t+1})
\mid a_0 = a\right]
=
\frac{1}{1-\gamma}
\mathbb{P}^a_\pi(s^\otimes\models\lozenge \mathrm{unsafe}).
\end{equation*}
where, 
\begin{equation*}
 \gamma_t =
\begin{cases}
    \gamma & \text{if $s_t^\otimes \in \F^\otimes$,}\\
    1 &\text{otherwise.}
\end{cases} \quad \text{and} \quad \Gamma_t = \prod^{t-1}_{i=0} \gamma^i, \, \Gamma_0 = 1
\end{equation*}
Taking the infimum over policies,
\begin{equation*}
    \pi^*(s^\otimes, a) = \arg \inf_{\pi \in \Pi^\otimes} Q_\pi(s^\otimes,a)
\end{equation*}
gives the optimal policy $\pi^*$ that minimizes the eventual-discounted return, which as established in Proposition \ref{prop:maximally} is a maximally safe policy. 

It remains to justify convergence of the update rule under the conditions of the Lemma.  Consider the following Bellman operator
\begin{equation*}
    \mathbf{B}_{\textbf{cost}} := \mathbb{E}^{\mathcal{M}^\otimes} \left[ \mathbf C (s^\otimes, a, {s^\otimes}') + \gamma({s^\otimes}') \min_{a'} Q({s^\otimes}', a) \right]
\end{equation*}
\end{proof}
for which the update rule in the Lemma statement is precisely the stochastic approximation of the fixed point equation,
\begin{equation*}
    Q = \mathbf{B}_{\textbf{cost}} Q
\end{equation*}
The operator $\mathbf{B}_{\textbf{cost}}$ is not necessarily a $\gamma$-contraction in the sup-norm as $\gamma(s^\otimes) = 1$ outside of $\F^\otimes$. Thus we require a proper stochastic shortest path condition. Assume that there exists at least one safe policy $\bar \pi$ such that, from every state-action pair with $p^*(s^\otimes, a) > 0$, the accepting set $\F^\otimes$ is reached with positive probability in finite time. Equivalently the product MDP decomposes into an absorbing unsafe class $\F^\otimes$ and at least one safe recurrent class that is disjoint from $\F^\otimes$. The holds under the assumption of existence of a non-empty maximally safe set which as established in the proof of Proposition \ref{prop:maximally} is true if there exists at least one safe policy, which in turn follows from the Theorem statement (all initial states have probability not more than $\epsilon$ of reaching $\F^\otimes$ under $\pi_{\text{safe}}$).

On any closed safe recurrent class the accumulated cost is precisely zero, while on $\F^\otimes$ the accumulated cost is $1 / (1 - \gamma)$. This is exactly a stochastic avoidance, or stochastic shortest-path, structure associated with the reachability objective. 

Under this condition $\mathbf{B}_{\textbf{cost}}$ has a unique fixed point on the interval,
\begin{equation*}
    0 \leq Q(s^\otimes, a) \leq \frac{1}{1-\gamma}
\end{equation*}
and this fixed point is the optimal avoidance value as shown above. Consider writing the update rule equivalently as follows,
\begin{equation*}
    Q_{t+1}(s^\otimes,a)
=
Q_t(s^\otimes,a)
+
\alpha_t(s^\otimes,a) \left[ (\mathbf{B}_{\textbf{cost}} Q_t)(s^\otimes,a) - Q_t(s^\otimes,a) + M_{t+1} \right]
\end{equation*}
where $M_{t+1}$ is the martingale-difference noise induced by sampling ${s^\otimes}'\sim \P^\otimes(\cdot\mid s^\otimes,a)$. Since $\mathbf{C}(s^\otimes,a, {s^\otimes}') \in [0, 1]$ and $\gamma({s^\otimes}') \in [\gamma, 1]$ for all $(s^\otimes,a, {s^\otimes}')$, and the iterates are initialized in the interval $[0, 1/(1-\gamma)]$, the update targets are bounded. Thus, the noise has bounded conditional second moment.

Standard stochastic approximation results for finite stochastic shortest-path problems can thus be applied \citep{jaakkola1993convergence}. Furthermore, the usual Robbins-Monro conditions and infinite visitation of every state-action pair implies that,
$Q_t(s^\otimes, a) \to Q^*(s^\otimes, a) \quad \text{(almost surely)}$
as $t \to \infty$ and for every $(s^\otimes,a)\in\S^\otimes\times\A$. Finally, because,
\begin{equation*}
    Q^*(s^\otimes,a)
=
\frac{1}{1-\gamma}p^*(s^\otimes,a)
\end{equation*}
and $1/(1-\gamma)>0$, minimizing $Q^*(s^\otimes,a)$ over actions is equivalent to minimizing $p^*(s^\otimes,a)$. Therefore every greedy policy with respect to $Q^\star$ is maximally safe. Since $Q_t\to Q^*$ almost surely and the greedy policy uses a fixed deterministic tie-breaking rule, for every product state visited infinitely often,

\begin{equation*}
    \lim_{t\to\infty}
\mathbb{P}\!\left(
\pi_t(s^\otimes)\in
\arg\min_{a\in\A}p^\star(s^\otimes,a)
\right)
=1
\end{equation*}
This completes the proof.

\paragraph{Convergence of the indicator} We provide the following Lemma to establish this result.

\begin{lemma}[Convergence of the indicator] \label{lem:indicator} Denoting $\indt{ a \in S(\pi_{\text{safe}}, a, \epsilon_t) }$ as the time varying empirical indicator for the true function $\ind{ a \in S(\pi_{\text{safe}}, a, \epsilon_t) }$, then
\begin{equation*}
    \indt{ a \in S(\pi_{\text{safe}}, a, \epsilon_t) } \to \ind{ a \in S^*(a, \epsilon_t) } \quad \forall a \in A \quad (\text{almost surely})
\end{equation*}
as $t \to \infty$
\end{lemma}

\begin{proof} By Lemma~\ref{lem:safe-policy-convergence},
$\pi_{\mathrm{safe}}$ converges almost surely to a maximally safe policy.
Therefore the reachability probabilities defining
$S(\pi_{\mathrm{safe}},a,\epsilon_t)$ converge to those defining
$S^\star(a,\epsilon_t)$.

In the \textsc{ModelBased} case we argue that $\widehat \P \to \P$ almost surely under infinite exploration. Fix a state-action pair $(s^\otimes, a)$ visited infinitely often, then given that $\widehat \P$ is a consistent unbiased estimator for, we have,
\begin{equation*}
	\widehat \P(s' \mid, s, a) \to \P(s' \mid s, a) \quad (\text{almost surely})
\end{equation*}
Since finite-horizon reachability probabilities are
continuous polynomials in the transition probabilities over horizon $H^*$, then this is sufficient to complete the proof for this mode. 

In the \textsc{ModelFree} case \eqref{eq:finitiebellman} is a finite-horizon dynamic-programming
recursion for the probability of reaching $\F^\otimes$ within $H^\star$ steps when the first action is $a$ and the safe policy is followed thereafter. Since the horizon is finite, $\mathbf{B}_{\textbf{fin}}$ can be evaluated backward from depth $H^\star$. Given each operator is non-expansive and bounded in $[0, 1]$, and the stochastic approximation thereof has noise that is a martingale-difference sequence with bounded variance, then standard stochastic approximation results \citep{jaakkola1993convergence} yield the following result,
\begin{equation*}
Q^{(0)}_t \to 
    \mathbb{P}^a_{\pi^\star_{\mathrm{safe}}}
(s^\otimes\models\lozenge^{\le H^\star}\mathrm{unsafe})
\quad\text{(almost surely)}
\end{equation*}
as $t \to \infty$. Therefore the time-varying indicator converges almost surely to the true maximal safe set indicator $s^\otimes_t \in S^*(a, \epsilon_t)$. Assumption \ref{ass:finite-horizon} then does the heavy lifting transferring the finite-horizon check to an infinite-horizon safety guarantee.
\end{proof}

Before continuing with the convergence properties of the Q-iterates we formally establish the following assumption which follows from Lemma \ref{lem:indicator}. 

\begin{assumption}[Eventual correctness of the indicator] \label{ass:eventual}  For every $(s^{\otimes},a) \in \S^{\otimes} \times \A$ there exists some (finite, possibly random) time $T_{(s^{\otimes},a)}$ such that,
\begin{equation*}
    \indt{ a \in S(\pi_{\text{safe}}, a, \epsilon_t) } = \ind{ a \in S^*(, a, \epsilon_t)} \quad \forall t \geq T_{(s^{\otimes},a)} \quad (\text{almost surely})
\end{equation*}
Thus the indicator may be wrong at first but (almost surely) stops making mistakes. 
\end{assumption}

\paragraph{Convergence of the shielded policy.} We provide the following Lemma to establish this result.

\begin{lemma} Let $\mathbf{B}_{\textbf{safe}, t}$ be the time-varying analogue for the safety-aware Bellman operator $\mathbf{B}_{\textbf{safe}}$, concretely,
\begin{equation*}
    (\mathbf{B}_{\textbf{safe}, t}Q)(s^{\otimes}, a) :=
    \;\begin{cases}
        \mathbb{E}^{\mathcal{M}^\otimes} \left[ \R(s, a, s') + \gamma \max_{a'} Q({s^{\otimes}}', a')\right] & \mathds{1}_t \left[ a \in \S^{\otimes}(\pi_{\text{safe}}, a, \epsilon_t) \right]\\
        0 & \text{otherwise}\\
    \end{cases}
\end{equation*}
Then under Assumption \ref{ass:eventual}, for the corresponding update step,
\begin{equation}
    Q_{t+1}(s^{\otimes}, a) = (1 - \alpha_t(s^{\otimes}, a)) Q_t(s^{\otimes}, a) + \alpha_t(s^{\otimes}, a) (\mathbf{B}_{\textbf{safe}, t} Q_t)(s^{\otimes}, a) \label{eq:stochasticapprox}
\end{equation}
with step-sizes satisfying the usual Robbins-Monro conditions and such that every $(s^\otimes, a)$ is visited infinitely often. Then with probability 1,
\begin{equation*}
    \lim_{t \to \infty} Q_{t+1}(s^{\otimes}, a)= Q^*_{\text{safe}}(s^{\otimes}, a)
\end{equation*}
where $Q^*_{\text{safe}}$ is the unique fixed point of $\mathbf{B}_{\textbf{safe}}$.
\end{lemma}

\begin{proof} The only difference between $\mathbf{B}_{\textbf{safe}, t}$ and $\mathbf{B}_{\textbf{safe}}$ is where a wrong safety indicator is used. Since all rewards are bounded and non-negative (i.e., $0\le \R(s,a,s')\le R_{\max}<\infty$ for all $(s, a, s') \in \S \times \A \times \S$), then we have a simple bound,
\begin{equation*}
    \Vert(\mathbf{B}_{\textbf{safe}, t}Q_t)(s^{\otimes},a) - (\mathbf{B}_{\textbf{safe}}Q_t)(s^{\otimes},a) \Vert_{\infty} = \Vert \Delta_t \Vert_{\infty} \leq  \Vert\varepsilon_t(s^\otimes, a) (R_{\max} +  \gamma Q_t ) \Vert_{\infty} 
\end{equation*}
where $\varepsilon_t(s^{\otimes}, a) = \left\vert \indt{ a \in S(\pi_{\text{safe}}, a, \epsilon_t) }- \ind{ a \in S^*(, a, \epsilon_t) } \right\vert$ is the 0-1 error. By Assumption \ref{ass:eventual} we have, $\varepsilon_t(s^{\otimes}, a)$ is eventually zero for all $(s^{\otimes}, a) \in \S^{\otimes} \times \A$. Thus,
\begin{equation*}
    \lim_{t=\infty}\Delta_t = 0
\end{equation*}
We are now equipped to apply standard stochastic approximation results \cite{jaakkola1993convergence}, as we have already established that $\mathbf{B}_{\textbf{safe}}$ is a $\gamma$-contraction with a unique fixed point in Proposition \ref{prop:bellmanoperator}. Furthermore, the total bias injected into each $(s^{\otimes}, a) \in \S^{\otimes} \times \A$ is bounded (since $\lim_{t=\infty}\Delta_t = 0$), and thus,
\begin{equation*}
    \sum_{t}\alpha_t(s^{\otimes}, a)\Delta_t(s^{\otimes}, a) < \infty \quad \text{(almost surely)}
\end{equation*}
thus we meet the sufficient conditions for convergence of the update rule given by,
\begin{equation}
    Q_{t+1}(s^{\otimes}, a) = (1 - \alpha_t(s^{\otimes}, a)) Q_t(s^{\otimes}, a) + \alpha_t(s^{\otimes}, a) (\mathbf{B}_{\textbf{safe}, t} Q_t)(s^{\otimes}, a) \label{eq:update-rule}
\end{equation}
Noting that as $t \to \infty$, we have $\mathbf{B}_{\textbf{safe}, t} Q^*_{\text{safe}} \to \mathbf{B}_{\textbf{safe}} Q^*_{\text{safe}}$ and by uniqueness the Q-iterates for \eqref{eq:update-rule} converges (almost surely) with probability 1 to $Q^*_{\text{safe}}$.
\end{proof}

\paragraph{Persistent $\delta$ corruptions.} 

\begin{lemma}[Optimal policy under $\delta$-corruptions] Define the $\delta$-corrupted safety-aware Bellman operator,
\begin{equation}
    (\mathbf{B}_{\textbf{safe}, \delta}Q)(s^{\otimes}, a) \begin{cases}
        (1 - \delta) (\mathbf{B}_{\text{safe}}Q) & \text{if $\ind{ a \in \S^*(, a, \epsilon_t)}$}\\
        \delta (\mathbf{B}Q) & \text{otherwise}
    \end{cases} \label{eq:bellmanadversarial}
\end{equation}
where $\mathbf B$ is the usual Bellman optimality operator defined as,
\begin{equation}
    (\mathbf{B} Q) := \mathbb{E}^{\mathcal{M}^\otimes} \left[ \R(s, a, s') + \gamma \max_a' Q({s^\otimes}', a') \right]
\end{equation}
the Q-iterates for the corresponding stochastic approximation,
\begin{equation}
    Q_{t+1}(s^{\otimes}, a) = (1 - \alpha_t(s^{\otimes}, a)) Q_t(s^{\otimes}, a) + \alpha_t(s^{\otimes}, a) (\mathbf{B}_{\textbf{safe}, \delta} Q_t)(s^{\otimes}, a) \label{eq:deltaupdate}
\end{equation}
converge to,
\begin{equation*}
    \Vert Q_t - Q^*_{\text{safe}} \Vert_{\infty} \to O(\varepsilon_{\delta}/(1-\gamma)) \leq \delta R_\text{max}/(1-\gamma)^2 \quad (\text{almost surely})
\end{equation*}
under the usual Robbins-Monro conditions and such that every $(s^\otimes, a)$ is visited infinitely often.
\end{lemma}
\begin{proof} It can easily be shown that $\mathbf{B}_{\textbf{safe}, \delta}$ is still a $\gamma$-contraction but perturbed by $\delta$, i.e, for every function $Q_1, Q_2$ on the metric space $\S^{\otimes} \times \A$, we have,
\begin{equation*}
    \Vert(\mathbf{B}_{\textbf{safe}, \delta}Q_1) - (\mathbf{B}_{\textbf{safe}, \delta}Q_2) \Vert_{\infty} \leq \gamma \Vert Q_1 - Q_2 \Vert_{\infty}
\end{equation*}
Thus the stochastic approximation in \eqref{eq:deltaupdate} converges to a unique fixed point. It remains to determine where this fixed point is in relation to $Q^*_{\text{safe}}$ (the fixed point of $\mathbf B_{\textbf{safe}}$).

We define the following quantity,
\begin{equation*}
    \varepsilon_{\delta} = \sup_{Q} \Vert (\mathbf{B}_{\textbf{safe}, \delta}Q) - (\mathbf{B}_{\textbf{safe}}Q) \Vert_{\infty}
\end{equation*}
Noting that for all $Q$ we have,
\begin{equation*}
    (\mathbf{B}_{\text{safe}}Q) \leq \max_{\pi, s} V_{\pi}(s) \leq R_\text{max}/(1-\gamma)
\end{equation*}
Then by \eqref{eq:bellmanadversarial} we have,
\begin{equation*}
    \varepsilon_{\delta}  \leq \delta R_\text{max}/(1-\gamma)
\end{equation*}
then standard perturbation results for $\gamma$-contractions give the desired result,
\begin{equation*}
    \Vert Q_t - Q^*_{\text{safe}} \Vert_{\infty} \to O(\varepsilon_{\delta}/(1-\gamma)) \leq \delta R_\text{max}/(1-\gamma)^2 \quad (\text{almost surely})
\end{equation*}
\end{proof}

\paragraph{Conclusion.} In \textsc{Exact} mode, the indicator $\ind{ a \in S^*(a, \epsilon_t) }$ is computed directly  from
the true product MDP, so the task update is a stochastic approximation of the fixed Bellman operator $\mathbf B_{\mathrm{safe}}$. Proposition \ref{prop:bellmanoperator} gives that $\mathbf B_{\mathrm{safe}}$ is a $\gamma$-contraction hence $Q_\text{task}$ almost surely converges to its unique fixed point $Q^*_\text{safe}$.

In \textsc{ModelBased} mode, Lemma \ref{lem:safe-policy-convergence} establishes convergence of $\pi_{\text{safe}}$ to the maximally safe policy, and Lemma \ref{lem:indicator} gives eventual correctness of the safe set indicator.  After the almost-sure finite time at which the indicator is
correct, the task policy update is again an ordinary stochastic approximation of $\mathbf B_{\mathrm{safe}}$, and therefore converges almost surely to $Q^*_\text{safe}$. When statistical model checking is used, the only additional error is the failure probability $\delta$. By standard the Hoeffding bounds used in Lemma \ref{lem:bounded}, $m$ samples paths gives failure probability decreasing exponentially in $m$, thus $\delta$ can be driven to zero rapidly, with only logarithmic dependence on
the desired confidence.

In \textsc{ModelFree} mode, the same conclusion follows from convergence of the finite-horizon safety value function Bellman update established Lemma \ref{lem:indicator}, The
learned indicator becomes correct almost surely, after which the task update is again governed by the operator $\mathbf B_{\mathrm{safe}}$. Therefore, in all three modes,
\begin{equation*}
    Q_{\mathrm{task},t}\to Q^*_{\mathrm{safe}} \quad \text{(almost surely)}
\end{equation*}
as $t \to \infty$. Finally, Theorem \ref{thm:constraintsatisfaction} applies to the deployed shielded policy $\pi^\lozenge$ which enforces the step-wise risk levels $\epsilon_t$ that satisfy $ \sum_t^\infty \epsilon_t \leq \epsilon$. The by construction $\pi^\lozenge$ satisfies the global safety constraint in \eqref{eq:problem}, which completes the proof.

\end{proof}

\clearpage
\newpage

\section{\textsc{RecReg} with Deep RL}
\label{sec:deep-rl}

\subsection{Overview}

\textsc{RecReg} with DreamerV3's RSSM \citep{hafner2025mastering} is described in Section \ref{sec:deep-recreg}. Here we provide a high level overview. Figure \ref{fig:prediction} illustrates how the probability $\mathbb{P}_\pi^a (s^\otimes_t \models \lozenge \mathrm{unsafe})$ is estimated. Algorithm \ref{alg:ambs} provides the high-level algorithmic framework for \textsc{Deep-RecReg}. 

\begin{figure}[H]
    \centering
    \includegraphics[width=0.85\linewidth]{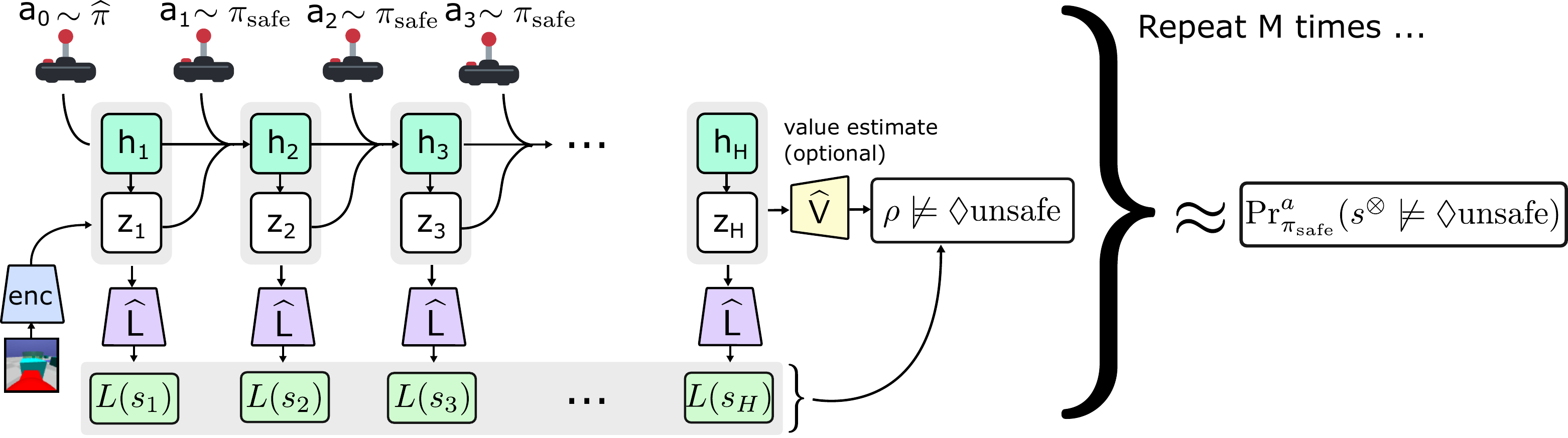}
    \caption{Estimating the safety probability with RSSM.}
    \label{fig:prediction}
\end{figure}

\begin{algorithm}[H]
\caption{\textsc{RecReg} with DreamerV3}
\label{alg:ambs} 
\raggedright
\textbf{Input:} $\mathcal{D}$, hyperparameters: $\epsilon_t$, $m$, $H$.\\
\textbf{Initialize:} replay buffer $D$, DreamerV3 parameters $\theta$, task policy $\widehat\pi$ and safe policy $\pi_{\text{safe}}$.
\begin{algorithmic}[1]
\For{{\bf each episode}}
\State Observe observation $o_0$.
\For{t = 1, \ldots, T}
\State \emph{// Shielding with the latent world model}
\State Sample $a_t \sim \widehat \pi$.
\State Using $a_t$ and $\pi_{\text{safe}}$ sample $m$ trajectories $M = \langle\{\hat o_{t:t+H}, \hat r_{t:t+H}, \hat c_{t:t+H}\} \rangle_{i=0}^{m} \sim p_{\theta}$.
\State \emph{// Compute the probability estimate}
\State $\bar X \gets \frac{1}{m}  \sum^m_{i=0} \max(\hat c_{t:t+H}) + \text{(optional value estimate $\, \ind{\hat v_c > c \cdot \gamma^H}$)}$.
\State $a_t \sim \pi_{\text{safe}}$ \textbf{if} $\bar X > \epsilon_t$.
\State Execute $a_t$ and observe $(o_{t+1}, r_t)$.
\State Compute DFA state $q_{t+1}$ and let $c_t \gets \ind{q_{t+1} \in \F}$. 
\State Append $(o_t, a_t, r_t, c_t, o_{t+1})$ to the replay buffer $D$.
\If{\texttt{time\_to\_update}}
\State \emph{// World model learning}
        \State Sample a batch $B$ of transition sequences $ \{ ( o_{t'}, a_{t'}, r_{t'}, c_{t'}, o_{t'+1} )\} \sim D$.
        \State Update DreamerV3 parameters $\theta$ with maximum likelihood \cite{hafner2023mastering}.
        \State \emph{// Task policy optimization}
        \State Sample sequences $\{\hat o_{t':t'+N}, \hat r_{t':t'+N}, \hat c_{t':t'+N}\} \sim p_{\theta}$ with the task policy $\widehat\pi$.
        \State Update the task policy $\widehat\pi$ with RL (to maximize reward).
        \State Update the corresponding value critics with maximum likelihood.
        \State \emph{// Safe policy optimization}
        \State Sample sequences $\{\hat o_{t':t'+N}, \hat r_{t':t'+N}, \hat c_{t':t'+N}\} \sim p_{\theta}$ with the safe policy $\pi_{\text{safe}}$.
        \State Update the safe policy $\pi_{\text{safe}}$ with RL (to minimize cost).
        \State Update the corresponding value critics with maximum likelihood.
\EndIf
\EndFor
\EndFor
\end{algorithmic}
\end{algorithm}

\subsection{Implementation Details}

Here we provide additional implementation details including environment descriptions and training details.

\begin{figure}[t]
    \centering
    \includegraphics[width=0.35\linewidth]{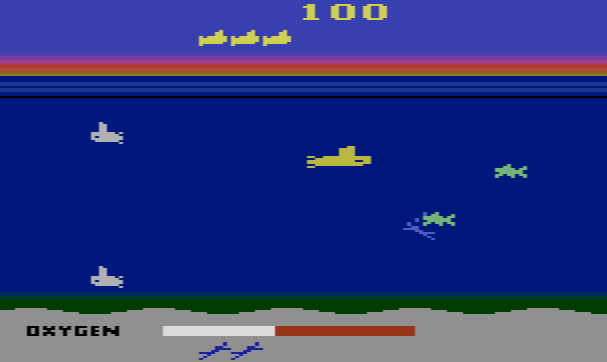}
    \caption{Seaquest environment}
    \label{fig:seaquest}
\end{figure}

\paragraph{Atari Seaquest.}  Seaquest is an Atari 2600 game provided as part of the Arcade Learning Environment (ALE) \cite{machado18arcade} (c.f., Fig.~\ref{fig:seaquest}).
\begin{itemize}
    \item \textbf{Dynamics:} the agent (submarine) can navigate freely in the water.
    \item \textbf{Action space:} Seaquest utilizes the full Atari controller action space (18 distinct actions), which includes all possible combinations of FIRE, UP, DOWN, LEFT, RIGHT.
    \item \textbf{State space:} Seaquest is partially observable, the observations correspond to $210 \times 160 \times 3$ pixel images provided by the ALE. The observations are pre-processed and scaled to $64 \times 64 \times 3$ tensors corresponding to pixel values with RGB information intact.
    \item \textbf{Additional modifications:} We also use \texttt{frame\_skip=4} (every 4 frames are skipped with the agent's action repeated) and \texttt{sticky\_actions=0.25} (there is a 0.25 chance that the agent's previous action is repeated), which is in line with the recommended settings for Atari games \cite{machado18arcade}.
    \item \textbf{Reward function:} the agent receives reward for collecting divers and returning them to the surface, and ``shooting'' enemy sharks and submarines.
    \item \textbf{Episode length:} $T=27,000$.
    \item \textbf{Safety properties}:
    \begin{itemize}
        \item property\_1 ($\varphi_1$): the agent must only surface with a diver, avoid running out of oxygen and avoid being hit by enemy ``sharks'' and ``submarines'', i.e., $\varphi_1 = \square((\text{surface} \land X \neg \text{surface}) \to \square(\text{surface} \to \text{diver})) \land \square (\neg \text{out-of-oxygen}) \land \square(\neg \text{hit})$. 
        \item property\_2 ($\varphi_2$): the agent must return to the surface in at most 30 timesteps after picking up a diver, i.e., $\varphi_2 = \square(\texttt{diver} \to \lozenge^{\leq 30} \texttt{surface})$.
    \end{itemize}
    \item \textbf{Automaton size:} $|\Q|=5$ for property\_1 and $|\Q|=32$ for property\_2.
    \item \textbf{Cost thresholds:}
    \begin{itemize}
        \item \textsc{RecReg} and AMBS: $\epsilon_t = 0.01$.
        \item DreamerV3 (Aug. Lag.), PPO-Lag and CPO: $1.0$
    \end{itemize}
    \item \textbf{Cost coefficient:} $c = 10.0$.
    \item \textbf{Model checking horizon}: $H=30$ for Property 1 ($\varphi_1$) and $H=50$ for Property 2 ($\varphi_2$).
    \item All other hyperparameters remained untouched and are detailed in Table~\ref{tab:hyperparams1}
\end{itemize}

\paragraph{SafetyGym.} SafetyGym \cite{ray2019benchmarking} is built on the advanced physics simulator, MuJoCo \cite{todorov2012mujoco}. In our experiments we use three separate environments: {\em PointGoal1}, {\em PointGoal2} and {\em CarGoal1}. The first two environment use the {\em Point} vehicle for navigation in the environment and the third environment uses the {\em Car} vehicle. 

\begin{figure}[H]
    \centering
    \begin{minipage}{0.5\textwidth}
    \begin{subfigure}{0.4\textwidth}
        \centering
        \includegraphics[width=\textwidth]{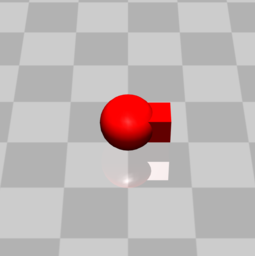}
        \caption{Point}
    \end{subfigure}
    \hfill
    \begin{subfigure}{0.4\textwidth}
        \centering
        \includegraphics[width=\textwidth]{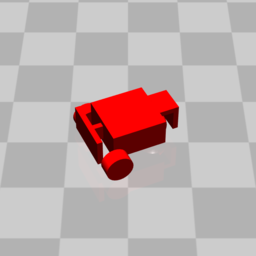}
        \caption{Car}
    \end{subfigure}
    \caption{Safety Gym vehicles.}
    \label{fig:vehicles}
    \end{minipage}
\end{figure}

\begin{itemize}
    \item \textbf{Dynamics}: the agent navigates a static environment with one of two vehicles (c.f., Fig~\ref{fig:vehicles}):
    \begin{itemize}
        \item \emph{Point}: A simple vehicle that is constrained to the 2D plane. One actuator controls the rotation of the vehicle the other controls its linear movement. The simple control scheme makes the robot easy to control. The corresponding action space is $[-1, +1]^2$.
        \item \emph{Car}: This vehicle exhibits more realistic dynamics, with two independently driven wheels. The mechanics are 3D, although the vehicle mostly resides in the 2D plane. The vehicle has a free moving wheel at the front that is not connected to any actuators. The control scheme is a little more advanced since, rotating, and moving forward and backwards require coordination between the two driven wheels. However, the action space is identical, i.e.~$[-1, +1]^2$.
    \end{itemize}
    \item \textbf{Action space:} $[-1, +1]^2$.
    \item \textbf{State space:} the agent observes $64 \times 64 \times 3$ pixel images rendered from a front facing camera in MuJoCo. As a result the environment is partially observable. 
    \item \textbf{Reward function:} the agent receives a reward of $+1$ for entering a goal state (c.f., Fig \ref{fig:goal}). The goal state and obstacles are then re-initialized to randomly generated positions. 
    \item \textbf{Episode length:} $T=1000$.
    \item \text{Safety properties:}
    \begin{itemize}
        \item For {\em PointGoal1} and {\em CarGoal1}: $\varphi = \square \neg\,\text{hazard}$
        \item For {\em PointGoal2}: $\varphi = \square \neg ( \text{hazard} \land \text{vase})$
    \end{itemize}
    \item \textbf{Automaton size:} $\vert \Q \vert = 2$.
    \item \textbf{Cost thresholds:}
    \begin{itemize}
        \item \textsc{RecReg} and AMBS: $\epsilon_t = 0.01$.
        \item DreamerV3 (Aug. Lag.), PPO-Lag and CPO: $1.0$
    \end{itemize}
    \item \textbf{Cost coefficient:} $c = 10.0$.
    \item \textbf{Model checking horizon}: $H=30$
    \item All other hyperparameters remained untouched and are detailed in Table \ref{tab:hyperparams1}.
    \end{itemize}

\begin{figure}[H]
    \centering
    \begin{minipage}{0.55\textwidth}
    \begin{subfigure}{0.3\textwidth}
        \centering
        \includegraphics[width=\textwidth]{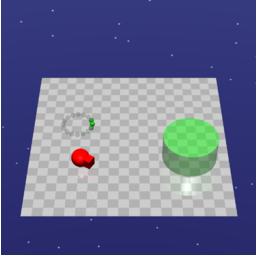}
        \caption{Goal positions}
        \label{fig:goal}
    \end{subfigure}
    \hfill
    \begin{subfigure}{0.3\textwidth}
        \centering
        \includegraphics[width=\textwidth]{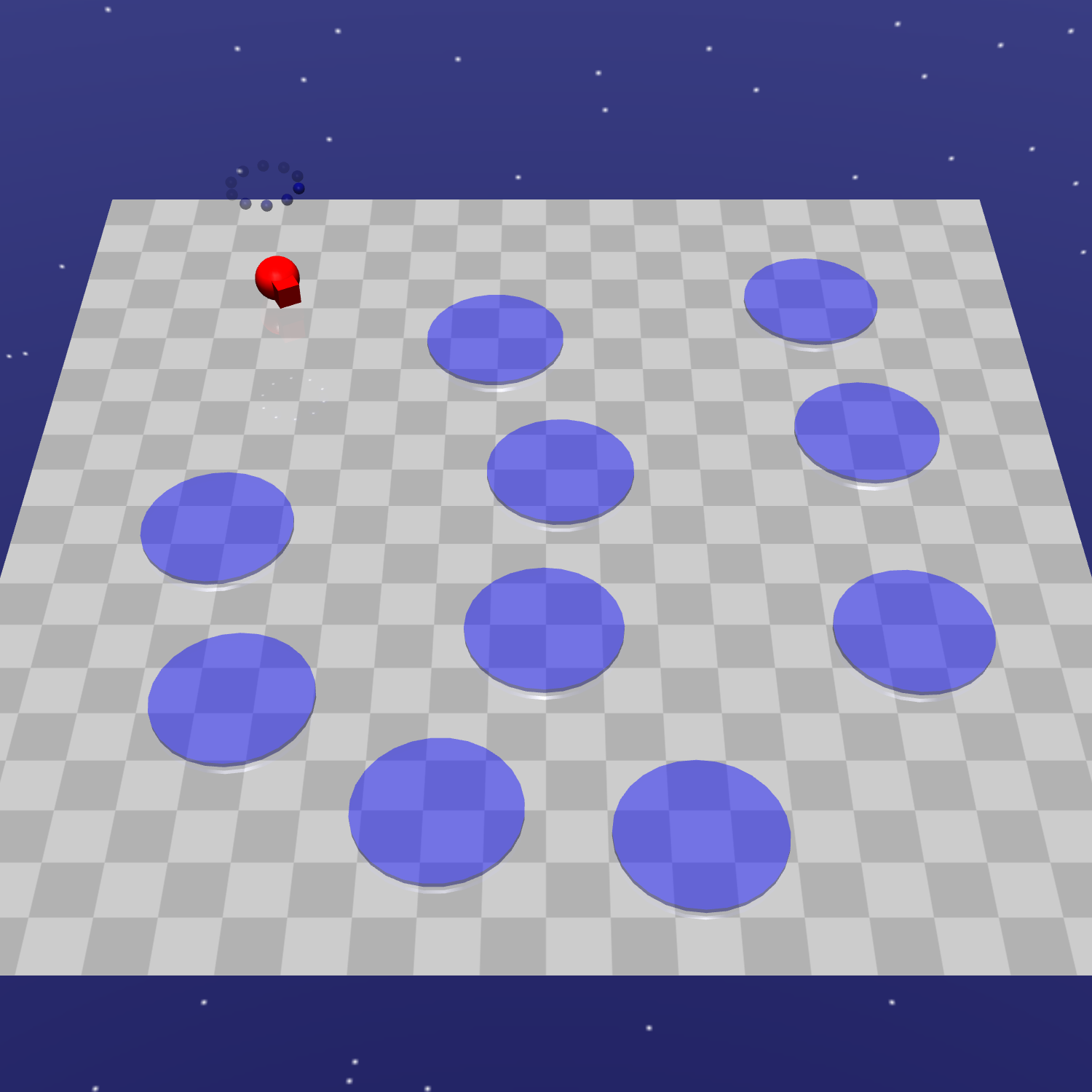}
        \caption{Hazards}
        \label{fig:hazards}
    \end{subfigure}
    \hfill
    \begin{subfigure}{0.3\textwidth}
        \centering
        \includegraphics[width=\textwidth]{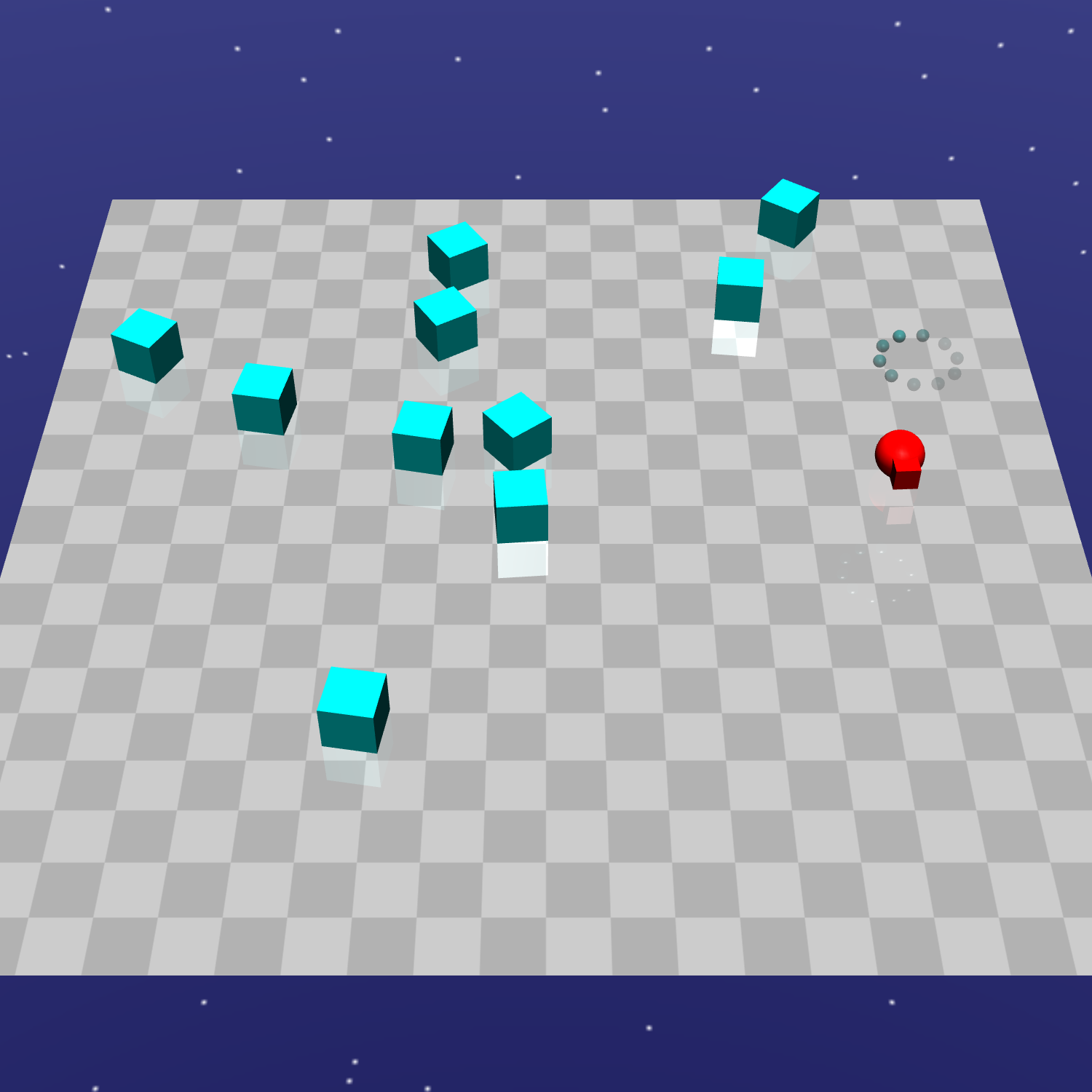}
        \caption{Vases}
        \label{fig:vases}
    \end{subfigure}
    \caption{Safety Gym tasks and constraints.}
    \label{fig:tasksandconstraints}
    \end{minipage}
\end{figure}

\paragraph{Access to code.} The code for the additional experiments on Atari Seaquest is our own private branch of the AMBS code base \cite{goodall2023approximate}, please refer to \url{https://github.com/sacktock/AMBS} (MIT License) for setup instructions. Our private code can be made available upon reasonable request. Furthermore, the implementations for PPO-Lag \cite{ray2019benchmarking} and CPO \cite{achiam2017constrained} were provided by \emph{Omnisafe} \cite{omnisafe}, for setup instructions please refer to the \emph{Omnisafe} code base \url{https://github.com/PKU-Alignment/omnisafe} (Apache-2.0 license).

\paragraph{Training details.}  We has access to 2 NVIDIA Tesla A40 (48GB RAM) GPU and a 24-core/48 thread Intel Xeon CPU each with 32GB-64GB of additional RAM. Each run for Atari Seaquest can take 8 hours to 1 day (per seed) depending on the exact configuration of DreamerV3, in general we see a slow down of factor of roughly 2 when using AMBS and \textsc{RecReg} compared to the unmodified DreamerV3 baseline.  Memory requirements may differ depending on the DreamerV3 configuration used, for the \textit{xlarge} configuration 32GB of GPU memory will suffice. 

\paragraph{Statistical significance.} Standard error (SE) intervals are provided for each of our experiments in Tab.~\ref{tab:seaquestresults} and Tab.~\ref{tab:safetygymresults}. In particular, we report 5 random initializations (seeds) for each experiment. In the learning curves, we plot the non-parametric (bootstrap) 95\% confidence intervals, provided by \texttt{seaborn.lineplot} with default parameters: \texttt{errorbar=(`ci', 95), n\_boot=1000}. The error bars capture the randomness in the initialization of the policy parameters (and DreamerV3 world model), the randomness of the environment and any randomness in the batch sampling.

\subsection{Atari Seaquest results}

We provide tabular results for the Atari Seaquest experiments (c.f., Table \ref{tab:seaquestresults}) and a larger version of Figure \ref{fig:seaquest-results} from the main text (c.f., Figure \ref{fig:seaquest-results-2}).

\begin{table}[H]
  \caption{Results for the deep RL experiments on Atari Seaquest. The violation rate is scaled by $10^{-4}$.}
  \label{tab:seaquestresults}
  \centering
  \tiny
  \begin{tabular}{llcccccc}
    \toprule
       &  & DreamerV3 & Aug. Lag. & AMBS & \textsc{Deep-RecReg} & PPO-Lag & CPO \\
    \midrule
    \texttt{seaquest} ($\varphi_1$) & Return & $ 2440 \pm 400 $& $ 706 \pm 162 $& $ 2350 \pm 340 $& $\bm{3140 \pm 326}$& $ 506 \pm 7.57 $& $ 543 \pm 47.4 $\\

    & Viol. rate & $ 5.30 \pm 0.657 $& $ 9.83 \pm 1.39 $& $\bm{4.20 \pm 0.215}$& $ 4.43 \pm 0.187 $& $ 25.6 \pm 0.468 $& $ 23.9 \pm 1.97 $\\
    \midrule
    \texttt{seaquest} ($\varphi_2$) & Return & $ 2440 \pm 400 $& $ 580 \pm 54.0 $& $ 2480 \pm 483 $& $\bm{2700 \pm 364}$& $ 79.6 \pm 10.5 $& $ 342 \pm 17.1 $\\
    & Viol. rate & $ 17.9 \pm 5.77 $& $\bm{1.86 \pm 0.891}$& $ 2.00 \pm 0.767 $& $ 2.52 \pm 0.673 $& $ 16.8 \pm 2.17 $& $ 22.0 \pm 4.45 $\\
    \bottomrule
  \end{tabular}
\end{table}

\begin{figure}[H]
    \centering
    \includegraphics[width=0.95\linewidth]{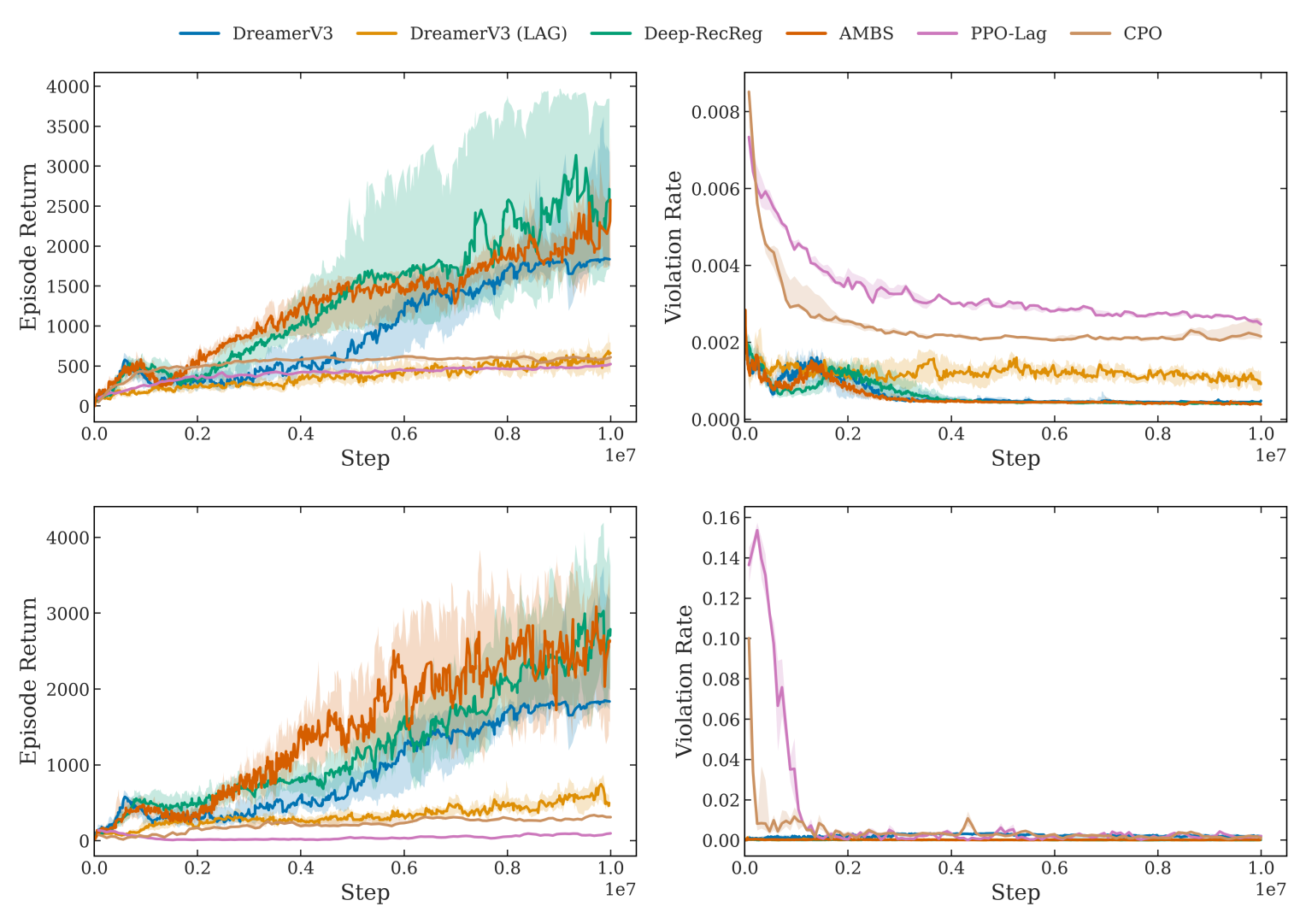}
    \caption{Learning curves for Atari Seaquest. \emph{Episodic reward} and \emph{violation rate} are averaged over 5 random seeds.}
    \label{fig:seaquest-results-2}
\end{figure}

\subsection{SafetyGym results}

We provide additional results for \textit{SafetyGym} \citep{ray2019benchmarking}. Table \ref{tab:safetygymresults} provides tabular scores for each of the experiments and Figure \ref{fig:safetygym-results} provides the learning curves.

\begin{table}[H]
  \caption{Results for the deep RL experiments on \textit{SafetyGym}.}
  \label{tab:safetygymresults}
  \centering
  \tiny
  \begin{tabular}{llcccc}
    \toprule
       &  & DreamerV3 & Aug. Lag. & AMBS & \textsc{Deep-RecReg} \\
    \midrule
    \textit{PointGoal1} & Return & $\bm{22.2 \pm 0.629}$& $ 19.1 \pm 0.661 $& $ 16.7 \pm 2.58 $& $ 17.3 \pm 0.628 $\\

    & Viol. rate & $ 0.0562 \pm 0.00239 $& $ 0.0111 \pm 0.00867 $& $\bm{0.00383 \pm 0.00346}$& $ 0.00500 \pm 0.00237 $\\
    \midrule
    \textit{PointGoal2} & Return & $\bm{20.6 \pm 1.22}$& $ 12.7 \pm 1.58 $& $ 6.42 \pm 1.04 $& $ 7.88 \pm 2.08 $\\
    & Viol. rate & $ 0.188 \pm 0.0177 $& $ 0.0337 \pm 0.0112 $& $\bm{0.0126 \pm 0.00909}$ & $ 0.0213 \pm 0.0179 $ \\
    \midrule
    \textit{CarGoal1} & Return & $\bm{17.4 \pm 2.13}$& $ 11.2 \pm 2.96 $& $ 7.26 \pm 3.52 $ &  $ 7.56 \pm 3.23 $\\
    & Viol. rate & $ 0.0836 \pm 0.0101 $& $ 0.0444 \pm 0.0308 $& $ 0.00552 \pm 0.00393 $& $\bm{0.00367 \pm 0.00415}$\\
    \bottomrule
  \end{tabular}
\end{table}

\begin{figure}[H]
    \centering
    \includegraphics[width=0.85\linewidth]{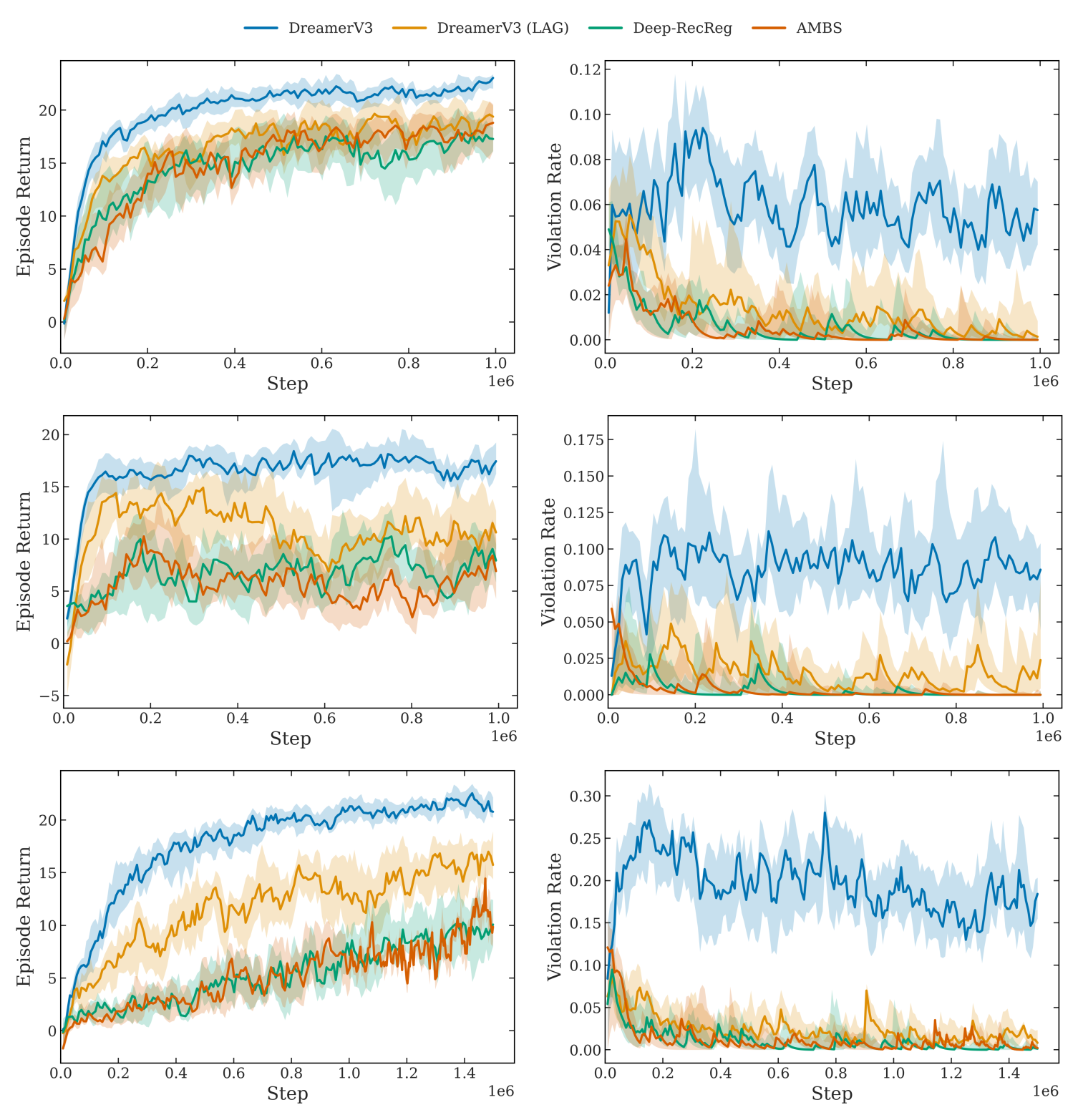}
    \caption{Learning curves for \emph{Safety Gym} environments. \emph{Episodic reward} and \emph{violation rate} are averaged over 5 random seeds.}
    \label{fig:safetygym-results}
\end{figure}

\paragraph{Discussion (\textit{Safety Gym}).} AMBS and \textsc{Deep-RecReg} are almost inseparable with markedly similar performance profiles. Both achieve relatively low violation rates but do not achieve the best returns. Broadly speaking \textsc{Deep-RecReg} achieves slightly better returns but slightly higher violation rates compared with AMBS across the board. On the other end of the spectrum, the unconstrained DreamerV3 achieves the highest returns but does not meaningfully reduce safety violations. The Augmented Lagrangian has a harder time at reducing safety violations likely due to fluctuating penalties but achieves higher returns than both AMBS and \textsc{Deep-RecReg}. 
\clearpage
\newpage

\begin{table}[t]
  \caption{Hyperparameter details for the DreamerV3-based algorithms: Aug. Lag., AMBS and \textsc{Deep-RecReg}.}
  \label{tab:hyperparams1}
  \centering
  \small
  \begin{tabular}{lcc}
    \toprule
       Name & Symbol & Value  \\
       \midrule
    Replay capacity& $D$ & $10^6$\\
    Batch size & $B$ & 16\\
    Batch length & - & 64\\
    Number of envs & - & 8 \\
    Train ratio & - & 64 \\
    Number of MLP layers & - & 5 \\
    Number of MLP units & - & 1024 \\
    Activation & - & LayerNorm + SiLU \\
    \midrule
    \multicolumn{3}{c}{World Model (shared)}\\
    \midrule
    Number of latents & - & 32\\
    Classes per latent & - & 32\\
    RSSM loss scales & $\beta_{\text{pred}}$, $\beta_{\text{dyn}}$, $\beta_{\text{rep}}$ & 1.0, 0.5, 0.1\\
    Predictor loss scales & $\beta_{o}, \beta_r, \beta_c, \beta_{\gamma}$  & 1.0, 1.0, 1.0, 1.0\\
    Learning rate & - & $1.0 \times10^{-4}$ \\
    Adam epsilon & $\epsilon_{\text{adam}}$& $1.0 \times10^{-8}$\\
    Gradient clipping & - & 1000 \\
    \midrule
    \multicolumn{3}{c}{World Model (Atari Seaquest)}\\
    \midrule
    Number of layers & - & 3 \\
    Number of hidden units & - & 640 \\
    Number of recurrent units & - & 1024 \\
    CNN depth & - & 48\\
    \midrule
    \multicolumn{3}{c}{World Model (\textit{SafetyGym})}\\
    \midrule
    Number of layers & - & 2 \\
    Number of hidden units & - & 768 \\
    Number of recurrent units & - & 2048 \\
    CNN depth & - & 64\\
    \midrule
    \multicolumn{3}{c}{Actor Critic (shared)}\\
    \midrule
    Roll-out horizon & - & 15\\
Discount factor & $\gamma$ & 0.997 \\
TD lambda & $\lambda$ & 0.95\\
Critic EMA decay & - & 0.98 \\
Critic EMA regularizer & -&1\\
Return norm. scale & $S_{\text{reward}}$ & $\text{Per}(R, 95) - \text{Per}(R, 5) $ \\
Return norm. limit & $L_{\text{reward}}$ & 1\\
Return norm. decay & - & 0.99\\
Actor entropy scale & $\eta_{\text{actor}}$ & $3.0 \times 10^{-4}$\\
Learning rate & - & $3.0 \times 10^{-5}$\\
Adam epsilon & $\epsilon_{\text{adam}}$& $1.0 \times 10^{-5}$\\
Gradient clipping & - & 100 \\
    \midrule
    \multicolumn{3}{c}{\textsc{Deep-RecReg} (and AMBS)}\\
    \midrule
    Number of samples & $m$ & 512\\
    Step-wise risk & $\epsilon_t$ & 0.01\\
    Failure probability & $\delta_t$ & 0.01 \\
    Lookahead/shielding horizon & $H$ & $\{30, 50\}$\\
    Approximation error & $\epsilon'$ & 0.01\\
    Cost coefficient (safe policy) & $c$ & $10.0$\\
    \midrule
    \multicolumn{3}{c}{Augmented Lagrangian}\\
    \midrule
    Penalty multiplier & $\mu_k$ & $5.0 \times 10^{-9}$ \\
    Initial Lagrange multiplier & $\lambda^k$ & $0.01$ \\
    Penalty power & $\sigma$ & $1.0 \times 10^{-6}$ \\
    Cost coefficient & $c$ & $1.0$ \\
    Cost threshold & - & $1.0$ \\
\bottomrule
  \end{tabular}
\end{table}
\clearpage
\newpage

\section{Additional Results}
\label{sec:additional-results}

Here we provide tabular summaries of the results presented in Figure \ref{fig:results}.

\begin{table}[H]
    \caption{Experimental results with \emph{raw returns}. Bold text denotes the best value among algorithms satisfying the safety threshold; grey text denotes algorithms that fail the safety threshold.}
    \label{tab:raw-returns1}
    \centering
    \tiny
    \begin{tabular}{llcccc}
    \toprule
\multirow{2}{*}{Algorithm} & \multirow{2}{*}{Metric} & MediaStreaming & BridgeCrossing & BridgeCrossingV2 & Colour...World \\
& & $1-\epsilon=0.96$ & $1-\epsilon=0.95$ & $1-\epsilon=0.99$ & $1-\epsilon=0.95$ \\
\midrule
\multirow{2}{*}{QL} & Ep. Return & \textcolor{gray}{$-0.017\pm 0.009$} & \textcolor{gray}{$0.303\pm 0.287$} & \textcolor{gray}{$0.986\pm 0.004$} & \textcolor{gray}{$1.000\pm 0.000$} \\
 & Safety Prob. & \textcolor{gray}{$0.000\pm 0.000$} & \textcolor{gray}{$0.368\pm 0.260$} & \textcolor{gray}{$0.986\pm 0.004$} & \textcolor{gray}{$0.000\pm 0.000$} \\
\midrule
\multirow{2}{*}{QL($\lambda$)} & Ep. Return & $-22.291\pm 0.377$ & $\mathbf{0.998\pm 0.001}$ & \textcolor{gray}{$0.414\pm 0.258$} & $\mathbf{1.000\pm 0.000}$ \\
 & Safety Prob. & $0.975\pm 0.011$ & $0.998\pm 0.001$ & \textcolor{gray}{$0.985\pm 0.010$} & $0.998\pm 0.001$ \\
\midrule
\multirow{2}{*}{\textsc{RecReg}} & Ep. Return & $\mathbf{-21.726\pm 0.234}$ & $\mathbf{0.998\pm 0.002}$ & $\mathbf{0.665\pm 0.272}$ & $0.911\pm 0.174$ \\
 & Safety Prob. & $0.961\pm 0.004$ & $0.998\pm 0.002$ & $0.995\pm 0.003$ & $0.996\pm 0.005$ \\
\midrule
\multirow{2}{*}{\textsc{RecReg-ModelBased}} & Ep. Return & $-22.877\pm 0.382$ & $0.997\pm 0.003$ & \textcolor{gray}{$0.681\pm 0.249$} & $0.706\pm 0.233$ \\
 & Safety Prob. & $0.985\pm 0.006$ & $0.998\pm 0.002$ & \textcolor{gray}{$0.985\pm 0.012$} & $0.980\pm 0.010$ \\
\midrule
\multirow{2}{*}{SEM} & Ep. Return & $-40.000\pm 0.000$ & \textcolor{gray}{$0.924\pm 0.085$} & \textcolor{gray}{$0.973\pm 0.010$} & \textcolor{gray}{$0.894\pm 0.158$} \\
 & Safety Prob. & $1.000\pm 0.000$ & \textcolor{gray}{$0.927\pm 0.078$} & \textcolor{gray}{$0.973\pm 0.010$} & \textcolor{gray}{$0.090\pm 0.157$} \\
\midrule
\multirow{2}{*}{LCRL} & Ep. Return & \textcolor{gray}{$-18.995\pm 0.301$} & \textcolor{gray}{$0.303\pm 0.287$} & \textcolor{gray}{$0.986\pm 0.004$} & $\mathbf{1.000\pm 0.000}$ \\
 & Safety Prob. & \textcolor{gray}{$0.438\pm 0.129$} & \textcolor{gray}{$0.368\pm 0.260$} & \textcolor{gray}{$0.986\pm 0.004$} & $0.968\pm 0.007$ \\
\midrule
\multirow{2}{*}{CPO} & Ep. Return & \textcolor{gray}{$-23.194\pm 0.310$} & $0.000\pm 0.000$ & $0.591\pm 0.315$ & $\mathbf{1.000\pm 0.000}$ \\
 & Safety Prob. & \textcolor{gray}{$0.957\pm 0.007$} & $0.982\pm 0.026$ & $0.993\pm 0.006$ & $0.969\pm 0.006$ \\
\midrule
\multirow{2}{*}{PPO-Lag} & Ep. Return & $-24.960\pm 0.355$ & $0.000\pm 0.000$ & \textcolor{gray}{$0.689\pm 0.295$} & $\mathbf{1.000\pm 0.000}$ \\
 & Safety Prob. & $0.990\pm 0.004$ & $0.959\pm 0.043$ & \textcolor{gray}{$0.979\pm 0.015$} & $0.974\pm 0.004$ \\
\midrule
\multirow{2}{*}{TRPO-Lag} & Ep. Return & $-24.147\pm 0.378$ & \textcolor{gray}{$0.000\pm 0.000$} & $0.495\pm 0.324$ & $\mathbf{1.000\pm 0.000}$ \\
 & Safety Prob. & $0.986\pm 0.005$ & \textcolor{gray}{$0.939\pm 0.102$} & $0.994\pm 0.006$ & $0.974\pm 0.004$ \\
\bottomrule
    \end{tabular}
\end{table}

\begin{table}[H]
    \caption{Experimental results with \emph{normalized returns}. Bold text denotes the best value among algorithms satisfying the safety threshold; grey text denotes algorithms that fail the safety threshold.}
    \label{tab:raw-returns2}
    \centering
    \tiny
    \begin{tabular}{llccccc}
    \toprule
\multirow{2}{*}{Algorithm} & \multirow{2}{*}{Metric} & ColourBomb...(Prop.~2) & ColourBomb...V2 & ColourBomb...V2 (Prop.~2) & Colour...V2 (Prop.~3) & MiniPacman \\
& & $1-\epsilon=0.95$ & $1-\epsilon=0.99$ & $1-\epsilon=0.99$ & $1-\epsilon=0.99$ & $1-\epsilon=0.99$ \\
\midrule
\multirow{2}{*}{QL} & $\ldots$ & \textcolor{gray}{$1.000\pm 0.000$} & \textcolor{gray}{$32.701\pm 0.599$} & \textcolor{gray}{$32.701\pm 0.599$} & \textcolor{gray}{$32.701\pm 0.599$} & \textcolor{gray}{$0.989\pm 0.003$} \\
 & $\ldots$ & \textcolor{gray}{$0.076\pm 0.007$} & \textcolor{gray}{$0.000\pm 0.000$} & \textcolor{gray}{$0.000\pm 0.000$} & \textcolor{gray}{$0.000\pm 0.000$} & \textcolor{gray}{$0.396\pm 0.018$} \\
\midrule
\multirow{2}{*}{QL($\lambda$)} & $\ldots$ & $0.996\pm 0.002$ & \textcolor{gray}{$20.484\pm 2.408$} & \textcolor{gray}{$22.663\pm 1.379$} & \textcolor{gray}{$22.394\pm 0.364$} & \textcolor{gray}{$0.925\pm 0.009$} \\
 & $\ldots$ & $0.996\pm 0.002$ & \textcolor{gray}{$0.877\pm 0.048$} & \textcolor{gray}{$0.949\pm 0.039$} & \textcolor{gray}{$0.947\pm 0.013$} & \textcolor{gray}{$0.883\pm 0.007$} \\
\midrule
\multirow{2}{*}{\textsc{RecReg}} & $\ldots$ & $0.935\pm 0.126$ & \textcolor{gray}{$21.358\pm 0.687$} & \textcolor{gray}{$21.952\pm 1.277$} & \textcolor{gray}{$26.466\pm 0.095$} & \textcolor{gray}{$0.912\pm 0.012$} \\
 & $\ldots$ & $0.996\pm 0.004$ & \textcolor{gray}{$0.980\pm 0.007$} & \textcolor{gray}{$0.978\pm 0.007$} & \textcolor{gray}{$0.979\pm 0.006$} & \textcolor{gray}{$0.894\pm 0.009$} \\
\midrule
\multirow{2}{*}{\textsc{RecReg-ModelBased}} & $\ldots$ & $0.885\pm 0.152$ & \textcolor{gray}{$18.523\pm 3.237$} & \textcolor{gray}{$24.897\pm 0.386$} & $\mathbf{23.218\pm 0.800}$ & \textcolor{gray}{$0.063\pm 0.011$} \\
 & $\ldots$ & $0.993\pm 0.005$ & \textcolor{gray}{$0.986\pm 0.004$} & \textcolor{gray}{$0.967\pm 0.005$} & $0.996\pm 0.003$ & \textcolor{gray}{$0.984\pm 0.005$} \\
\midrule
\multirow{2}{*}{SEM} & $\ldots$ & \textcolor{gray}{$0.960\pm 0.058$} & \textcolor{gray}{$14.334\pm 4.888$} & \textcolor{gray}{$14.988\pm 4.513$} & \textcolor{gray}{$14.815\pm 5.172$} & \textcolor{gray}{$0.980\pm 0.011$} \\
 & $\ldots$ & \textcolor{gray}{$0.198\pm 0.167$} & \textcolor{gray}{$0.014\pm 0.022$} & \textcolor{gray}{$0.025\pm 0.030$} & \textcolor{gray}{$0.022\pm 0.030$} & \textcolor{gray}{$0.386\pm 0.011$} \\
\midrule
\multirow{2}{*}{LCRL} & $\ldots$ & \textcolor{gray}{$0.991\pm 0.004$} & \textcolor{gray}{$25.525\pm 0.311$} & \textcolor{gray}{$17.474\pm 0.483$} & \textcolor{gray}{$17.134\pm 0.377$} & \textcolor{gray}{$0.989\pm 0.004$} \\
 & $\ldots$ & \textcolor{gray}{$0.945\pm 0.008$} & \textcolor{gray}{$0.538\pm 0.035$} & \textcolor{gray}{$0.412\pm 0.021$} & \textcolor{gray}{$0.504\pm 0.016$} & \textcolor{gray}{$0.396\pm 0.018$} \\
\midrule
\multirow{2}{*}{CPO} & $\ldots$ & $\mathbf{1.000\pm 0.000}$ & \textcolor{gray}{$20.905\pm 0.567$} & $\mathbf{19.580\pm 3.937}$ & \textcolor{gray}{$11.333\pm 5.236$} & \textcolor{gray}{$0.485\pm 0.127$} \\
 & $\ldots$ & $0.963\pm 0.010$ & \textcolor{gray}{$0.978\pm 0.019$} & $0.993\pm 0.005$ & \textcolor{gray}{$0.974\pm 0.021$} & \textcolor{gray}{$0.973\pm 0.012$} \\
\midrule
\multirow{2}{*}{PPO-Lag} & $\ldots$ & \textcolor{gray}{$1.000\pm 0.000$} & \textcolor{gray}{$21.426\pm 0.283$} & \textcolor{gray}{$22.672\pm 1.007$} & \textcolor{gray}{$19.053\pm 3.429$} & \textcolor{gray}{$0.978\pm 0.005$} \\
 & $\ldots$ & \textcolor{gray}{$0.945\pm 0.005$} & \textcolor{gray}{$0.938\pm 0.038$} & \textcolor{gray}{$0.942\pm 0.062$} & \textcolor{gray}{$0.968\pm 0.021$} & \textcolor{gray}{$0.891\pm 0.008$} \\
\midrule
\multirow{2}{*}{TRPO-Lag} & $\ldots$ & $\mathbf{1.000\pm 0.000}$ & \textcolor{gray}{$21.328\pm 0.205$} & \textcolor{gray}{$23.291\pm 0.985$} & \textcolor{gray}{$21.250\pm 0.141$} & \textcolor{gray}{$0.981\pm 0.005$} \\
 & $\ldots$ & $0.961\pm 0.017$ & \textcolor{gray}{$0.975\pm 0.013$} & \textcolor{gray}{$0.853\pm 0.075$} & \textcolor{gray}{$0.979\pm 0.015$} & \textcolor{gray}{$0.894\pm 0.013$} \\
\bottomrule

    \end{tabular}
\end{table}

\hfill 
\newpage

\subsection{Learning Curves (tabular)}
\label{sec:learning-curves-tabular}

Here we provide the learning curves for all our experiments on the tabular environments provided in the MASA-Safe-RL library \citep{Goodall2025MASASafeRL}.

\begin{figure}[H]
    \centering
    \includegraphics[width=\linewidth]{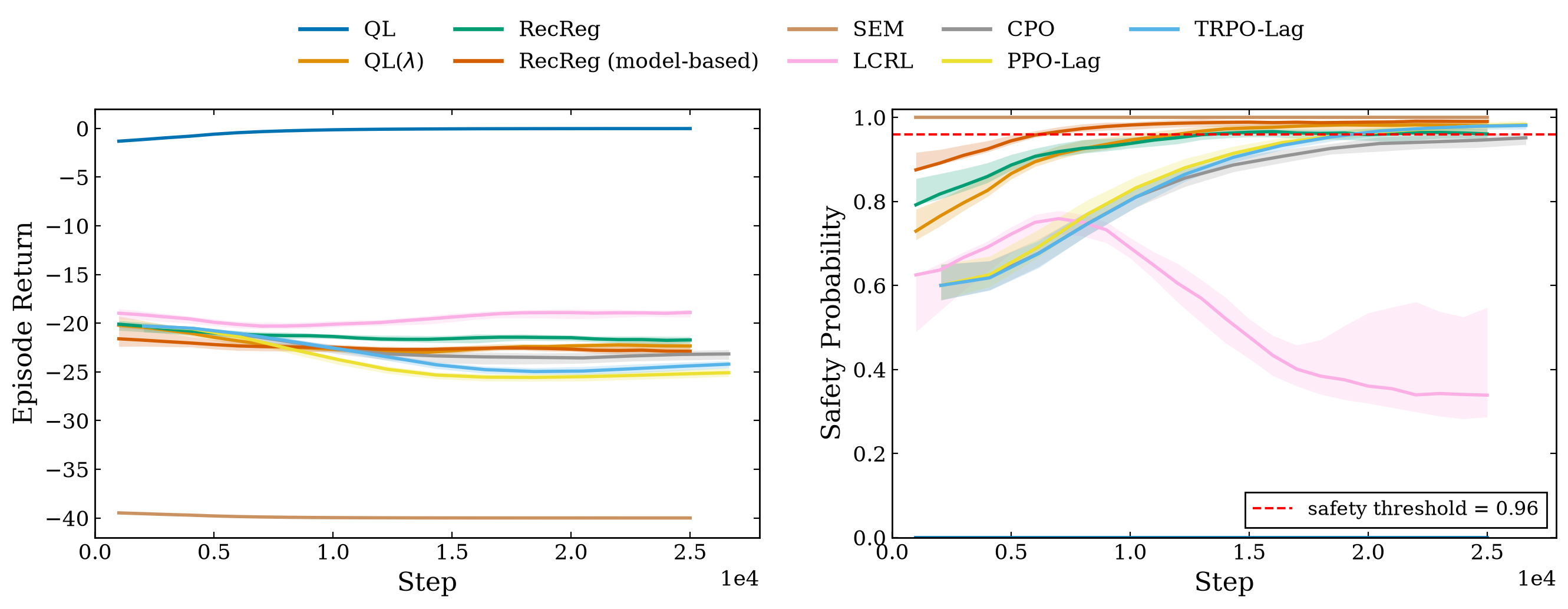}
    \caption{Learning curves for Media Streaming.}
    \label{fig:learningcurves1}
\end{figure}

\begin{figure}[H]
    \centering
    \includegraphics[width=\linewidth]{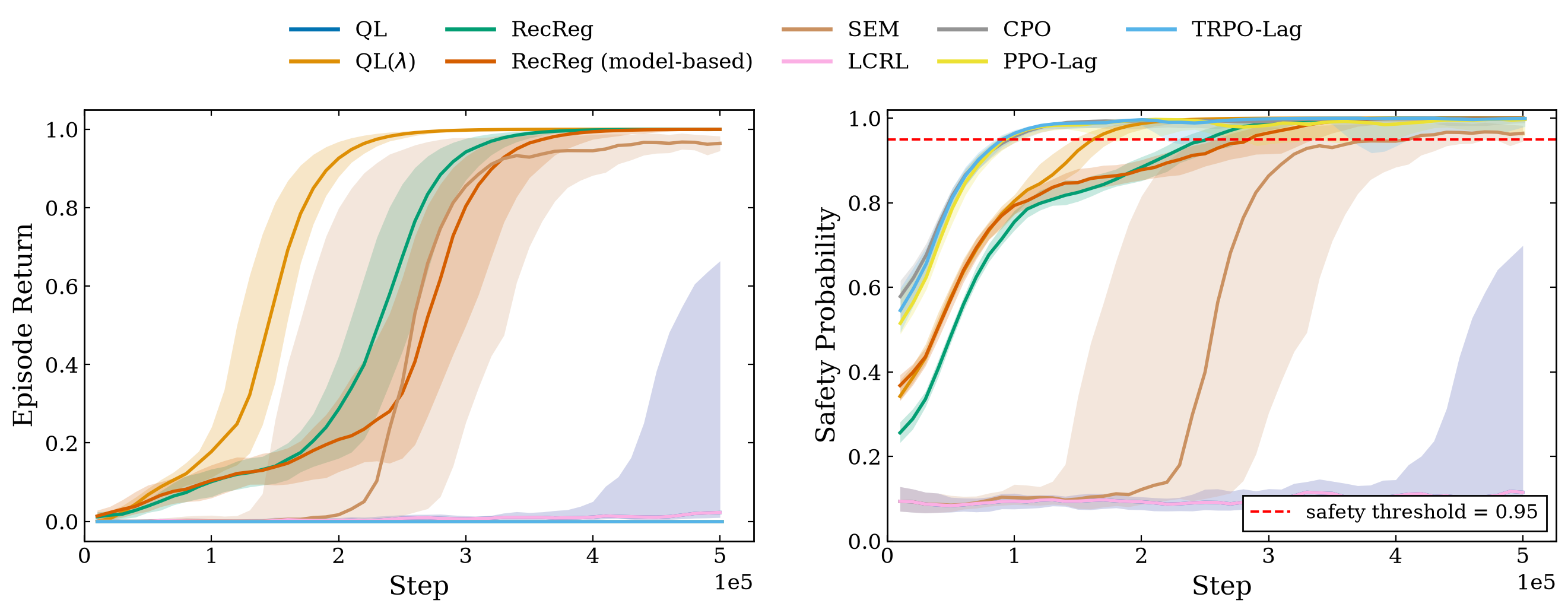}
    \caption{Learning curves for Bridge Crossing.}
    \label{fig:learningcurves2}
\end{figure}
\begin{figure}[H]
    \centering
    \includegraphics[width=\linewidth]{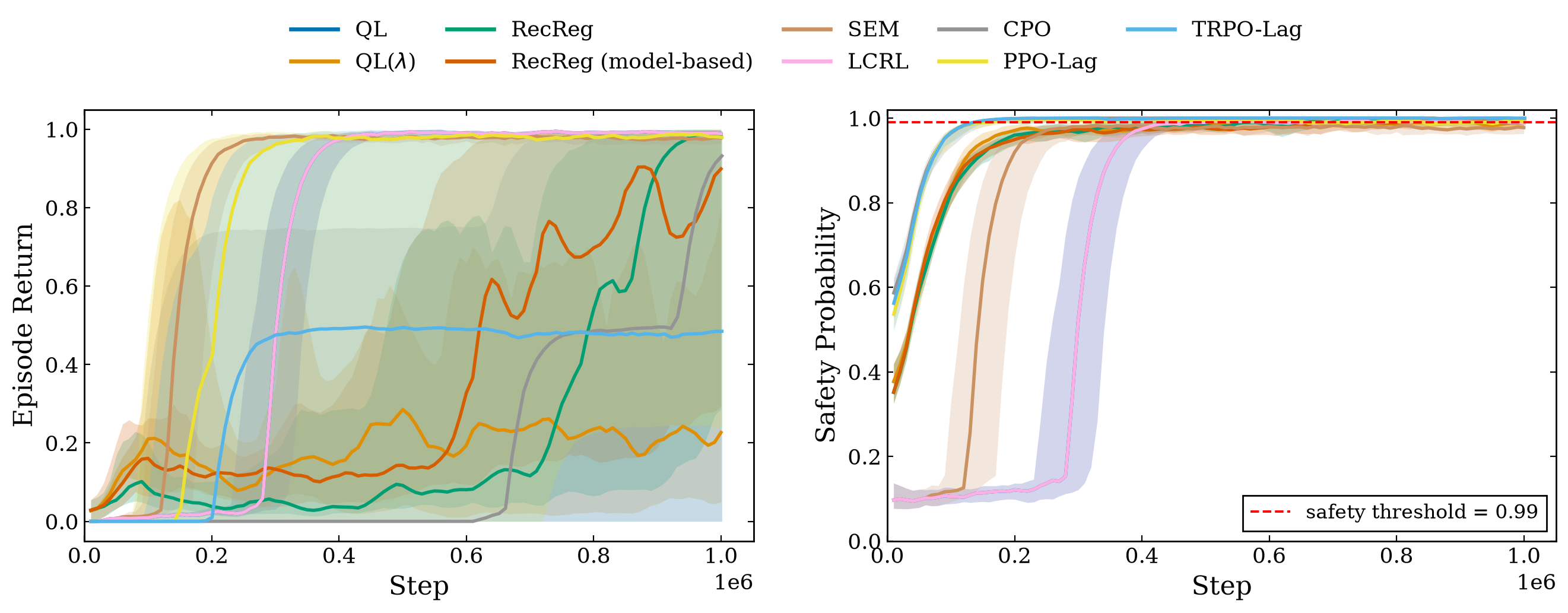}
    \caption{Learning curves for Bridge Crossing V2.}
    \label{fig:learningcurves999}
\end{figure}
\begin{figure}[H]
    \centering
    \includegraphics[width=\linewidth]{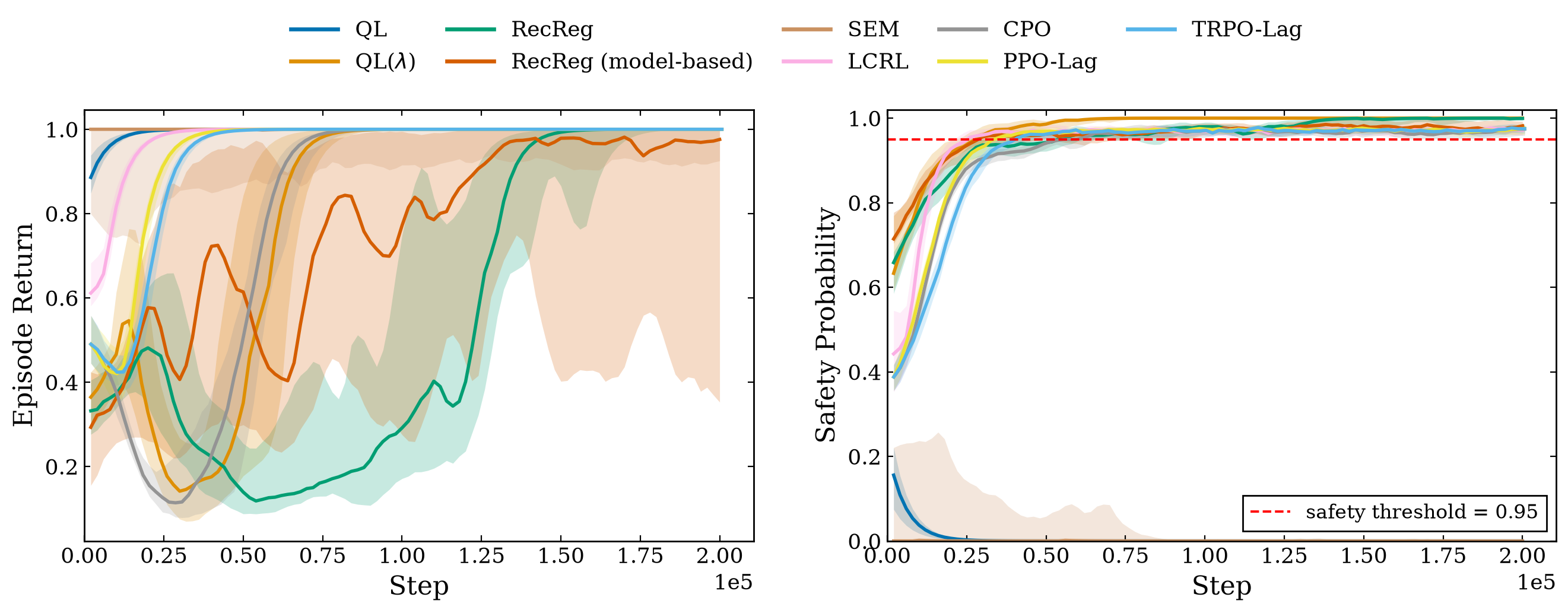}
    \caption{Learning curves for Colour Bomb Gridworld (Property 1).}
    \label{fig:learningcurves3}
\end{figure}

\begin{figure}[H]
    \centering
    \includegraphics[width=\linewidth]{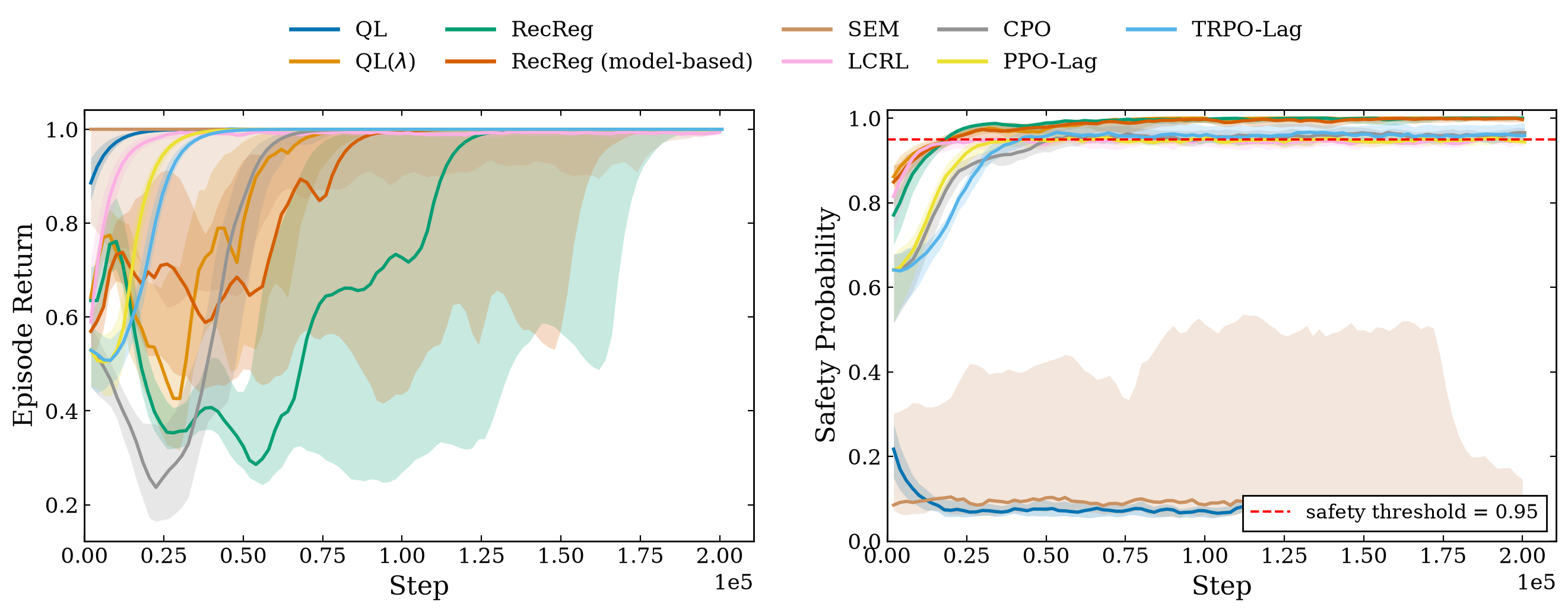}
    \caption{Learning curves for Colour Bomb Gridworld (Property 2).}
    \label{fig:learningcurves4}
\end{figure}

\begin{figure}[H]
    \centering
    \includegraphics[width=\linewidth]{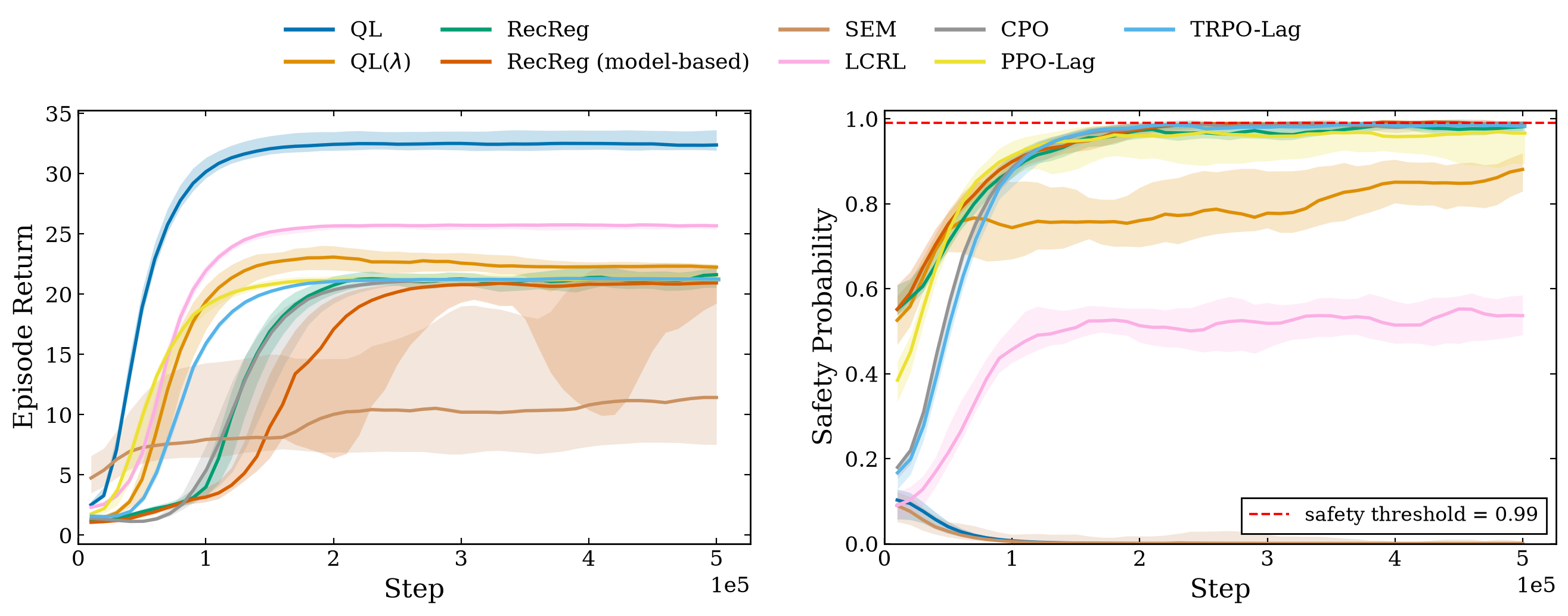}
    \caption{Learning curves for Colour Bomb Gridworld V2 (Property 1).}
    \label{fig:learningcurves5}
\end{figure}

\begin{figure}[H]
    \centering
    \includegraphics[width=\linewidth]{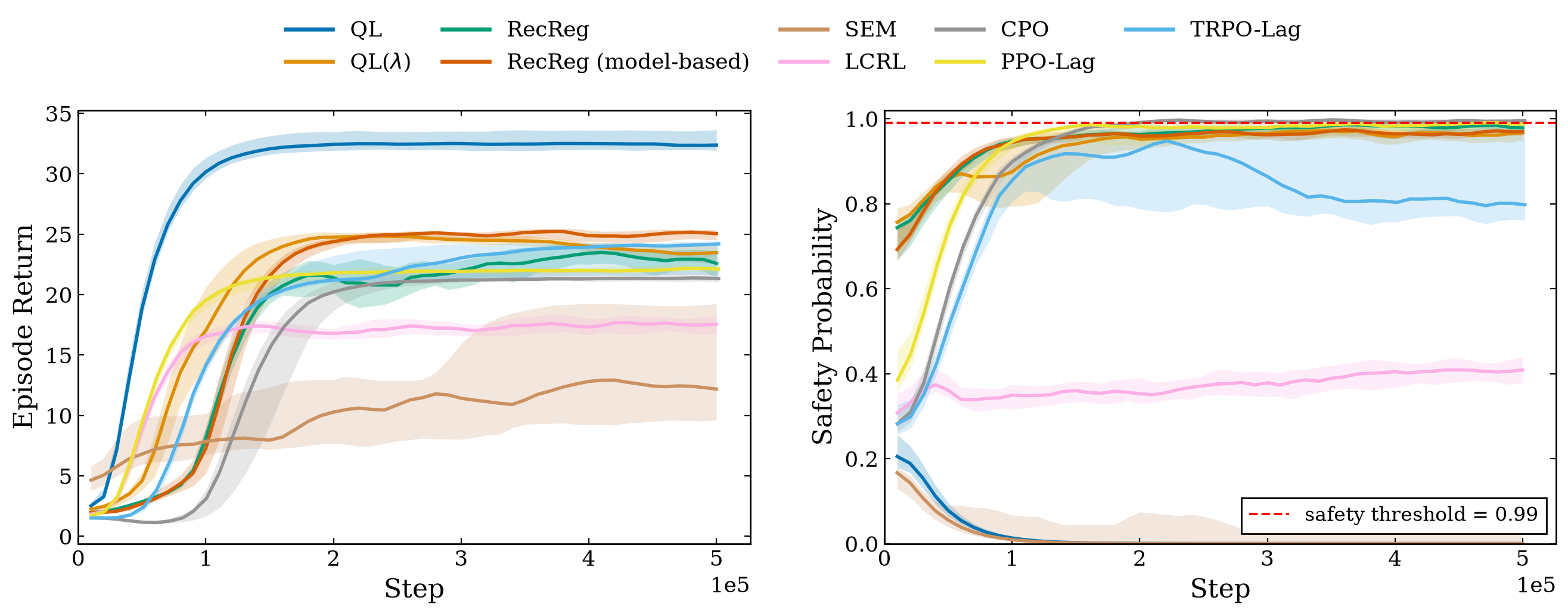}
    \caption{Learning curves for Colour Bomb Gridworld V2 (Property 2).}
    \label{fig:learningcurves6}
\end{figure}

\begin{figure}[H]
    \centering
    \includegraphics[width=\linewidth]{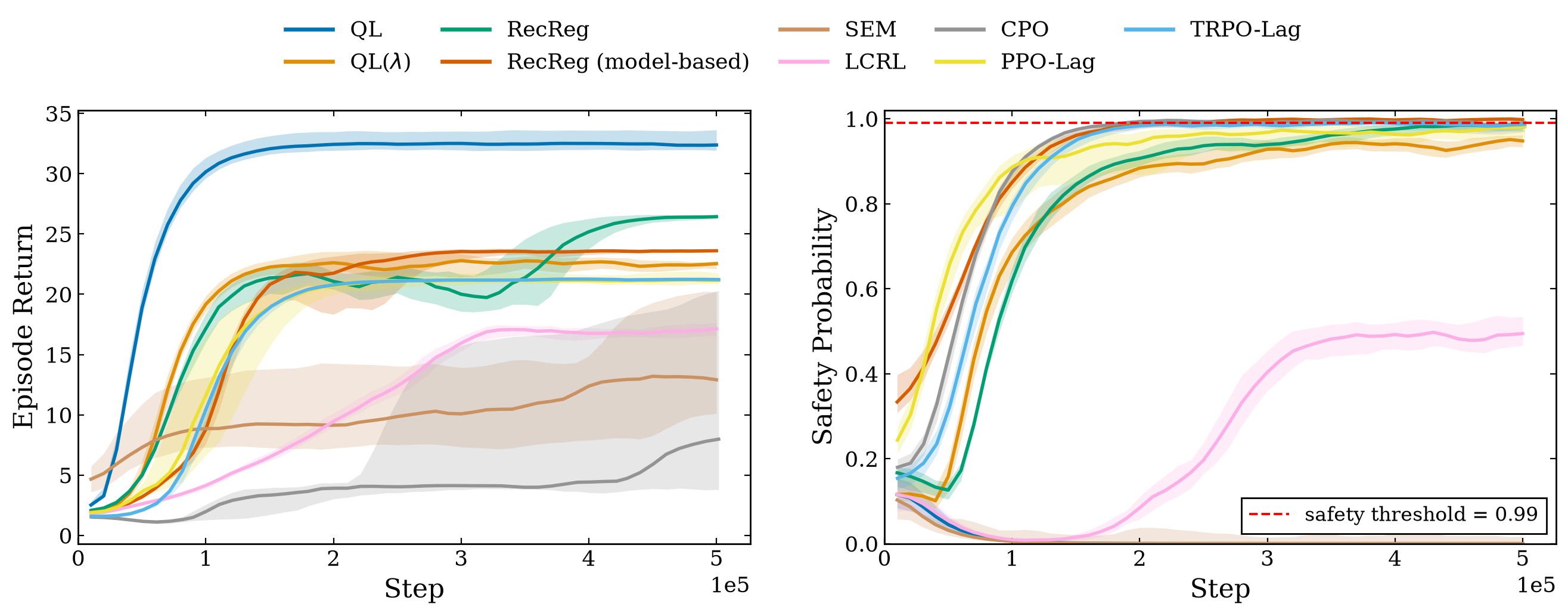}
    \caption{Learning curves for Colour Bomb Gridworld V2 (Property 3).}
    \label{fig:learningcurves7}
\end{figure}

\begin{figure}[H]
    \centering
    \includegraphics[width=0.95\linewidth]{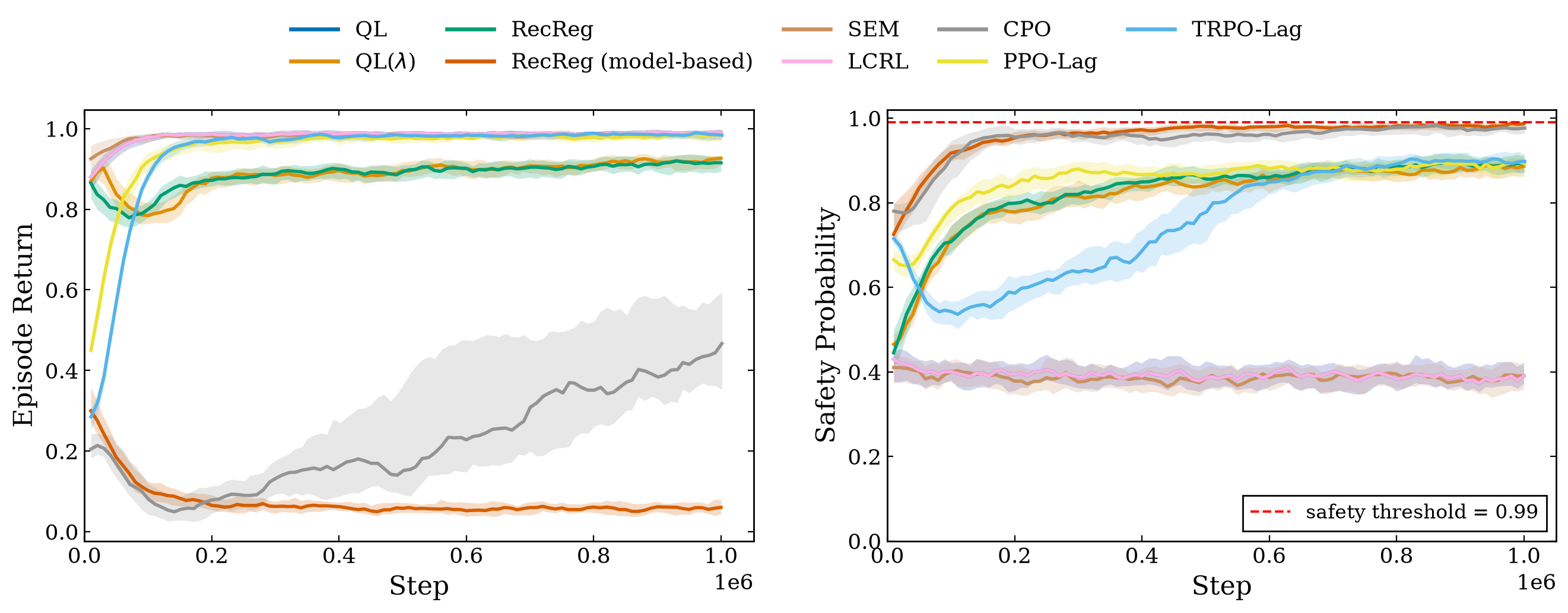}
    \caption{Learning curves for Mini Pacman.}
    \label{fig:learningcurves8}
\end{figure}

\subsection{Learning Curves (\textsc{ProbShield})}
\label{sec:learning-curves-probshield}

Here we provide the learning curves for our experiments comparing \textsc{RecReg-Exact} to \textsc{ProbShield}.

\begin{figure}[H]
    \centering
    \includegraphics[width=\linewidth]{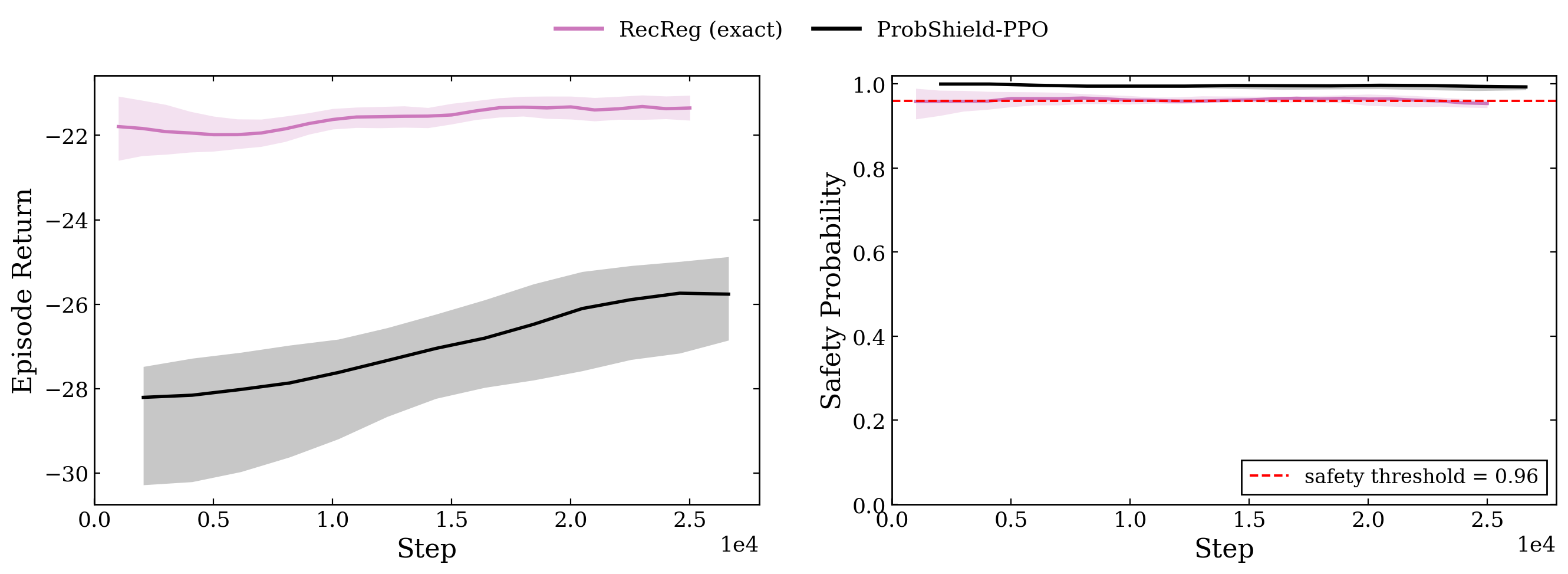}
    \caption{Learning curves for Media Streaming.}
    \label{fig:learningcurves12}
\end{figure}

\begin{figure}[H]
    \centering
    \includegraphics[width=\linewidth]{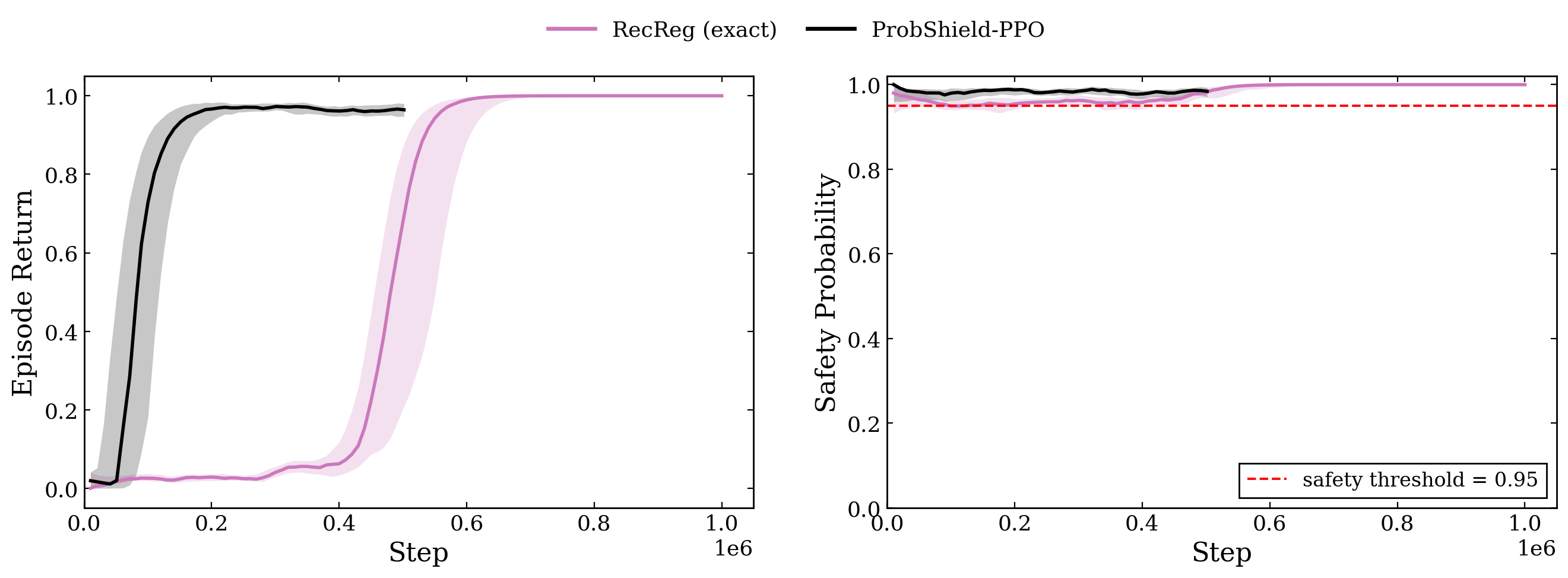}
    \caption{Learning curves for Bridge Crossing.}
    \label{fig:learningcurves11}
\end{figure}

\begin{figure}[H]
    \centering
    \includegraphics[width=\linewidth]{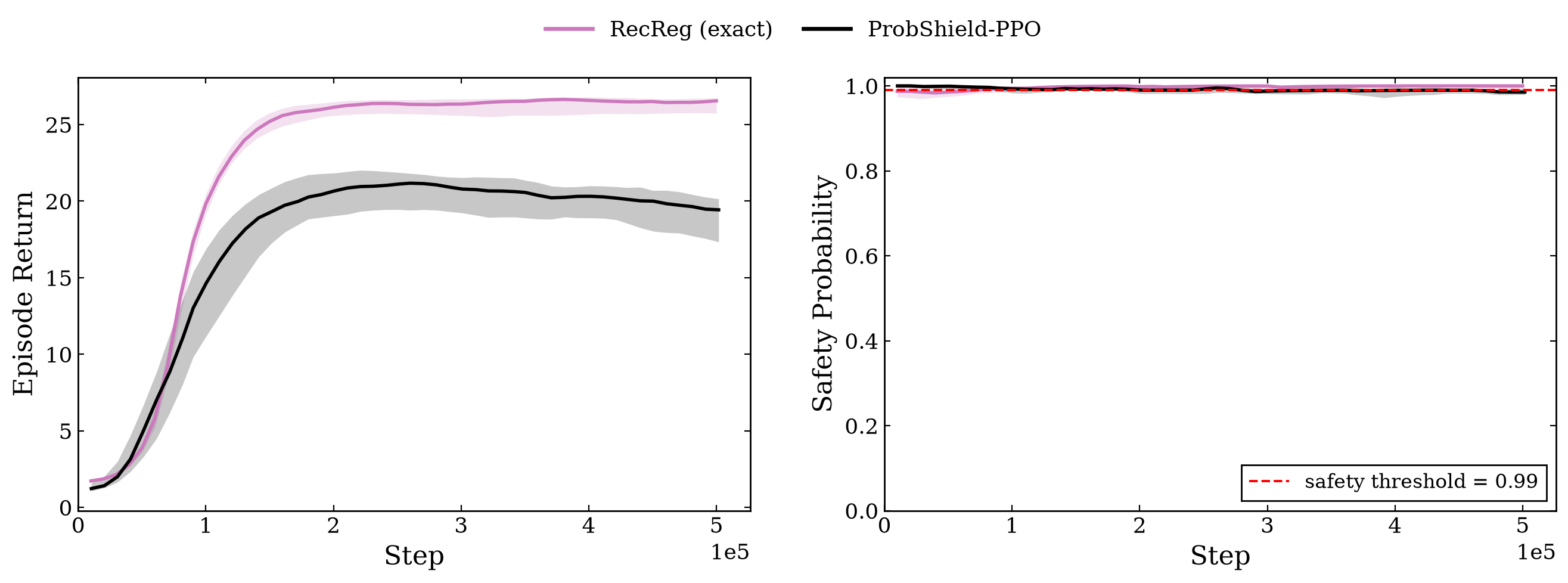}
    \caption{Learning curves for Colour Bomb Gridworld V2 (Property 3).}
    \label{fig:learningcurves9}
\end{figure}

\begin{figure}[H]
    \centering
    \includegraphics[width=\linewidth]{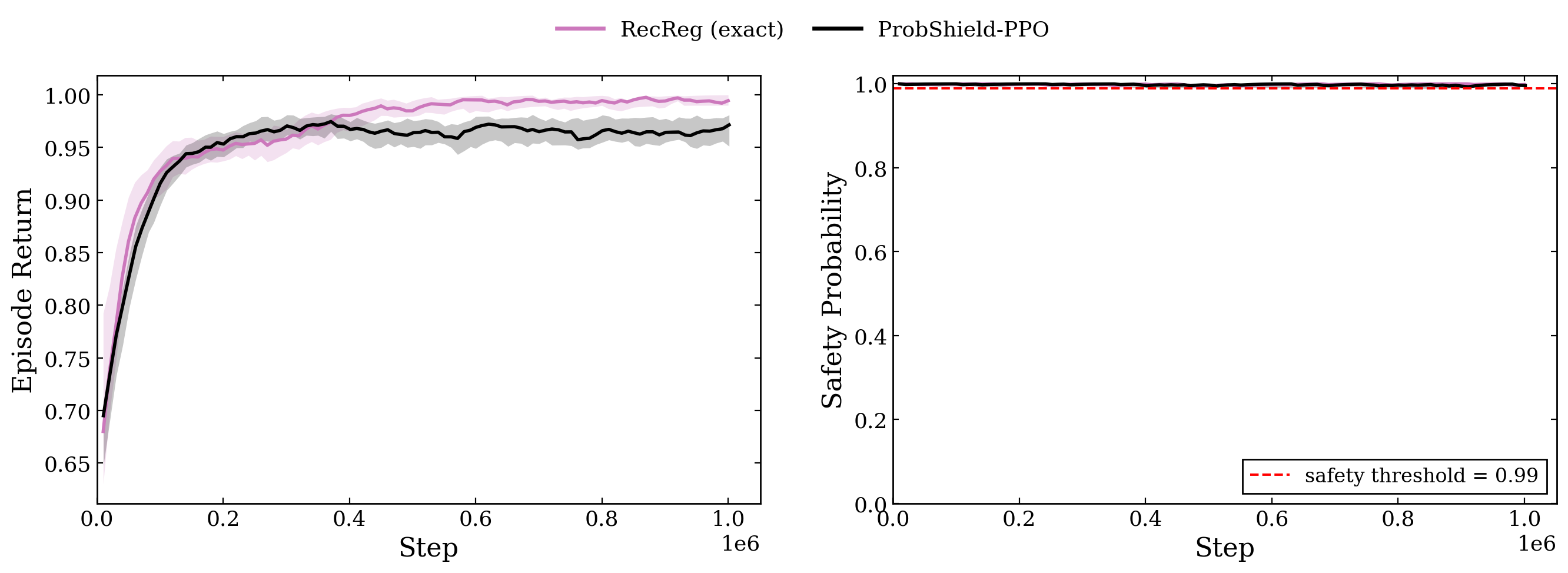}
    \caption{Learning curves for Mini Pacman.}
    \label{fig:learningcurves10}
\end{figure}

\subsection{Ablation Studies}
\label{sec:ablation-studies}

We conduct a simple ablation study for \emph{Colour Bomb Gridworld V2} with \emph{Property 2}. The DFA for Property 2 has 22 states, meaning the product MDP has 22 times more states than the original MDP. In particular, we ablate the use of \emph{counter factual experiences} \cite{icarte2022reward} for training the task and safe policy. We also investigate the effect of two reward shaping mechanisms based on the idea of \emph{potential-based reward shaping} \citep{ng1999policy}. First we provide a primer.

\paragraph{Potential-based reward shaping.}  Potential-based reward shaping is an approach used to typically deal with sparse or delayed reward environments. Potential-based reward shaping provides intermediate rewards of the following form,
\begin{equation}
    r'(s', s, a) = r(s', s, a) + \gamma \Phi(s') - \Phi(s)
\end{equation}
where $\Phi: \S \to \mathbb{R}$ is the potential function.Intuitively this provides the agent with a dense reward signal for moving to `more promising' states with higher potential values. Note that by discounting the potential values with the original discount factor $\gamma$ this keeps the set of optimal policies that maximize the original reward function unchanged \cite{ng1999policy}. Then if the potential function is useful in some sense, this allows us to learn optimal policies more quickly.

\paragraph{Automated reward shaping.}  Automated reward shaping (RS) for reward machines (RM) was proposed in \cite{icarte2022reward}. We adopt the same methodology here. The potential function $\Phi: \S^{\otimes} \to \mathbb{R}$ is defined over product state space $\S^{\otimes}$, each automaton state $q \in \Q$ is assigned a potential value $v(q)$. The potential value for each automaton state is then computed by running value iteration with a different discount factor $\gamma' < \gamma$. The cost function for the safe policy $\pi_{\text{safe}}$ is then redefined as,
\begin{equation}
    \mathbf C'(s^{\otimes}, a, {s^{\otimes}}') = \mathbf C(s^{\otimes}, a, {s^{\otimes}}') + \gamma v(q') - v(q) \quad \text{where $s^{\otimes} = (s, q) $ and $ {s^{\otimes}}'=(s', q')$} \label{eq:reward-shaped-cost}
\end{equation}
Intuitively the value function $v : \Q \to \mathbb{R}$ essentially assigns automaton states that are closer to the accepting state with higher values. This in turn penalized the agent for moving closer to the accepting state in the automaton. 

\paragraph{Cycle-based reward shaping.} This approach is inspired by \emph{cycle experience replay} \cite{shah2024deep} which was used for full LTL constraints. The approach is as follows, we first compute the length shortest path through the DFA $\mathcal{D}$, from each automaton state $q \in \Q$ to the initial state $\Q_0$, e.g., via breadth-first search. Each automaton state $q \in \Q$ is assigned a value $v(q)$ which corresponds to the length of the shortest path from $q$ to $\Q_0$, which will be used for potential-based reward shaping. This provides an incentive to the agent to return to the initial state $\Q_0$ as quickly as possible (short cycles), while avoiding the accepting states of the DFA $\mathcal{D}$, which are sink nodes, i.e. $v(q) = \infty \; \forall q \in \F$. Again, the cost function is redefined as in \eqref{eq:reward-shaped-cost}.

\paragraph{Results for \textsc{RecReg-ModelFree}.} Figure \ref{fig:ablation1} provides the learning curves for \textsc{RecReg-ModelFree} with the two different potential-based reward shaping approaches. First we see that without counter-factual experiences (CFE) the algorithm struggles to converge quickly and achieves a low safety probability due to slower convergence of the safe policy. Automated (Potential) reward shaping destabilizes learning slightly and the algorithm converges suboptimally, whereas cycle-based reward shaping improves learning stability and converges optimally. 

\begin{figure}[H]
    \centering
    \includegraphics[width=0.85\linewidth]{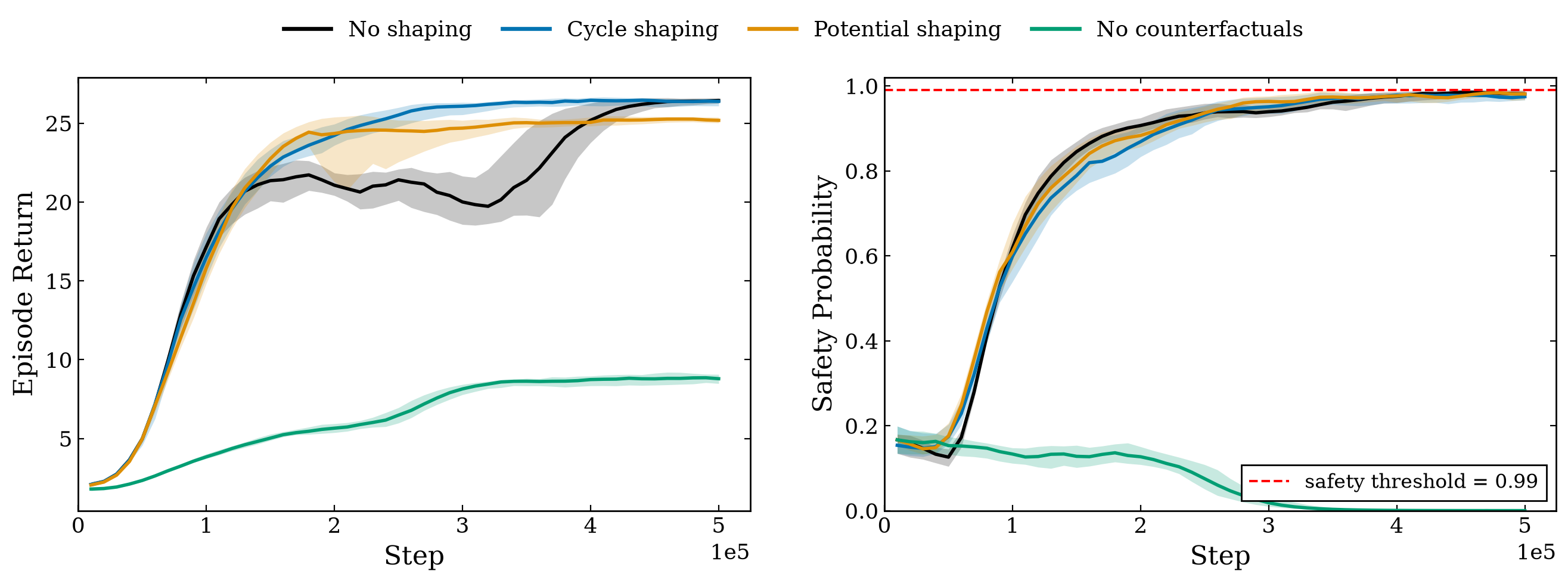}
    \caption{Reward Shaping for Colour Bomb Gridworld V2 with \textsc{RecReg-ModelFree}.}
    \label{fig:ablation1}
\end{figure}

\paragraph{Results for QL-$\lambda$.} Figure \ref{fig:ablation2} provides the learning curves for QL-$\lambda$ with the two different potential-based reward shaping approaches. Again we see that without counter-factual experiences (CFE) the algorithm struggles to converge quickly. However, in contrast to before both automated (Potential) and reward shaping fully destabilizes learning collapsing the policy to the safest possibly policy (no reward). Whereas cycle-based reward shaping improves both safety and task performance uniformly finding the right balance. This highlights the difficulty with tuning Lagrangian approaches like QL-$\lambda$, subtle shifts in the cost function and penalty magnitudes can break the balance between task performance and safety.

\begin{figure}[H]
    \centering
    \includegraphics[width=0.85\linewidth]{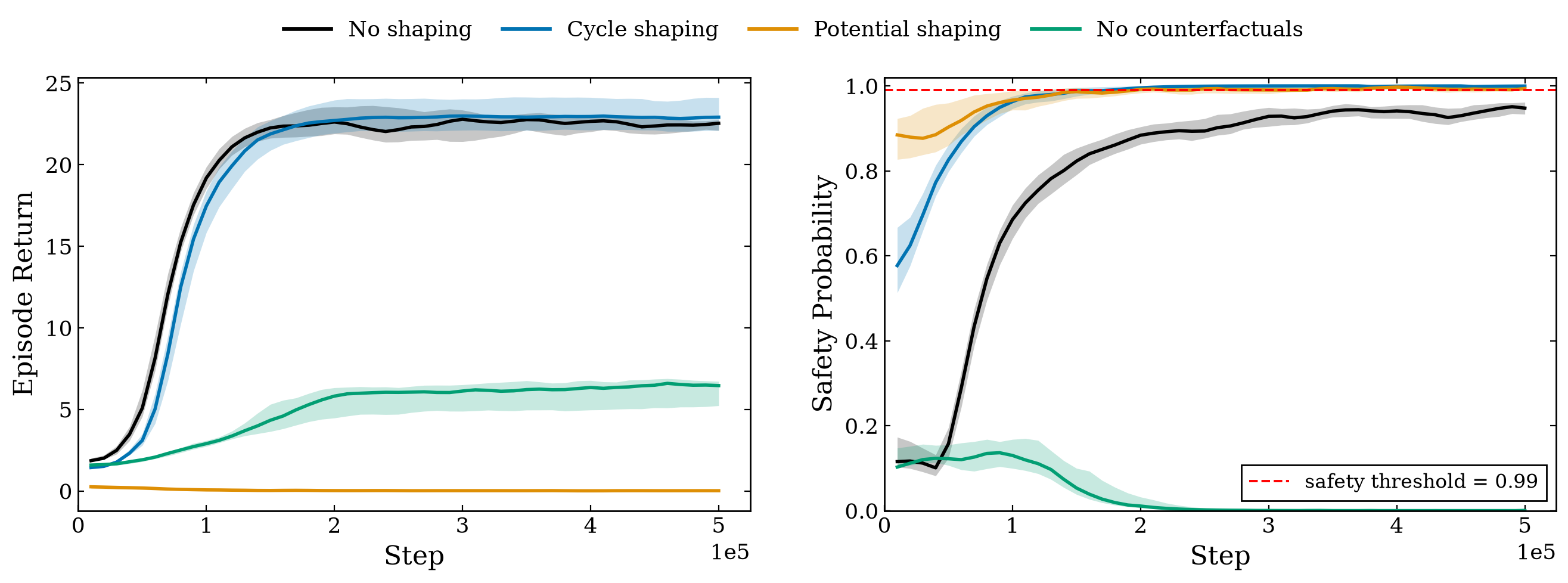}
    \caption{Reward Shaping for Colour Bomb Gridworld V2 with QL-$\lambda$.}
    \label{fig:ablation2}
\end{figure}

\subsection{Runtime Analysis}
\label{sec:rumtime-analysis}

Running a model checking routine at each decision making step inevitable incurs some non-trivial computational overhead. This overhead grows as the size of the MDP or DFA increases. In Tab.~\ref{tab:runtime} we provide a quick runtime comparison for the different operation mode of \textsc{RecReg} on environments of varying size. For reference \textsc{Exact} and \textsc{ModelFree} operate at roughly the same speed as the other mode-free tabular baselines (e.g., Q-learning). 

\begin{table}[H]
    \caption{Frames per second (fps) for different \textsc{RecReg} modes.}
    \label{tab:runtime}
    \centering
    \tiny
    \begin{tabular}{lccc}
        \toprule
           & \textsc{Exact} & \textsc{ModelBased} & \textsc{ModelBased} \\
           & \& \textsc{ModelFree} & (exact) & (stat. $m$=10,000) \\
        \midrule
        \texttt{media\_streaming} & $\sim2600$fps & $\sim1300$fps & $\sim 90$fps \\
        \texttt{bridge\_crossing} & $\sim2200$fps & $\sim1000$fps & $\sim 80$fps \\
        \texttt{$15\times15$\_gridworld} (3) & $\sim2000$fps & $\sim 525$fps & $\sim 60$fps \\
        \texttt{mini\_pacman} & $\sim2600$fps & $\sim 250$fps & $\sim27$fps \\
        \bottomrule
    \end{tabular}
\end{table}

\newpage

\section{Environment Details}
\label{sec:environment-details}

Here we provide a more complete description of each of the tabular environments considered in the main paper. Please note that all environments are provided as part of MASA-Safe-RL \url{https://github.com/sacktock/MASA-Safe-RL} \cite{Goodall2025MASASafeRL}. The exact same implementations are provided in our code submitted in the supplementary material.

\paragraph{Media Streaming.} The environment models a simple data-buffer management task, where the agent must regulate packet inflow to prevent buffer underflow.

\begin{itemize}
    \item \textbf{Dynamics:} packets leave in the data-buffer according to a Bernoulli process with rate $\mu_{\text{out}}=0.7$.
    \item \textbf{Action space:} $\A = \{\textit{fast}, \textit{slow}\}$
    \begin{itemize}
        \item $\textit{fast}$ adds new packets with rate $\mu_{\textit{fast}} = 0.9$.
        \item $\textit{slow}$ adds new packets with rate $\mu_{\textit{slow}} = 0.1$.
    \end{itemize}
    \item \textbf{State space:} 20 different buffer states $\Rightarrow |\S| = 20$.
    \item \textbf{Reward function:} The agent receives a penalty of $-1.0$ when choosing the \textit{fast} action.
    \item \textbf{Episode length:} $T=40$.
    \item \textbf{Safety property:} The buffer must never become empty ($\varphi=\square \neg\, \text{empty}$).
    \item \textbf{Automaton size:} $|\Q|=2$.
    \item \textbf{Cost thresholds}: 
    \begin{itemize}
        \item \textsc{RecReg} (lower-bound): $1 - \epsilon = 0.96$.
        \item PPO-Lag / CPO (cost upper-bound): $0.04$.
        \item QL-$\lambda$ (coefficient): $100.0$.
        \item Rec. RL (step-wise upper-bound): $\epsilon_t = 0.001$.
    \end{itemize}
    \item \textbf{Model checking horizon:} $H=5$.
\end{itemize}

\begin{figure}[H]
  \centering
  \includegraphics[width=0.25\linewidth]{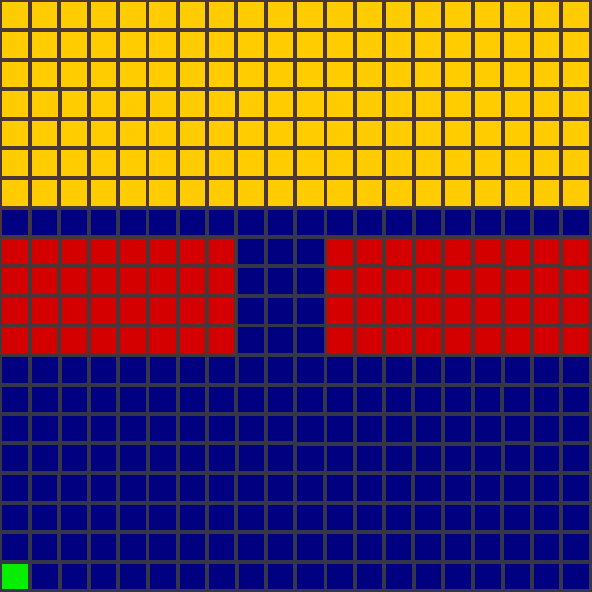}
        \caption{The Bridge Crossing environment.}
        \label{fig:bridgecrossinglarge}
\end{figure}

\paragraph{Bridge Crossing.} In this environment the agent acts in a $20 \times 20$ ``slippery'' gridworld (c.f., Fig.~\ref{fig:bridgecrossinglarge}), i.e., there is a $0.04$ chance that the agent's action is ignored and another action is uniformly sampled.

\begin{itemize}
    \item \textbf{Dynamics:} typical gridworld dynamics with slip probability $0.04$. 
    \item \textbf{Action space:} $\A = \{\textit{Left}, \textit{Right}, \textit{Up}, \textit{Down}, \textit{Stay}\}$
    \item \textbf{State space:} $20 \times 20$ different locations in the gridworld $\Rightarrow |\S| = 400$.
    \item \textbf{Reward function:} The agent receives a reward of $1.0$ and terminates when entering either the \textit{yellow} goal states.
    \item \textbf{Episode length:} $T=400$.
    \item \textbf{Safety property:} the agent must avoid all red \textit{lava} states ($\varphi = \square\neg \, \text{lava}$).
    \item \textbf{Automaton size:} $|\Q|=2$.
    \item \textbf{Cost thresholds}: 
    \begin{itemize}
        \item \textsc{RecReg} (lower-bound): $1 - \epsilon = 0.95$.
        \item PPO-Lag / CPO (cost upper-bound): $0.05$.
        \item QL-$\lambda$ (coefficient): $10.0$.
        \item Rec. RL (step-wise upper-bound): $\epsilon_t = 0.000125$.
    \end{itemize}
    \item \textbf{Model checking horizon:} $H=4$.
\end{itemize}

\begin{figure}[H]
    \centering
    \includegraphics[width=0.25\linewidth]{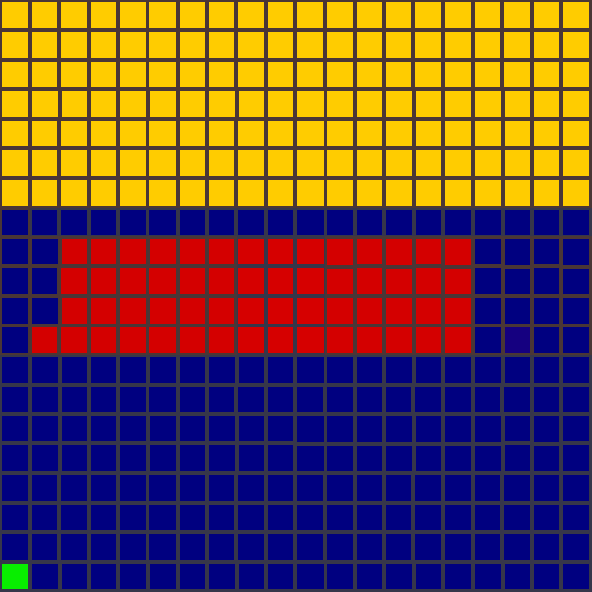}
    \caption{The Bridge Crossing V2 environment.}
    \label{fig:bridgecrossingvslarge}
\end{figure}

\paragraph{Bridge Crossing V2.} In this environment the agent acts in a $20 \times 20$ ``slippery'' gridworld (c.f., Fig.~\ref{fig:bridgecrossinglarge}), i.e., there is a $0.04$ chance that the agent's action is ignored and another action is uniformly sampled.

\begin{itemize}
    \item \textbf{Dynamics:} typical gridworld dynamics with slip probability $0.04$. 
    \item \textbf{Action space:} $\A = \{\textit{Left}, \textit{Right}, \textit{Up}, \textit{Down}, \textit{Stay}\}$
    \item \textbf{State space:} $20 \times 20$ different locations in the gridworld $\Rightarrow |\S| = 400$.
    \item \textbf{Reward function:} The agent receives a reward of $1.0$ and terminates when entering either the \textit{yellow} goal states.
    \item \textbf{Episode length:} $T=400$.
    \item \textbf{Safety property:} the agent must avoid all red \textit{lava} states ($\varphi = \square\neg \, \text{lava}$).
    \item \textbf{Automaton size:} $|\Q|=2$.
    \item \textbf{Cost thresholds}: 
    \begin{itemize}
        \item \textsc{RecReg} (lower-bound): $1 - \epsilon = 0.99$.
        \item PPO-Lag / CPO (cost upper-bound): $0.01$.
        \item QL-$\lambda$ (coefficient): $10.0$.
        \item Rec. RL (step-wise upper-bound): $\epsilon_t = 0.000025$.
    \end{itemize}
    \item \textbf{Model checking horizon:} $H=4$.
\end{itemize}

\begin{figure}[H]
    \centering
    \includegraphics[width=0.25\linewidth]{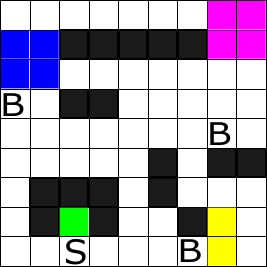}
    \caption{The Colour Bomb Gridworld environment.}
    \label{fig:colourgridworldlarge}
\end{figure}

\paragraph{Colour Bomb Gridworld.} The access point for this environment in MASA is under the env id: \emph{colour\_bomb\_gridworld}. In this environment the agent acts in a $9 \times 9$ ``slippery'' gridworld (c.f., Fig.~\ref{fig:colourgridworldlarge}) with a $0.1$ chance that the agent's action is ignored and another action is uniformly sampled. 

\begin{itemize}
    \item \textbf{Dynamics:} typical gridworld dynamics with slip probability $0.1$. 
    \item \textbf{Action space:} $\A = \{\textit{Left}, \textit{Right}, \textit{Up}, \textit{Down}, \textit{Stay}\}$
    \item \textbf{State space:} $9 \times 9$ different locations in the gridworld $\Rightarrow |\S| = 81$.
    \item \textbf{Reward function:} The agent receives a reward of $1.0$ and terminates when entering either the \textit{blue}, \textit{pink} or \textit{yellow} states.
    \item \textbf{Episode length:} $T=100$.
    \item \textbf{Safety properties:} 
    \begin{itemize}
        \item Property 1: avoid all $\texttt{B}$ states ($\varphi = \square \neg \texttt{B}$).
        \item Property 2: if $\texttt{B}$ state is reached the agent must remain there for one additional timestep ($\varphi= (\neg \texttt{B} \land X \texttt{B})\to X \texttt{B}$).
    \end{itemize}
    \item \textbf{Automaton size:} $|\Q|=2$ for Property 1 and $|\Q|=4$ for Property 2.
    \item \textbf{Cost thresholds}: 
    \begin{itemize}
        \item \textsc{RecReg} (lower-bound): $1 - \epsilon = 0.95$.
        \item PPO-Lag / CPO (cost upper-bound): $0.05$.
        \item QL-$\lambda$ (coefficient): $10.0$.
        \item Rec. RL (step-wise upper-bound): $\epsilon_t = 0.005$.
    \end{itemize}
    \item \textbf{Model checking horizon:} $H=2$ for Property 1 and 2.
\end{itemize}

\begin{figure}[H]
    \centering
    \includegraphics[width=0.35\linewidth]{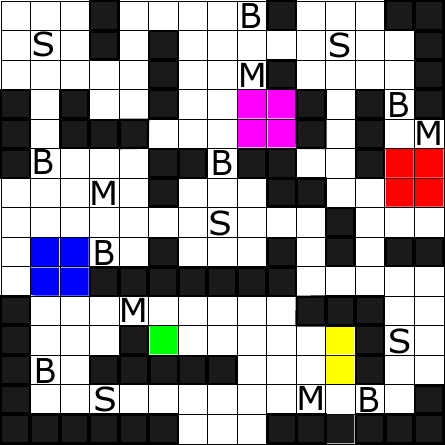}
    \caption{The Colour Bomb Gridworld V2 environment}
    \label{fig:colourgridworldv2large}
\end{figure}

\paragraph{Colour Bomb Gridworld V2.} The access point for this environment in MASA is under the env id: \emph{colour\_bomb\_gridworld\_v2}. In this environment the agent acts in a $9 \times 9$ ``slippery'' gridworld (c.f., Fig.~\ref{fig:colourgridworldv2large}) with a $0.1$ chance that the agent's action is ignored and another action is uniformly sampled. 

\begin{itemize}
    \item \textbf{Dynamics:} typical gridworld dynamics with slip probability $0.1$. After entering a coloured goal state, the agent transitions to a random starting ($\texttt{S}$) state.
    \item \textbf{Action space:} $\A = \{\textit{Left}, \textit{Right}, \textit{Up}, \textit{Down}, \textit{Stay}\}$
    \item \textbf{State space:} $15\times 15$ different locations in the gridworld $\Rightarrow |\S| = 225$.
    \item \textbf{Reward function:} The agent receives a reward of $1.0$ when entering the (non-terminal) coloured states \textit{blue}, \textit{pink}, \textit{red}, \textit{green} or \textit{yellow} states.
    \item \textbf{Episode length:} $T=250$.
    \item \textbf{Safety properties:} 
    \begin{itemize}
        \item Property 1: avoid all $\texttt{B}$ states ($\varphi = \square \neg \texttt{B}$).
        \item Property 2: if $\texttt{B}$ state is reached the agent must remain there for one additional timestep ($\varphi= (\neg \texttt{B} \land X \texttt{B})\to X \texttt{B}$).
        \item Property 3: If a $\texttt{B}$ state is reached the agent must reach a $\texttt{M}$ state and remain there for 2 timesteps all within 10 timesteps ($\varphi= \texttt{B} \to \lozenge^{\leq 10} (\texttt{M} \to (X \texttt{M}))$).
    \end{itemize}
    \item \textbf{Automaton size:} $|\Q|=2$ for Property 1, $|\Q|=4$ for Property 2 and $|\Q|=22$ for Property 3.
    \item \textbf{Cost thresholds}: 
    \begin{itemize}
        \item \textsc{RecReg} (lower-bound): $1 - \epsilon = 0.975$.
        \item PPO-Lag / CPO (cost upper-bound): $0.025$.
        \item QL-$\lambda$ (coefficient): $10.0$.
        \item Rec. RL (step-wise upper-bound): $\epsilon_t = 0.0001$.
    \end{itemize}
    \item \textbf{Model checking horizon:} $H=9$ for Property 1 and 2; $H=12$ for Property 3.
\end{itemize}

\begin{figure}[H]
    \centering
    \includegraphics[width=0.3\linewidth]{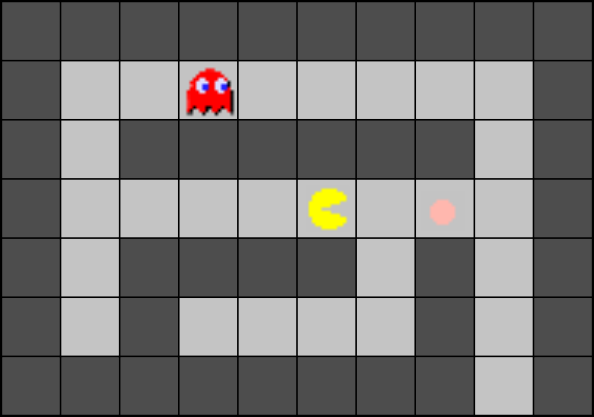}
    \caption{The Mini Pacman environment.}
    \label{fig:pacmanlarge}
\end{figure}

\paragraph{Mini Pacman.} The agent (Pacman) operates in a Pacman world, with one ghost and one piece of food (yellow circle) to collect (c.f., Fig.~\ref{fig:pacmanlarge}). The agent can exit via the bottom right hand corner, although from the initial configuration the ghost is likely to catch the agent if it heads straight for the food or exit.

\begin{itemize}
    \item \textbf{Dynamics:} typical Pacman dynamics: can only change direction at an intersection and turn around when facing a wall. The ghost follows the path with shortest Manhattan distance to the agent at every point, but with probability 0.6 picks a random action.
    \item \textbf{Action space:} $\A = \{\textit{Left}, \textit{Right}, \textit{Up}, \textit{Down}, \textit{No-Change}\}$
    \item \textbf{State space:} Accounting for all valid configurations for the Pacman's location, direction and the ghost's location and direction, there are precisely 8124 game states $\Rightarrow |\S| = 8124$.
    \item \textbf{Reward function:} The agent receives a reward of $1.0$ when it enters the location with the food present (c.f., Fig.~\ref{fig:pacmanlarge}); once the food is picked up it disappears. 
    \item \textbf{Episode length:} $T=100$ (no termination upon collecting the food, but can terminate via the exit).
    \item \textbf{Safety property:} The agent must avoid the ghost for the entire duration of the episode, even after picking up the food ($\varphi=\square \neg\, \text{ghost}$)
    \item \textbf{Automaton size:} $|\Q|=2$.
    \item \textbf{Cost thresholds}: 
    \begin{itemize}
        \item \textsc{RecReg} (lower-bound): $1 - \epsilon = 0.99$.
        \item PPO-Lag / CPO (cost upper-bound): $0.01$.
        \item QL-$\lambda$ (coefficient): $10.0$.
        \item Rec. RL (step-wise upper-bound): $\epsilon_t = 0.0001$.
    \end{itemize}
    \item \textbf{Model checking horizon:} $H=20$.
\end{itemize}

\section{Hyperparameters and Code}
\label{sec:hyperparameters-and-code}

\paragraph{Access to code.} 
The tabular environments are provided by MASA-Safe-RL \citep{Goodall2025MASASafeRL} (MIT license) and implemented with Gymnasium \citep{towers2024gymnasium} (MIT license). Tabular Q-learning is implemented with \emph{numpy} in \emph{Python}, the model checking procedures (both exact and statistical) are implemented with JAX \cite{jax2018github} (Apache-2.0 license) which supports vectorized computation on GPU and CPU. For PPO-Lag \cite{ray2019benchmarking} and CPO \cite{achiam2017constrained}, we use the implementations provided by Omnisafe \cite{omnisafe}, the code for running these benchmarks is mostly public facing at \url{https://github.com/sacktock/MASA-Safe-RL} our private code can also be made available upon reasonable request, for setup instructions please refer to the Omnisafe code base \url{https://github.com/PKU-Alignment/omnisafe} (Apache-2.0 license).

\paragraph{Training details.} For collecting all sets of experiments we has access to 2 NVIDIA Tesla A40 (48GB RAM) GPU and a 24-core/48 thread Intel Xeon CPU each with 32GB of additional RAM. Each run for the tabular experiments can take anywhere from several minutes to several hours depending on the number of training steps and the complexity of the environment or safety property being used.

\paragraph{Statistical significance.} Standard error (SE) intervals are provided in Table~\ref{tab:raw-returns1} \& \ref{tab:raw-returns2} for each of our experiments. In particular, we report 10 random initializations (seeds) for each experiment. In the learning curves (c.f., Section \ref{sec:learning-curves-tabular}) and bar charts (c.f., Figure \ref{fig:results}), we plot the non-parametric (bootstrap) 95\% confidence intervals, provided by \texttt{seaborn.lineplot} with default parameters: \texttt{errorbar=(`ci', 95), n\_boot=1000}. The error bars capture the randomness in the initialization of the policy parameters and batch sampling (e.g., for PPO-Lag and CPO), and any randomness in the environment.

\hfill
\newpage
\subsection{Hyperparameters for Tabular Algorithms}

Here we provide hyperparameter details for all of the tabular algorithms tested in MASA-Safe-RL \citep{Goodall2025MASASafeRL}. For more concrete per-environment hyperparameter settings we encourage the reader to study the relevant python scripts provided in the supplementary material.

\begin{table}[H]
  \caption{Hyperparameter details for Q-Learning based algorithms: \textsc{RecReg} Rec. RL, QL-$\lambda$, SEM, LCRL.}
  \label{tab:hyperparams2}
  \centering
  \small
  \begin{tabular}{lcc}
    \toprule
       Name & Symbol & Value  \\
    \midrule
    \multicolumn{3}{c}{Q-Learning (shared)}\\
    \midrule
    Learning rate & $\alpha$ & $\{0.05, 0.1\}$\\
    Discount factor & $\gamma$ & $\{0.95, 0.99\}$\\
    Exploration type & - & $\{\texttt{Boltzmann}, \texttt{Epsilon-greedy} \}$\\
    Boltzmann Temp. & $\tau$ & $0.05$\\
    Epsilon start & $\epsilon_\text{init}$ & 1.0 \\
    Epsilon final & $\epsilon_\text{final}$ & 0.05 \\
    Epsilon decay & - & \textit{Linear} \\
    Decay steps & - & 250,000 \\
    Counter factual experiences & - & \textit{True} \\
    \midrule
    \multicolumn{3}{c}{\textsc{RecReg}}\\
    \midrule
    Model checking type & - & \texttt{Statistical} (for \textsc{ModelBased})\\
    Number of samples & $m$ & $8124$ \\
    Step-wise risk & $\epsilon_t$ & Variable (c.f., Section \ref{sec:experiments}) \\
    Failure probability & $\delta_t$ & $0.0001$ \\
    Model checking horizon & $H$ & Variable (c.f., Section \ref{sec:environment-details}) \\
    Approximation error  & $\epsilon'$ & $0.0005$\\
    \midrule
    \multicolumn{3}{c}{\textsc{RecReg} (Safe policy $\pi_{\text{safe}}$)}\\
    \midrule
    Learning rate & $\alpha_\text{safe}$ & $\{0.05, 0.1\}$\\
    Discount factor & $\gamma_\text{safe}$ & $\{0.95, 0.99\}$\\
    Exploration type & - & $\{\texttt{Boltzmann}, \texttt{Epsilon-greedy} \}$\\
    Boltzmann Temp. & $\tau_\text{safe}$ & $0.01$\\
    Cost coefficient & - & $10.0$\\
    \midrule
    \multicolumn{3}{c}{Rec. RL}\\
    \midrule
    Risk Critic Alpha & $\alpha_\text{risk}$ & 0.1\\
    Risk Critic Gamma & $\gamma_\text{risk}$ & 0.99 \\
    Step-wise risk & $\epsilon_t$ & Variable (c.f., Section \ref{sec:environment-details})\\
    \midrule
    \multicolumn{3}{c}{Q-$\lambda$}\\
    \midrule
    Dual variable & $\lambda$ & $\{1.0, 10.0, 100.0\}$\\
    \midrule
    \multicolumn{3}{c}{SEM}\\
    \midrule
    $d(s, a)$ metric learning rate & $\alpha_d$ & $0.1$\\
    $d(s, a)$ metric discount & $\gamma_d$ & $0.95$\\
    $c(s, a)$ metric learning rate & $\alpha_d$ & $0.05$\\
    $c(s, a)$ metric discount & $\gamma_d$ & $0.999$\\
    \bottomrule
  \end{tabular}
\end{table}

\clearpage 
\newpage

\subsection{Hyperparameters for PPO-Lag and CPO}

\begin{table}[H]
  \caption{Hyperparameter details for PPO-Lag, TRPO-Lag and CPO.}
  \label{tab:hyperparam3}
  \centering
  \small
  \begin{tabular}{lcc}
    \toprule
       Name & Symbol & Value  \\
    \midrule
    \multicolumn{3}{c}{PPO-Lag, TRPO-Lag, CPO shared (Tabular)}\\
    \midrule
    Actor learning rate (PPO-Lag) & $\eta$ & $3\times 10^{-4}$\\
    Critic learning rate (CPO/TRPO-Lag) & $\eta_v$ & $1\times 10^{-3}$\\
    Discount factor & $\gamma$ & $\{0.95, 0.99\}$\\
    Cost coefficient & $c$ & $1.0$\\
    Cost threshold & - & Variable (c.f., Section \ref{sec:environment-details})\\
    Cost gamma & $\gamma_c$ & $\{0.9, 0.95, 0.99\}$ \\
    TD-lambda & $\lambda$ & $0.95$ \\
    Cost TD-lambda & $\lambda_c$ & $0.95$ \\
    Max grad norm & - & $0.5$ \\
    Entropy coefficient & - & $0.0$\\
    Update iterations (per epoch) & $k$ & $10$ \\
    Steps per epoch & $n$ & $2048$\\
    \midrule
    \multicolumn{3}{c}{CPO, TRPO-Lag shared trust region (Tabular)}\\
    \midrule
    Target KL divergence & $\delta$ & $0.01$\\
    Conjugate gradient iterations & - & $15$\\
    Conjugate gradient damping & - & $0.1$\\
    FVP sample frequency & - & $1$\\
    Line search steps & - & $15$\\
    Line search decay & - & $0.8$\\
    \midrule
    \multicolumn{3}{c}{PPO, CPO shared (Atari Seaquest)}\\
    \midrule
    Actor/critic learning rate & $\eta$ & $3\times 10^{-5}$\\
    Discount factor & $\gamma$ & $0.9967$\\
    Cost coefficient & $c$ & $1.0$\\
    Cost threshold & - & $1.0$\\
    Cost gamma & $\gamma_c$ & $0.95$ \\
    TD-lambda & $\lambda$ & $0.95$ \\
    Cost TD-lambda & $\lambda_c$ & $0.95$ \\
    Max grad norm & - & $40.0$ \\
    Entropy coefficient & - & $0.0$\\
    Steps per epoch & $n$ & $200000$\\
    \midrule
    \multicolumn{3}{c}{PPO-Lag, TRPO-Lag shared (Tabular and Atari Seaquest)}\\
    \midrule
    Initial Lagrangian multiplier & $\lambda_{\textit{init}}$ & $\{10.0, 100.0\}$\\
    Lagrangian learning rate & $\eta_\lambda$ & $0.01$ \\
    Lagrangian upper bound & $\lambda_{\max}$ & $100.0$ \\
    \midrule
    \multicolumn{3}{c}{PPO-Lag (Tabular and Atari Seaquest)}\\
    \midrule
    Epsilon clip & $\epsilon_{\textit{clip}}$ & $0.2$\\
    Batch size & $B$ & $64$\\
    \bottomrule
  \end{tabular}
\end{table}

\clearpage 
\newpage

\subsection{Hyperparameter details for \textsc{ProbShield}}

\begin{table}[H]
  \caption{Hyperparameter details for \textsc{ProbShield}.}
  \label{tab:hyperparams4}
  \centering
  \small
  \begin{tabular}{lcc}
    \toprule
    Name & Symbol & Value  \\
    \midrule
    \multicolumn{3}{c}{\texttt{ProbShieldWrapperDisc}}\\
    \midrule
    Initial safety bound & $\epsilon_\text{init}$ & 0.01\\
    Max VI Steps & - & 10,000\\
    VI Cut off & - & $1\times 10^{-10}$\\
    Disc. Granularity & - & 20\\
    \midrule
    \multicolumn{3}{c}{PPO (underlying algorithm)}\\
    \midrule
    Num. steps & - & 2048\\
    Num. envs. & - & 1\\
    Batch size & - & 64\\
    $\pi$ training epochs & \texttt{n\_epochs} & 10\\
    Discount factor & $\gamma$ & 0.99\\
    (GAE) lambda & $\lambda$ & 0.95\\
    Clip range & $\epsilon$ & 0.2\\
    Normalize & - & advantages\\
    Ent. coef. & $\beta_{\text{ent}}$ & 0.0\\
    Value coef. & $\beta_{\text{value}}$ & 1.0\\
    $\pi$ optimizer & - & adam (max\_grad=0.5, $\epsilon=10^{-5}$)\\
    $\pi$ learning rate & - & $3\times 10^{-4}$\\
    $\pi$ network & - & 2-layers; 64 units; \texttt{ReLU}\\
    $V$ optimizer & - & adam (max\_grad=0.5, $\epsilon=10^{-5}$)\\
    $V$ learning rate & - & $3\times 10^{-4}$\\
    $V$ network & - & 2-layers; 64 units; \texttt{ReLU}\\
    \bottomrule
\end{tabular}
\end{table}

\end{document}